\documentclass[11pt]{article}
\usepackage[margin=1in]{geometry}

\usepackage{times}
\usepackage{microtype}

\usepackage{amsmath,amssymb,amsfonts,amsthm}

\usepackage{graphicx}
\usepackage{booktabs}

\usepackage{natbib}
\usepackage{url}

\usepackage{algorithm2e}

\usepackage{tikz}
\usetikzlibrary{arrows.meta,positioning,decorations.pathreplacing,calc}

\usepackage[pdfencoding=auto,psdextra]{hyperref}
\usepackage{bookmark}
\usepackage[capitalize,nameinlink,noabbrev]{cleveref}

\title{Online Realizable Regression and Applications for ReLU Networks\footnote{Accepted for presentation at the
Conference on Learning Theory (COLT) 2026}}
\author{%
	Ilan Doron-Arad\thanks{Supported by grant NSF DMS-2031883 and Vannevar Bush Faculty Fellowship ONR-N00014-20-1-2826 (PI Mossel).}\\
	Massachusetts Institute of Technology, Cambridge, MA, USA\\
	\texttt{ilanda@mit.edu}
	\and
	Idan Mehalel\thanks{Supported by the European Research Council (ERC) under the European Union’s Horizon 2022 research and innovation program (grant agreement No. 101041711), the Israel Science Foundation (grant number 2258/19), and the Simons Foundation (as part of the Collaboration on the Mathematical and Scientific Foundations of Deep Learning).}\\
	The Hebrew University of Jerusalem, Jerusalem, Israel\\
	\texttt{idanmehalel@gmail.com}
    \and
	Elchanan Mossel\thanks{Partially supported by NSF DMS-2031883, the Vannevar Bush Faculty Fellowship ONR-N00014-20-1-2826, MURI N000142412742, and a Simons Investigator Award.}\\
	Massachusetts Institute of Technology, Cambridge, MA, USA\\
	\texttt{elmos@mit.edu}
}
\date{}

\theoremstyle{plain}
\newtheorem{theorem}{Theorem}[section]
\newtheorem{lemma}[theorem]{Lemma}
\newtheorem{corollary}[theorem]{Corollary}
\newtheorem{proposition}[theorem]{Proposition}

\theoremstyle{definition}
\newtheorem{definition}[theorem]{Definition}

\theoremstyle{remark}

\newtheorem{claim}[theorem]{Claim}

\newcommand{\cR}{\mathcal{R}}
\newcommand{\cX}{\mathcal{X}}
\newcommand{\cF}{\mathcal{F}}
\newcommand{\cG}{\mathcal{G}}
\newcommand{\cH}{\mathcal{H}}
\newcommand{\cZ}{\mathcal{Z}}
\newcommand{\cN}{\mathcal{N}}
\newcommand{\cA}{\mathcal{A}}
\newcommand{\N}{\mathbb{N}}

\newcommand{\tree}{\mathcal{T}}
\newcommand{\rel}{\textnormal{ReLU}}
\newcommand{\clip}{\textnormal{clip}}

\newcommand{\cV}{\mathcal{V}}

\DeclareMathOperator{\fat}{fat}

\newcommand{\Alg}{\mathsf{Alg}}

\newcommand{\OPT}{\mathrm{OPT}}

\newcommand{\relu}{\operatorname{ReLU}}
\newcommand{\LD}{\mathrm{Ldim}}

\newcommand{\poly}{\mathrm{poly}}

\newcommand{\E}{\mathbb{E}}

\newcommand{\Rad}{\mathfrak{R}}
\newcommand{\VRad}{\mathfrak{R}^{\mathrm{vec}}}

\newcommand{\cY}{\mathcal{Y}}
\newcommand{\cP}{\mathcal{P}}
\newcommand{\R}{\mathbb{R}}
\newcommand{\eps}{\varepsilon}
\newcommand{\diam}{\operatorname{diam}}

\begin{document}
	\maketitle

	\begin{abstract}%
		
Realizable online regression can behave very differently from online classification.
Even without any margin or stochastic assumptions, realizability may enforce \emph{horizon-free}
(finite) cumulative loss under metric-like losses, even when the analogous classification
problem has an infinite mistake bound.
We study realizable online regression in the adversarial model under losses that satisfy an
approximate triangle inequality (approximate pseudo-metrics).
Recent work of \cite{attias2023optimal} shows that the minimax realizable cumulative loss is characterized by the scaled Littlestone/online dimension $\mathbb{D}_{\mathrm{onl}}$,
but this quantity can be difficult to analyze.
		
Our main contribution is a generic potential method that upper bounds $\mathbb{D}_{\mathrm{onl}}$
by a concrete Dudley-type entropy integral that depends only on covering numbers of the hypothesis
class under the induced sup pseudo-metric. For an hypothesis class $\mathcal{H}$,
we define an \emph{entropy potential}
$\Phi(\cH)=\int_{0}^{\diam(\cH)} \log N(\cH,\eps)\,d\eps$, where $N(\cH,\eps)$ is the $\eps$-covering number of $\cH$,
and show that for every $c$-approximate pseudo-metric loss,
$\mathbb{D}_{\mathrm{onl}}(\cH)\le O(c)\,\Phi(\cH)$.
In particular, polynomial metric entropy implies $\Phi(\cH)<\infty$ and hence a horizon-free realizable cumulative-loss bound with transparent dependence on effective dimension.

We illustrate the method on two families.
For the class $\mathcal{H}_L$ of all $L$-Lipschitz functions on $[-1,1]^d$ under $\ell_q(y,y')=|y-y'|^q$, we establish a sharp phase transition:
if $q>d$ then $\mathbb{D}_{\mathrm{onl}}(\cH_L)=\Theta_{d,q}(L^d)$ (and the bound is achievable efficiently), whereas if $q\le d$ then $\mathbb{D}_{\mathrm{onl}}(\cH_L)=\infty$.
Complementing these metric-specific results, we also show
that for any continuous loss with $\ell(y,y)=0$, the loss along realizable sequences in $\cH_L$
satisfies $\ell(\hat y_t,y_t)\to 0 ~\text{as } t\to\infty$.
As a second application, we study bounded-norm $k$-ReLU networks over $[-1,1]^d$ with squared loss and highlight a regression--classification separation:
realizable online classification is impossible already for $\cH_{k=2,d=1}$ under $0/1$ loss, yet realizable regression admits finite total loss, including a
$\widetilde O(k^2)$ cumulative-loss upper bound, a lower bound of $\Omega(k^2 \log (d/k^2))$, and an efficient $O(1)$ guarantee for a single ReLU independent of the input dimension.
Finally, 
we rule out (conditionally) efficient proper online learners that achieve realizable accumulated loss $\widetilde o(d)$ 
for any constant $k\ge 2$.
	\end{abstract}

\newpage 

\tableofcontents

\newpage 

	\section{Introduction}
	
	Online classification and online regression can behave very differently. Without a margin assumption, online classification is hard even for very simple hypothesis classes such as Perceptrons \cite{novikoff1962perceptron,shalevshwartz2014uml}, whereas regression, under \emph{realizable} sequences, can admit horizon-free (finite) cumulative-loss guarantees. 
	In this paper, we study \emph{realizable} online regression:
	\emph{when does realizability imply a horizon-free bound on cumulative loss, and what complexity parameter controls it?}
	We work in the adversarial online model (no stochastic assumptions), and we deliberately avoid any margin/separation condition.
	The key message is that, for regression under metric-like losses, realizability can be understood by a generic potential function mechanism, and enforce \emph{finite total loss} even when the analogous
	classification problem incurs an infinite mistake bound.

	We start with some basic definitions of the model. Fix an instance domain $\cX$, a label space $\cY$, a hypothesis class $\cH\subseteq \cY^{\cX}$, and a loss
	$\ell:\cY\times\cY\to\R_{\ge 0}$.
	In each round $t=1,2,\dots$, the learner observes $x_t\in\cX$, predicts $\hat y_t\in\cY$, then observes
	$y_t\in\cY$ and incurs the loss $\ell(\hat y_t,y_t)$. We say the sequence is \emph{realizable} by $\cH$
	if there exists $h^\star\in\cH$ such that $y_t=h^\star(x_t)$ for all $t$.
	The object is the minimax realizable cumulative loss:
	\[
	\inf_{\Alg}\ \sup_{h^\star\in\cH, (x_t)_{t \geq 1}}\ \sum_{t\ge 1}\ell(\hat y_t,y_t),
	\]
	where $\hat y_t \;:=\; \Alg\!\bigl(x_t \mid (x_r,y_r)_{r < t}\bigr)$ and $y_t=h^\star(x_t)$ for all $t$. 
	The goal is to upper bound this quantity (or show it is infinite), as a function of $\cH$ and $\ell$. 
	
	The recent work of \cite{attias2023optimal} shows that
	the minimax realizable cumulative loss is characterized (up to a factor of $2$) by an abstract dimension: the scaled Littlestone/online dimension $\mathbb D_{\mathrm{onl}}$. This works for every loss that satisfies a triangle-type inequality.\footnote{it rules out degenerate losses where a learner can predict an intermediate label and pay tiny loss forever
		even though the scaled Littlestone tree gaps stay large (we give an explicit example in \Cref{sec:prel})}
	When $\ell$ behaves like a (pseudo-)metric, the online game becomes tightly linked to the geometry of $\cH$
	under the induced sup pseudo-metric $d_\ell(f,g)=\sup_x \ell(f(x),g(x))$.
	
	Following \cite{attias2023optimal}, we focus on losses that satisfy an (approximate) triangle inequality.
	For $c\ge 1$, we call $\ell$ a \emph{$c$-approximate pseudo-metric} if $\ell(y,y)=0$,
	$\ell(y_1,y_2)=\ell(y_2,y_1)$, and
	$\ell(y_1,y_2)\le c(\ell(y_1,y_3)+\ell(y_2,y_3))$ for all $y_1,y_2,y_3\in\cY$
	(see Definition~\ref{def:approx-pseudometric} for the formal definition).
	Under this assumption, \cite{attias2023optimal} define the abstract complexity measure
	of scaled Littlestone dimension $\mathbb D_{\mathrm{onl}}$,
	that characterizes the minimax realizable cumulative loss up to constant factors.
	However, $\mathbb D_{\mathrm{onl}}$ can be difficult to compute directly and determining the online realizable regression complexity of specific hypothesis classes remains challenging.

	Our contribution is to upper bound this abstract quantity $\mathbb D_{\mathrm{onl}}$ by a concrete \emph{Dudley entropy integral}~\cite{	dudley1967sizes,rakhlin2015online} 
	that depends only on covering numbers of $\cH$ and $d_\ell$.
	Conceptually, this parallels classical “entropy integral” phenomena:
	finite metric entropy implies finite total loss. Moreover, using this entropy technique we derive tight bounds for interesting case study classes such as Lipschitz functions and ReLU networks.

	\subsection{Our Results}
	Our main results can be summarized as follows, where all results consider solely realizable online learning without a margin assumption. We connect online learnability to metric entropy via an entropy-potential bound; we then derive clean horizon-free consequences from polynomial covering numbers and establish a sharp phase transition for Lipschitz regression under the $\ell_q$ loss. As an important case study, we explore bounded ReLU classes, contrasting classification and regression behavior, and analyze the total loss obtainable for these classes. 
	
	We start with our general machinery of upper bounding the online dimension $\mathbb D_{\mathrm{onl}}$ by a simple Dudley entropy integral, which then yields
	covering-number-based guarantees.
	Let $d_\ell$ be the induced (pseudo-)metric on $\cH$:
	\[
	d_\ell(f,g)\ :=\ \sup_{x\in\cX} \ell\big(f(x),g(x)\big),\qquad f,g\in\cH,
	\]
	and let $N(U,\eps) = N_{\cH}(U,d_{\ell},\eps)$ denote the $\eps$-covering number of $U\subseteq\cH$ with respect to $d_\ell$; that is, $N(U,\eps)$ is the minimum cardinality of a set $S \subseteq \cH$ such that for every $f \in U$ there is $g \in S$ such that $d_{\ell}(f,g) \leq \eps$. 
	Assume $\diam(\cH):=\sup_{f,g\in\cH} d_\ell(f,g)<\infty$, and define the {\em entropy potential}
	\[
	\Phi(U)\ :=\ \int_{0}^{\diam(\cH)} \log_2 N(U,\eps)\, d\eps,
	\]
	interpreted as an extended real value. Our first result is a useful upper bound on the online dimension. 
	
	\begin{theorem}[Online dimension via entropy potential]\label{thm:intro_Donl-via-Phi}
		Assume $\ell$ is a $c$-approximate pseudo-metric for some $c\ge 1$ and $\diam(\cH)<\infty$.
		Then
		\[
		\mathbb{D}_{\mathrm{onl}}(\cH)\ \le\ 4c \cdot \Phi(\cH).
		\]
	\end{theorem}

	Let us describe the proof's main idea. We analyze a realizable scaled Littlestone tree (the formal definition of \cite{attias2023optimal} is given in \Cref{sec:prel}) through \emph{version spaces} $U\subseteq\cH$: all hypotheses consistent with the realizable sequence thus far.
	At a node with gap $\gamma$, the (approximated) triangle inequality implies that for all $\eps < \frac{\gamma}{2 \cdot c}$,
	any $\eps$-cover of $U$ must split across the two children (no cover center can serve both sides).
	Integrating this cover splitting across scales yields a quantitative \emph{potential drop} of order $\gamma$,
	and a greedy descent argument produces a branch of the scaled Littlestone tree whose total gap is controlled by $\Phi(\cH)$. 
	This turns an abstract online dimension bound into a metric-entropy calculation.
	
	Unlike agnostic regret bounds that use sequential covers/chaining via sequential Rademacher processes
	(e.g.,~\cite{rakhlin2015online,rakhlin2015sequential,block2021majorizing}), realizability lets us work with standard
	(non-sequential) covers of $(\cH,d_\ell)$ under the sup pseudo-metric.

	As an immediate consequence of \Cref{thm:intro_Donl-via-Phi}, polynomial covering numbers of $(\cH,d_\ell)$ imply $\Phi(\cH)<\infty$ and hence a
	horizon-free realizable cumulative-loss bound. 
	
	\begin{corollary}\label{cor:intro_poly-covers}
		Assume that $\diam(\cH)\le 1$ and that there exist constants $A\ge 1$ and $p\ge 1$ such that for all
		$\eps\in(0,1]$ it holds that
		$
		N(\cH,\eps)\ \le\ \Big(\frac{A}{\eps}\Big)^p.
		$
		Then,
		\[
		\Phi(\cH)\ \le\ p\big(\log_2 A + 1/\ln 2\big),
		\]
		and thus by Theorem~\ref{thm:intro_Donl-via-Phi}:
		\[
		\mathbb{D}_{\mathrm{onl}}(\cH)\ \le\ 4c\,p\big(\log_2 A + 1/\ln 2\big).
		\]
	\end{corollary}
	
	As an immediate consequence of \Cref{thm:intro_Donl-via-Phi}, polynomial covering numbers of $(\cH,d_\ell)$ imply
	$\Phi(\cH)<\infty$, and therefore a horizon-free realizable cumulative-loss bound.
	We develop two instantiations: (i) $L$-Lipschitz classes on $[-1,1]^d$, exhibiting a sharp $q\gtrless d$ transition, and
	(ii) bounded ReLU networks, where parameter-based entropy bounds yield guarantees depending only on the dimension.
	
	\subsection*{Lipschitz Realizable Regression}
	For the class $\cH_L$ of $L$-Lipschitz functions on $[-1,1]^d$ (with respect to $\|\cdot\|_\infty$) and
	$\ell_q(y,y'):=|y-y'|^q$, we show a sharp dichotomy: : horizon-free learnability holds exactly when $q>d$, a condition that can be restrictive in high dimensions.
	
	\begin{theorem}\label{thm:intro_qgt-d}
		Fix $d\ge 1$ and $L\ge 1$.
		\begin{itemize}
			\item If $q>d$, then $\Phi(\cH_L)<\infty$ and $\mathbb{D}_{\mathrm{onl}}(\cH_L)=\Theta_{d,q}(L^d)$, which can be obtained efficiently.
			\item If $q\leq d$, then $\mathbb{D}_{\mathrm{onl}}(\cH_L)=\infty$. 
		\end{itemize}
	\end{theorem}

	The entropy bound becomes borderline at $q=d$ and we show that the total loss grows logarithmically with the horizon $T$, whereas for $q<d$ the total loss grows linearly with $T$.
	These regimes highlight that realizable cumulative-loss guarantees are loss-dependent.
	This motivates the following natural question: under what conditions we have $\ell(\hat y_t,y_t)\to 0$ as $t\to\infty$ for arbitrary loss functions?
	Interestingly, we show as a corollary that for $L$-Lipschitz functions, learning \emph{eventually} occurs for any truthful continuous loss.
	
	\begin{corollary}\label{cor:env-pointwise-loss}
		Let $\ell:[0,1]^2\to\mathbb{R}_{\ge 0}$ be a continuous loss satisfying $\ell(y,y)=0$ for all $y\in[0,1]$.
		Then, there is a deterministic proper algorithm that for every realizable sequence in $\cH_L$ satisfies
		$
		\lim_{t \to \infty} \ell(\hat y_t,y_t) = 0$. 
	\end{corollary}
	
	In Appendix~\ref{app:total}, we generalize the above result by showing that total boundedness of a function class $\cH$ with respect to $d_{\ell}$ is a sufficient condition for a deterministic proper realizable algorithm achieving eventual learnability, using an $\eps$-net elimination scheme.

	\begin{figure}[t]
		\centering
		\begin{tikzpicture}[>=Latex, font=\small]
			\def\W{8.2}
			\def\D{4.1}
			\def\Y{0.35}
			\def\H{1.45}
			
			\fill[red!30] (0,\Y) rectangle (\D,\H);
			\fill[blue!30]  (\D,\Y) rectangle (\W,\H);
			\draw (\D,\Y) -- (\D,\H);
			
			\node[align=center] at ({0.5*\D},{0.5*(\Y+\H)})
			{${q\leq d}$\\[-1pt]\footnotesize ${\mathbb D_{\mathrm{onl}}=\infty}$};
			
			\node[align=center] at ({0.5*(\D+\W)},{0.5*(\Y+\H)})
			{${q>d}$\\[-1pt]\footnotesize ${\mathbb D_{\mathrm{onl}} \leq \Phi(\cH_L)<\infty}$};
			
			\draw[->] (0,0) -- (\W+0.35,0) node[below] {};
			\draw (\D,0) -- (\D,-0.12) node[below] {$q$};
		\end{tikzpicture}
		\caption{Lipschitz realizable regression under $\ell_q(y,y')=|y-y'|^q$ loss: Divergence for $q\leq d$ and finiteness for $q>d$.}
		\label{fig:lipschitz-phase}
	\end{figure}
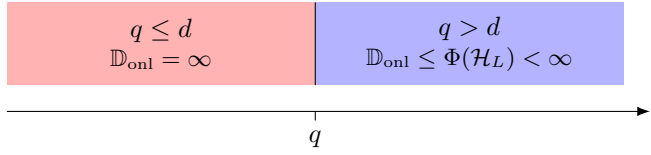

	\subsection*{ReLU Networks}
	A second, more concrete application we give for the generic potential technique is online realizable regression over ReLU networks. For the simplicity of the presentation, we restrict ourselves to the canonical squared loss and focus on shallow networks (see the discussion for more details on deep network and other Lipschitz activations). 
	Fix an input dimension $d$ and a constant $k\ge 2$, and let $\cH_{k,d}$ be the class of
	\emph{bounded $k$-\textnormal{ReLU}}~\cite{DBLP:conf/innovations/GoelKM021} on $[-1,1]^d$ with inputs of bounded norm $\|x\|_2\le 1$. That is, it is the function class $\cH_{k,d}$ that contains all functions of the form: 
	\[
	h(x)=\sum_{j=1}^k a_j \rel(w^j \cdot x),
	\qquad
	a\in [-1,1]^k,~~ ||w^j||_2 \leq 1 ~\forall j = 1, \ldots, k
	\] with the output clipped to $[-1,1]$. 
	We start by contrasting regression with classification for this class. Namely,
	even for bounded parameters, online classification in this class can be unlearnable.

	\begin{table}[t]
		\centering
		\setlength{\tabcolsep}{7pt}
		\renewcommand{\arraystretch}{1.25}
		\begin{tabular}{@{}p{0.15\linewidth}p{0.4\linewidth}p{0.33\linewidth}@{}}
			\toprule
			\textbf{Setting} & \textbf{Assumptions} & \textbf{Result for bounded $k$-ReLU} \\
			\midrule
			\textbf{Classification} & $0/1$ loss & 
			$\cH_{2,1}$ is unlearnable \\
			\addlinespace
			\textbf{Regression} & squared loss & 
			information-theoretic $\widetilde \Theta(k^2)$.\\
			\addlinespace
			\textbf{Regression} & squared loss & 
			Efficient $O(1)$-loss for $k=1$.\\
			\addlinespace
			\textbf{Regression} & squared loss; proper learners; Gap-ETH & 
			No $d^{O(1)}$-time $\widetilde o(d)$-loss $\forall k\ge 2$.\\
			\bottomrule
		\end{tabular}
		\caption{Our online realizable regression/classification results for bounded $k$-ReLU.}
		\label{tab:relu-contrast}
	\end{table}

	\begin{theorem}\label{thm:intro_classification}
		There is no realizable online learning algorithm for $\cH_{2,1}$ under the discrete $0/1$ loss
		$\ell_{0/1}(\hat y,y):=\mathbf{1}\{\hat y\neq y\}$ on $\cY=[0,1]$.
		
	\end{theorem}
	\Cref{thm:intro_classification} strengthens~\cite[Theorem 4.13]{geneson2025mistake} in two orthogonal ways:
	most importantly, we show infiniteness of the Littlestone dimension even with bounded weights and biases;
	less importantly, we use only two hidden neurons (rather than four).
	The above hardness does \emph{not} carry over to regression under approximate pseudo-metric losses, as shown next. 
	
	Corollary~\ref{cor:intro_poly-covers} yields qualitatively tight linear dependency with respect for the number of scalar parameters $p$ for the unbounded counterpart of the class $\cH_{k,d}$. However, 
	a direct entropy-potential argument does not by itself yield the dependence on $d$ that we seek.
	We show that one can still obtain a polylogarithmic dependence on $d$ by combining two ideas:
	(i) controlling \emph{large} gaps via sequential fat-shattering bounds \cite{rakhlin2015online},
	and (ii) controlling \emph{small} gaps via a scaling/coupling argument to an unbounded-norm comparison class.
	We write $\widetilde O(\cdot)$ to hide polylogarithmic factors in the input dimension $d$. 
	
	\begin{theorem}\label{thm:intro_k_bounded}
		There is a (not necessarily efficient) online realizable algorithm for bounded $k$-ReLU with squared loss
		whose cumulative loss is $O(k^2\cdot \log^4 d) = \widetilde O(k^2)$.
	\end{theorem}
	
	In particular, for constant $k$ this theorem yields $\widetilde O(1)$ accumulated loss.  
    The proof of \Cref{thm:intro_k_bounded} combines two ingredients. Large gaps are controlled using $\log^{O(1)} d$ sequential fat-shattering bounds, whereas small gaps are controlled using the entropy-potential method developed above, via a scaling/coupling argument to an unbounded-norm comparison class. Thus, the entropy-potential argument is essential for the small-loss regime and allows us to obtain the $\widetilde O(k^2)$ cumulative-loss guarantee.

The next lower bound shows that the quadratic dependency on $k$ and a logarithmic dependency on $d$ are necessary. Hence, our algorithm is tight up to polylogarithmic factors in $d$. 

\begin{theorem}
\label{thm:LBd}
If $d/k^2\to\infty$, then the realizable online square-loss minimax value of $\cH_{k,d}$ is at least
\(
\Omega\left(k^2\log(d/k^2)\right). 
\)
\end{theorem}

We also show in the appendix how to generalize the result for any depth $L$ ReLU network of width $k$ and input dimension $d$ and to obtain total loss $\tilde{O}\left(k^L\right)$, with a matching lower bound of ${\Omega}\left(k^{L} \log (d / k^L)\right)$.  
    For the interesting special case of a single ReLU where $k = 1$ with no biases or output scalar parameter, i.e., the class 
	$\cH = \{x \mapsto \rel\left(w \cdot x\right) \mid w \in [-1,1]^d, ||w||_2 \leq 1\}$,
	we improve the above polylogarithmic bound to a true constant,
	independent of $d$. This extends the classic constant-loss guarantee for linear prediction
	\cite{cesa1996worst}.
	\begin{theorem}\label{thm:intro_one_relu}
		There is an efficient $O(1)$ online realizable algorithm for $1$-ReLU with squared loss.
	\end{theorem}
	
	\Cref{thm:intro_one_relu} is efficient while \Cref{thm:intro_k_bounded} is only information-theoretic. Thus, it is a natural question whether an efficient algorithm exists for more than one ReLU ($k>1$). We rule this out conditionally for \emph{proper} learners based on the hardness of \cite{DBLP:conf/innovations/GoelKM021}.
	
	\begin{theorem}[Proper online hardness]\label{cor:intro_proper-online-hard}
		Assume \textnormal{Gap-ETH}~\cite{DBLP:journals/eccc/Dinur16,DBLP:journals/corr/ManurangsiR16} and fix any constant $k\ge 2$. For any $\eps \in (0,1)$ such that $\eps = \omega(1/\log d)$ and $\eps=o(1)$,
		there is no $\eps$-efficient proper online algorithm that achieves realizable accumulated loss $\widetilde o(d)$
		in time $\poly(d)$.
	\end{theorem}

	\subsection{Related Work}
	
	Realizable online regression is relatively unexplored, especially in comparison to classification and agnostic regression.  
	In terms of results, most related to our work, in particular to the application to neural nets, are the following papers.
	\cite{geneson2025mistake} prove some mistake bounds for nets with 0-1 loss. We improve their lower bound by showing that even nets with bounded input, number of parameters, and weights are unlearnable with 0-1 loss.
	The recent paper of \cite{daniely2025online} also study the 0-1 loss case, but with sign activation and margin assumptions, which are not used in this work.
	\cite{rakhlin2015online} study a similar setting to ours, but in the agnostic setting (i.e., not necessarily realizable), where a horizon-free loss bound as ours is out of reach. On training rather learning with ReLU networks, there is also a large body of work, e.g.,~\cite{dey20,DBLP:conf/innovations/GoelKM021,Froese22,Froese2024,Boob22,AroraBMM18,PilanciE20,ergen2021convex}.

	In terms of techniques, we rely on \cite{attias2023optimal}, which introduced the online dimension.
	We also use a variation of the Dudley entropy integral \cite{dudley1967sizes,cesa1999prediction},
	originating in stochastic processes and later adapted to adversarial online learning
	(e.g.,~\cite{rakhlin2010online,rakhlin2015online}).
	Unlike sequential chaining tools for \emph{agnostic} regret, our entropy-potential argument upper bounds the
	realizable online dimension via standard (non-sequential) covers in the induced sup pseudo-metric.

	The metric-entropy behavior of Lipschitz classes under the sup norm is classical, dating back to $\varepsilon$-entropy/capacity theory~\cite{kolmogorov1961entropy}. 
	In adversarial \emph{online} nonparametric regression, \cite{rakhlin2014online} developed a sequential-complexity framework (via sequential entropies) that yields minimax regret bounds and algorithms for H\"older/Lipschitz classes, with extensions to general loss functions~\cite{rakhlin2015general} and constructive chaining procedures~\cite{gaillard2015chaining}. 
	Our results are complementary: we study the \emph{realizable} setting and characterize when realizability yields horizon-free cumulative loss for $L$-Lipschitz regression under $\ell_q(y,y')=|y-y'|^q$. 
	

	\subsection{Discussion}
	
	In this paper, we give a Dudley-integral based technique to study realizable online regression. We use this technique to obtain tight bounds for Lipschitz functions and ReLU networks, in the context of realizable (online) regression. 
	We list below some implications and limitations of our work and suggest directions for future work.

	\paragraph{When the entropy potential diverges.}
	The entropy-potential approach yields horizon-free guarantees precisely when $\Phi(\cH)<\infty$.
	If $\log N(\cH,\eps)$ diverges too quickly as $\eps\to 0$ (e.g.\ $\log N(\cH,\eps)\asymp \eps^{-p}$ with $p\ge 1$),
	then $\Phi(\cH)=\infty$ and the bound from \Cref{thm:intro_Donl-via-Phi} is vacuous.
	Moreover, while the potential bound can be tight up to poly-logarithmic factors (e.g., unbounded norm ReLU nets, see Appendix~\ref{sec:param-to-potential}), the potential bound is not tight in general: $\Phi(\cH)$ can be infinite even when
	$\mathbb{D}_{\mathrm{onl}}(\cH)$ is finite. We give the proof of the following proposition in Appendix~\ref{app:div}. Thus, $\Phi(\cH)$ should be viewed as an explicit upper-bound surrogate for $\mathbb D_{\mathrm{onl}}(\cH)$, not as a characterization in general.
	\begin{proposition}[$\Phi(\cH)=\infty$ while $\mathbb{D}_{\mathrm{onl}}(\cH)<\infty$]
		There exist $(\cX,\cY,\ell)$ with $\ell$ a metric and a class $\cH\subseteq \cY^{\cX}$
		such that $\diam(\cH)<\infty$, $\Phi(\cH)=\infty$, but $\mathbb{D}_{\mathrm{onl}}(\cH)<\infty$.
	\end{proposition}

	\paragraph{Near-realizable sequences.}
	Realizability ($L=0$) can be too optimistic, while fully agnostic regret bounds can be pessimistic when the best
	comparator has loss $L \ll T$.
	Can one extend our framework to a near-realizable setting where $\inf_{h\in\cH}\sum_{t\le T}\ell(h(x_t),y_t)\le L$,
	with guarantees that depend (essentially) linearly on $L$?
	At a structural level, this suggests defining and characterizing an $L$-realizable analogue of the online dimension,
	and asking whether an entropy-potential mechanism continues to control it.
	Many works provided an analysis of this type for standard online classification. See, e.g.\, \cite{cesa1996line, cesa1997use, auer1999structural, branzei2019online, filmus2023optimal}.
	
	\paragraph{Efficiency.}
	Our Lipschitz upper bounds in the learnable regime are efficient, and \Cref{thm:intro_one_relu} is efficient as well; on the other hand, \Cref{thm:intro_k_bounded} is information-theoretic as it relies on our generic (and not necessarily efficient) potential technique.
	Can one design an efficient {\em improper} online algorithm for bounded $k$-\textnormal{ReLU with $k>1$} that achieves
	$\widetilde O(1)$ (or $O(1)$) cumulative loss, or alternatively prove computational lower bounds that apply even to improper learners? Our $\tilde{\Omega}(d)$ lower bound presented in \Cref{cor:intro_proper-online-hard} stems from the hardness of training and PAC learning bounded $k$-ReLU networks; thus, an improper $O(1)$ online learning algorithm does not contradict this result and it is left open. We note that proper and efficient eventual learnability follows from Lipschitz properties and \Cref{cor:env-pointwise-loss}.

\paragraph{Organization.}
\Cref{sec:prel} introduces the model, approximate pseudo-metric losses, and the scaled online dimension.
\Cref{sec:potential} develops the entropy-potential method and proves the general bound $\mathbb D_{\mathrm{onl}}(\cH)\lesssim \Phi(\cH)$ and \Cref{sec:Lip} gives a high level overview of our results for Lipschitz functions. Our full results for Lipschitz functions and ReLU networks as well as implications of the potential technique are delegated to the appendices.

\section{Preliminaries}
\label{sec:prel}

We use the scaled Littlestone dimension of \cite{attias2023optimal} for a general family of loss functions of approximate pseudo-metrics defined below. Let $\cX$ be an instance domain, $\cY$ a label space, and $\cH\subseteq \cY^{\cX}$ a hypothesis class.
Fix a loss $\ell:\cY\times\cY\to\R_{\ge 0}$. 

\begin{definition}[$c$-approximate pseudo-metric]\label{def:approx-pseudometric}
	For $c\ge 1$, we say that $\ell$ is a \emph{$c$-approximate pseudo-metric} if:
	(i) $\ell(y,y)=0$ for all $y\in\cY$;
	(ii) $\ell(y_1,y_2)=\ell(y_2,y_1)$ for all $y_1,y_2\in\cY$;
	(iii) for all $y_1,y_2,y_3\in\cY$,
	\[
	\ell(y_1,y_2)\ \le\ c\big(\ell(y_1,y_3)+\ell(y_2,y_3)\big).
	\]
	If $c=1$, this is a (true) pseudo-metric. 
\end{definition}

In the following, we give the definition of a scaled Littlestone tree. 

\begin{definition}[Scaled Littlestone Tree]
	\label{def:scaled-tree}
	A scaled Littlestone tree of depth $D\le \infty$ for $\cH\subseteq \cY^{\cX}$ (under loss $\ell$)
	is a complete binary tree where each internal node $u\in\{0,1\}^{<D}$ is labeled by an instance
	$x_u\in\cX$, and its two outgoing edges are labeled by $s_{u,0},s_{u,1}\in\cY$.
	The gap at node $u$ is
	$
	\gamma_u \ :=\ \ell(s_{u,0},s_{u,1}).$
	
	The tree is \emph{realizable by $\cH$} if for every branch $b\in\{0,1\}^D$ and every finite $n<D$,
	there exists $h\in\cH$ such that for all $t=0,1,\dots,n-1$,
	$
	h\!\left(x_{b_{\le t}}\right)\ =\ s_{b_{\le t},\, b_{t+1}},
	$
	where $b_{\le t}$ is the length-$t$ prefix of $b$ (and $b_{\le 0}=\emptyset$).
\end{definition}

\begin{definition}[Online Dimension]
	\label{def:online-dim}
	Let $\cH \subseteq [0,1]^{\cX}$. For a realizable scaled Littlestone tree $T$ for $\cH$ (Definition~\ref{def:scaled-tree}),
	let $\cP(T)$ denote its set of branches. For a branch $b\in\cP(T)$, write $b_{\le t}$ for its length-$t$ prefix;
	then the node at depth $t$ is indexed by $b_{\le t}$ and has gap $\gamma_{b_{\le t}}$.
	Define
	\[
	\mathbb D_{\mathrm{onl}}(\cH)
	\ :=\
	\sup_{T}\ \inf_{b\in \cP(T)}\ \sum_{t=0}^{\mathrm{dep}(T)-1} \gamma_{b_{\le t}},
	\]
	with the natural interpretation when $\mathrm{dep}(T)=\infty$ (sum over $t\ge 0$).
\end{definition}

The following result of \cite{attias2023optimal} shows that this dimension characterizes online realizable regression up to a constant factor. 

\begin{theorem}[\cite{attias2023optimal}]
	\label{thm:SLD}
	Let $\cH \subseteq [0,1]^{\cX}$ and $\eps > 0$ for some input domain $\cX$. Then, there exists a deterministic
	algorithm whose cumulative loss in the realizable setting is bounded by $\mathbb D^{\mathrm{onl}}(\mathcal H) + \eps$. Conversely,
	for any $\eps > 0$, every deterministic algorithm in the realizable setting incurs loss at least $\mathbb D^{\mathrm{onl}}(\mathcal H) / 2 -\eps$. 
\end{theorem}

\paragraph{Why an (approximate) triangle inequality is necessary.}
Without a triangle-type condition, “sum of gaps” does not control realizable cumulative loss.
For example, let $\cY=\{a,b,c\}$ and define a symmetric loss by $\ell(y,y)=0$, $\ell(a,b)=1$, and
$\ell(a,c)=\ell(b,c)=\varepsilon$ for arbitrarily small $\varepsilon>0$. Then $\ell$ violates any
$c$-approximate triangle inequality for $c<1/(2\varepsilon)$ since $\ell(a,b)=1>2c\varepsilon$.
In a realizable online problem with true labels in $\{a,b\}$, the learner predicting $\hat y_t\equiv c$
incurs loss $\varepsilon$ every round, regardless of the realized label, while any scaled Littlestone
tree labeled only by $\{a,b\}$ has gaps $\gamma_u=\ell(a,b)=1$ at each internal node, hence arbitrarily
large total gap along a branch. Therefore some (approximate) triangle inequality is essential and we focus our attention on approximate pseudo metrics.

\section{A Potential Upper bound for Online Realizable Regression}
\label{sec:potential}
In this section, we give our general upper bound on the scaled-Littlestone dimension for a general family of loss functions of approximate pseudo-metrics defined below. Let $\cX$ be an instance domain, $\cY$ a label space, and $\cH\subseteq \cY^{\cX}$ a hypothesis class.
Fix a loss $\ell:\cY\times\cY\to\R_{\ge 0}$.

Define the induced (pseudo-)metric on $\cH$ by
\begin{equation}\label{eq:induced-metric}
	d_\ell(f,g)\ :=\ \sup_{x\in\cX} \ell\big(f(x),g(x)\big),\qquad f,g\in\cH .
\end{equation}
If $\ell$ is a $c$-approximate pseudo-metric, then $d_\ell$ satisfies the same loss-properties as $\ell$ with the
same constant $c$ (by taking $\sup_{x\in\cX}$ in the inequalities for $\ell$).
For $U\subseteq\cH$ and $\eps>0$, define the $\eps$-covering number $N(U,\eps)$ w.r.t.  $d_\ell$ and $\cH$ by
\[
N(U,\eps)\ :=\ \min\Big\{|S|:\ S\subseteq\cH,\ \forall u\in U\ \exists s\in S\text{ with } d_\ell(u,s)\le \eps\Big\},
\]
with the convention $N(U,\eps)=\infty$ if no finite cover exists.
For the following, assume that the diameter $\diam(\cH):=\sup_{f,g\in\cH} d_\ell(f,g)$ is finite, and define the potential
\begin{equation}\label{eq:potential}
	\Phi(U)\ :=\ \int_{0}^{\diam(\cH)} \log_2 N(U,\eps)\, d\eps,
\end{equation}
interpreted as an extended real value (possibly $+\infty$). Since $\eps\mapsto N(U,\eps)$ is non-increasing, it is Borel measurable (every monotone function is measurable). We note that all logarithms are interchangeable up to constant factors; we keep base-$2$ in the formal definitions.

\subsection{Scaled Littlestone trees and a branch bound}

We use the following minimal structure of a scaled Littlestone tree for $\cH$:
each internal node $u$ is labeled by an instance $x_u\in\cX$ and has two outgoing edges labeled
$s_{u,0},s_{u,1}\in\cY$. Each node $u$ has an associated \emph{version space} $U\subseteq\cH$ consisting
of hypotheses consistent with the labels along the root-to-$u$ path. The children of $u$ have version spaces
\[
U_b\ :=\ \{h\in U:\ h(x_u)=s_{u,b}\},\qquad b\in\{0,1\}.
\]
We assume that for every internal node, both children's version spaces are non-empty. This follows by the realizability property in the definition of a scaled Littlestone tree (with respect to $\cH$). 
Define the \emph{gap} at an internal node $u$ by
\begin{equation}\label{eq:gap}
	\gamma_u\ :=\ \ell(s_{u,0},s_{u,1}).
\end{equation}
Let $\cP(\tree)$ denote the set of branches (root-to-leaf paths). For a branch $y\in\cP(\tree)$, let $y_0$
be the root and $y_i$ the $i$-th node along the branch. \Cref{thm:intro_Donl-via-Phi} follows directly from the next result. 

\begin{theorem}[Entropy-potential upper bound]\label{thm:potential-bound-general}
	Let $\tree$ be any realizable scaled Littlestone tree for $\cH$ (w.r.t.\ $\ell$).
	Assume $\ell$ is a $c$-approximate pseudo-metric for some $c\ge 1$ and $\diam(\cH)<\infty$.
	Then there exists a branch $y\in\cP(\tree)$ such that
	\[
	\sum_{i\ge 0}\gamma_{y_i}\ \le\ 4c\,\Phi(\cH).
	\]
	(If $\tree$ has finite depth $D$, the sum is over $i=0,1,\dots,D-1$.)
\end{theorem}

In particular, it follows that  $	\mathbb D^{\mathrm{onl}}(\mathcal H) \leq 4c\,\Phi(\cH)$, giving an upper bound on the total loss of an online realizable regression algorithm by \Cref{thm:SLD}. 
In the following, we prove \Cref{thm:potential-bound-general}. If $\Phi(\cH)=\infty$ then Theorem~\ref{thm:potential-bound-general} is vacuous; hence, in what follows we assume $\Phi(\cH)<\infty$. We illustrate the following result in \Cref{fig:potential-mechanism}. 

\begin{lemma}[Covers split below the gap]\label{lem:cover-sum}
	Let $u$ be an internal node with version space $U$ and children version spaces $U_0,U_1$.
	Then for all $\eps\in\big(0,\gamma_u/(2c)\big)$,
	\[
	N(U,\eps)\ \ge\ N(U_0,\eps)\ +\ N(U_1,\eps).
	\]
\end{lemma}

\begin{figure}[t]
	\centering
	\resizebox{0.7\columnwidth}{!}{%
		\begin{tikzpicture}[
			>=Latex, font=\small,
			box/.style={draw, rounded corners=4pt, thick, fill=black!2, inner sep=6pt},
			cover/.style={
				rounded corners=6pt,
				fill=black!3,
				minimum width=5.2cm,
				minimum height=2.0cm
			}
			pt/.style={circle, fill=black, inner sep=2.0pt},
			ball/.style={draw, circle, line width=0.5pt},
			cover2/.style={rounded corners=6pt, thick,
				minimum width=10.4cm, minimum height=1.2cm},
			]
			
			\node[box] (U)  at (0,3.0) {$U$};
			\node[box] (U0) at (-4,1.4) {$U_0$};
			\node[box] (U1) at ( 4,1.4) {$U_1$};
			
			\draw[->,thick] (U) -- (U0);
			\draw[->,thick] (U) -- (U1);
			
			(U0.north east) -- (U1.north west)
			
			\node[cover2] (C2) at (0,4.2) {};
			\node[cover] (C0) at (-4,-0.4) {};
			\node[cover] (C1) at ( 4,-0.4) {};
			\node[font=\normalsize] at (-4,0.3) {$\varepsilon\text{-cover of }U_0$};
			\node[font=\normalsize] at ( 4,0.3) {$\varepsilon\text{-cover of }U_1$};
			
			\node[font=\normalsize] at (0,4.3) {$\varepsilon\text{-cover of }U$};
			\foreach \x/\y in {-4.3/4.3, -3.2/4.3, -2.1/4.3, 2.1/4.3, 3.2/4.3, 4.2/4.3} {
				\draw[] (\x,\y) circle (0.4);
			}
			\foreach \x/\y [count=\i from 1] in {-4.3/4.3, -3.2/4.3, -2.1/4.3, 2.1/4.3, 3.2/4.3, 4.2/4.35} {
				\node at (\x,\y) {$f_{\i}$};
			}

			\foreach \x/\y in {-5.3/-0.7, -4.2/-0.9, -3.1/-0.5} {
				\draw[] (\x,\y) circle (0.4);
			}
			\foreach \x/\y [count=\i from 1] in {-5.3/-0.7, -4.2/-0.9, -3.1/-0.5} {
				\node at (\x,\y) {$f_{\i}$};
			}
			\foreach \x/\y [count=\i from 4] in {3.1/-0.7, 4.2/-0.5, 5.2/-0.9} {
				\node at (\x,\y) {$f_{\i}$};
			}
			\foreach \x/\y in {3.1/-0.7, 4.2/-0.5, 5.2/-0.9} {
				\draw[ball] (\x,\y) circle (0.4);
			}
			
		\end{tikzpicture}%
	}
	\caption{Illustration of the inequality $N(U,\varepsilon)\ \ge\ N(U_0,\varepsilon)+N(U_1,\varepsilon)$. Below the gap scale, no $\varepsilon$-ball can hit both children, so any $\varepsilon$-cover splits into disjoint covers of $U_0$ and $U_1$.}
	\label{fig:potential-mechanism}
\end{figure}
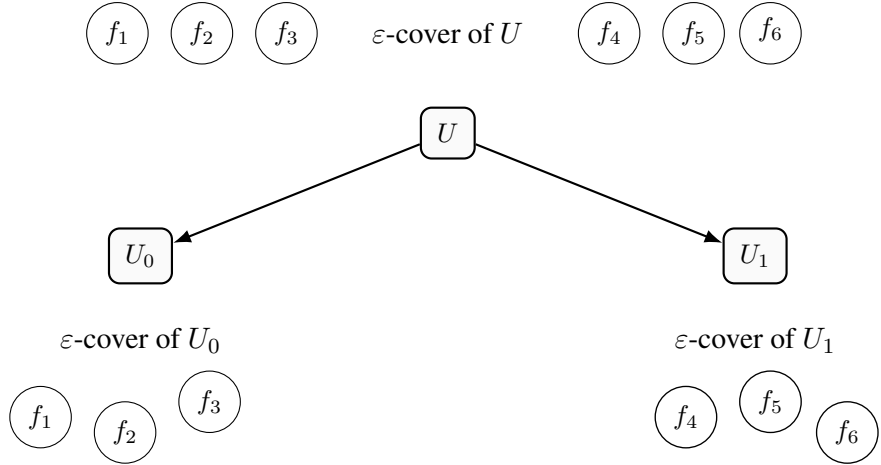

\begin{proof}
	Fix $\eps\in(0,\gamma_u/(2c))$. First note that for any $h_0\in U_0$ and $h_1\in U_1$,
	\[
	d_\ell(h_0,h_1)\ \ge\ \ell\big(h_0(x_u),h_1(x_u)\big)\ =\ \ell(s_{u,0},s_{u,1})\ =\ \gamma_u.
	\]
	Now suppose (towards contradiction) that there exist $h_0\in U_0$, $h_1\in U_1$, and $s\in\cH$ such that
	$d_\ell(h_0,s)\le \eps$ and $d_\ell(h_1,s)\le \eps$. By the $c$-approximate triangle inequality for $d_\ell$,
	\[
	d_\ell(h_0,h_1)\ \le\ c\big(d_\ell(h_0,s)+d_\ell(h_1,s)\big)\ \le\ 2c\eps\ <\ \gamma_u,
	\]
	a contradiction. Hence, intuitively, no single center $s\in\cH$ can be within $\eps$ of both some point in $U_0$
	and some point in $U_1$.
	Let $S\subseteq\cH$ be any $\eps$-cover of $U$. Define
	\[
	S_b\ :=\ \{s\in S:\ \exists h\in U_b \text{ with } d_\ell(h,s)\le \eps\},\qquad b\in\{0,1\}.
	\]
	Then $S_b$ is an $\eps$-cover of $U_b$, and the previous paragraph implies $S_0\cap S_1=\emptyset$.
	Therefore,
	\[
	|S|\ \ge\ |S_0|+|S_1|\ \ge\ N(U_0,\eps)+N(U_1,\eps).
	\]
	Minimizing over $\eps$-covers $S$ of $U$ gives the claim.
\end{proof}

\begin{lemma}[One-step potential drop]\label{lem:potential-drop}
	Let $u$ be an internal node with version space $U$ and children version spaces $U_0,U_1$.
	There exists $b\in\{0,1\}$ such that
	$
	\Phi(U_b)\ \le\ \Phi(U)\ -\ \frac{\gamma_u}{4c}.
	$
\end{lemma}

\begin{proof}
	Let $I:=(0,\gamma_u/(2c))$. By Lemma~\ref{lem:cover-sum}, for every $\eps\in I$ it holds that
	\[N(U,\eps)\ge N(U_0,\eps)+N(U_1,\eps).\] Hence, for each $\eps\in I$ there exists $b\in\{0,1\}$
	with $N(U_b,\eps)\le N(U,\eps)/2$.
	For $b\in\{0,1\}$ define the measurable set (because $\eps\mapsto N(U,\eps)$ is monotone)
	\[
	A_b\ :=\ \{\eps\in I:\ N(U_b,\eps)\le N(U,\eps)/2\}.
	\]
	Since $A_0\cup A_1=I$, there exists $b$ with $|A_b|\ge |I|/2=\gamma_u/(4c)$, where $|A_b|$ denotes the Lebesgue measure of $A_b$.
	For $\eps\in A_b$ we have $\log_2 N(U_b,\eps)\le \log_2 N(U,\eps)-1,$ while for $\eps\notin A_b$
	we have $\log_2 N(U_b,\eps)\le \log_2 N(U,\eps)$ because $U_b\subseteq U$ implies $N(U_b,\eps)\le N(U,\eps)$.
	Therefore,
	\[
	\begin{aligned}
		\Phi(U_b)
		&= \int_{0}^{\diam(\cH)} \log_2 N(U_b,\eps)\,d\eps 
		\le \int_{0}^{\diam(\cH)} \log_2 N(U,\eps)\,d\eps - \int_{A_b} 1\,d\eps\\
		&= \Phi(U) - |A_b| 
		\le \Phi(U) - \frac{\gamma_u}{4c}.
	\end{aligned}
	\]
	The above gives the statement of the lemma.
\end{proof}

\begin{proof}[of Theorem~\ref{thm:potential-bound-general}]
	Construct a branch $y$ greedily: start at the root $y_0$; at each internal node $y_i$ choose the child
	$y_{i+1}$ whose version space satisfies Lemma~\ref{lem:potential-drop}. Then for every $n\ge 1$,
	telescoping gives
	\[
	\Phi(U_{y_n})\ \le\ \Phi(\cH)\ -\ \sum_{i=0}^{n-1}\frac{\gamma_{y_i}}{4c}.
	\]
	All version spaces along the branch are non-empty by the tree realizability assumption, hence $\Phi(U_{y_n})\ge 0$.
	Rearranging yields for all $n$,
	\[
	\sum_{i=0}^{n-1}\gamma_{y_i}\ \le\ 4c\,\Phi(\cH).
	\]
	If the branch is finite, this proves the claim. If it is infinite, letting $n\to\infty$ (the partial sums are
	non-decreasing and uniformly bounded) gives $\sum_{i\ge 0}\gamma_{y_i}\le 4c\,\Phi(\cH)$.
\end{proof}

\section{Lipschitz Regression under $\ell_q$ Loss}
\label{sec:Lip}

In this section, we give a high level overview of our results for the (nonparametric) class $\cH_L$ of  $L$-Lipschitz functions. Specifically, 
we instantiate our entropy--potential framework for $\cH_L$ under the power loss $\ell_q(y,y'):=|y-y'|^q$. Then, we present an efficient online realizable algorithm for the class, and briefly describe a matching lower bound. Note that the loss $\ell_q$ is a $c$-approximate pseudo-metric with $c=2^{q-1}$, and the induced sup pseudo-metric is
$d_\ell(f,g)=\sup_x \ell_q(f(x),g(x))=\|f-g\|_\infty^q.$

\paragraph{Entropy potential and horizon-free learnability for $q>d$.}
A standard metric-entropy estimate for Lipschitz classes in $\|\cdot\|_\infty$ yields
$\log N(\cH_L,\delta;\|\cdot\|_\infty)\lesssim (L/\delta)^d\log(1/\delta)$
(e.g., Lemma~\ref{lem:lip-cover}, based on \cite[Thm.~17]{vonLuxburgBousquet2004}).
Since an $\eps$-cover in $d_\ell$ is exactly an $\eps^{1/q}$-cover in $\|\cdot\|_\infty$, we obtain
$\log N(\cH_L,\eps)\lesssim L^d\,\eps^{-d/q}\log(1/\eps)$ for $\eps\in(0,1]$.
Consequently,
\[
\Phi(\cH_L)\;=\;\int_0^1 \log N(\cH_L,\eps)\,d\eps
\;\lesssim\;
L^d\int_0^1 \eps^{-d/q}\log(1/\eps)\,d\eps
\;<\;\infty
\quad\Longleftrightarrow\quad q>d.
\]
Plugging this into \Cref{thm:intro_Donl-via-Phi} gives $\mathbb D_{\mathrm{onl}}(\cH_L)\lesssim_{d,q} L^d$; hence, a horizon-free minimax realizable cumulative-loss bound of order $O_{d,q}(L^d)$
(via the characterization of \cite{attias2023optimal}).
A horizon-free baseline lower bound $\mathbb D_{\mathrm{onl}}(\cH_L)\ge (2L)^d$ holds for all $q\ge 1$
(Corollary~\ref{cor:ld-lb}); thus, in the learnable regime $q>d$ we obtain $\mathbb D_{\mathrm{onl}}(\cH_L)=\Theta_{d,q}(L^d)$. We note that the constants are generally exponential in $d$ and $q$. 

\paragraph{An efficient proper strategy (constructive $O_{d,q}(L^d)$ for $q>d$).}
Beyond the non-constructive entropy argument, we give a simple deterministic proper algorithm based on
\emph{Lipschitz envelopes}.
Given past data $S_{t-1}=\{(x_s,y_s)\}_{s<t}$, define pointwise lower/upper envelopes
\[
\underline h_t(x)=\max\Big\{0,\max_{s<t}\big(y_s-L\|x-x_s\|_\infty\big)\Big\},\qquad
\overline h_t(x)=\min\Big\{1,\min_{s<t}\big(y_s+L\|x-x_s\|_\infty\big)\Big\},
\]
and predict the midpoint $\hat y_t=\tfrac12(\underline h_t(x_t)+\overline h_t(x_t))$.
Realizability implies $y_t\in[\underline h_t(x_t),\overline h_t(x_t)]$, so the instantaneous error is controlled by
the envelope width $w_t:=\overline h_t(x_t)-\underline h_t(x_t)$.
For $q>d$ we analyze the potential $\Psi_t=\int_{\cX}(\overline h_t-\underline h_t)^{\,q-d}\,dx$:
each round in which $w_t$ is large forces a \emph{geometric tightening} of the envelopes on a cube around $x_t$,
causing a definite drop in $\Psi_t$ of size $\gtrsim w_t^q/L^d$.
Telescoping from $\Psi_1=2^d$ gives $\sum_{t\ge 1}|y_t-\hat y_t|^q\lesssim_{d,q} L^d$.
Each prediction can be computed in time $O(td)$ by scanning $S_{t-1}$; we did not attempt to optimize the running time further. 
The full proof and constants appear in Appendix~\ref{sec:lipschitz-phase-transition}. 

\paragraph{Eventual learnability for any continuous loss.}
The same envelope mechanism yields a strong pointwise guarantee that is \emph{loss-agnostic}:
for every $\varepsilon\in(0,1]$,
\[
\#\{t:\ |\hat y_t-y_t|>\varepsilon\}\ \le\ (8L/\varepsilon)^d,
\qquad\text{hence}\qquad |\hat y_t-y_t|\to 0.
\]
As a corollary, for any continuous loss $\ell$ with $\ell(y,y)=0$, we have $\ell(\hat y_t,y_t)\to 0$
on every realizable sequence (Corollary~\ref{cor:env-pointwise-loss}).

\paragraph{The critical and subcritical regimes $q\le d$.}
At the critical exponent $q=d$, the envelope strategy satisfies a sharp horizon-dependent bound:
$
\sum_{t=1}^T |y_t-\hat y_t|^d\ \lesssim_d\ L^d(1+\log T),
$
proved by combining the $\varepsilon$-mistake bound above with a layer-cake identity.
We also prove a matching lower bound
$\sum_{t=1}^T |\hat y_t-y_t|^d \gtrsim_d L^d\log(1+T/L^d)$
via an explicit multiscale adversary built from dyadic cube refinements and a McShane extension argument,
showing the logarithmic dependence is inherent.
For $q<d$, we construct a realizable scaled Littlestone tree on a $T$-point $T^{-1/d}$-separated set, yielding
$
\mathbb D_{\mathrm{onl}}(\cH_L)\ \gtrsim\ (2L)^q\,T^{1-q/d},
$
and therefore polynomially growing minimax realizable cumulative loss.
In particular, it holds that $\mathbb D_{\mathrm{onl}}(\cH_L)=\infty$ for all $q\le d$, matching the phase diagram in
Figure~\ref{fig:lipschitz-phase}.

\newpage

\bibliography{bib2}
\bibliographystyle{plain}

\appendix

\section{Lipschitz regression under $\ell_q$ losses}
\label{sec:lipschitz-phase-transition}

In this section, we specialize the general entropy--potential framework to the hypothesis class of
Lipschitz functions on the $d$-dimensional cube, and to the power loss
\[
\ell_q(y,y') \ :=\ |y-y'|^q,\qquad q\ge 1 .
\]
We focus on the regimes $q>d$ (finite horizon-free realizable loss) and $q<d$ (unbounded realizable loss).
The critical case $q=d$ is discussed only briefly at the end.

\paragraph{Setup.}
Let $\cX=[-1,1]^d$ and let $\cY=[0,1]$.
For $L\ge 1$ define the Lipschitz class
\[
\cH_L \ :=\ \Big\{h:\cX\to\cY:\ |h(x)-h(x')|\le L\|x-x'\|_\infty\ \ \forall x,x'\in\cX\Big\}.
\]
Fix $q\ge 1$ and $\ell=\ell_q$.
Then the induced pseudo-metric on $\cH_L$ is
\[
d_\ell(f,g)\ =\ \sup_{x\in\cX}\ell_q\big(f(x),g(x)\big)\ =\ \|f-g\|_\infty^q,
\]
so $\diam(\cH_L)\le 1$.

\paragraph{Approximate triangle inequality.}
For $q\ge 1$, $\ell_q$ is a $c$-approximate pseudo-metric with
\begin{equation}\label{eq:cq}
	c\ :=\ 2^{q-1},
\end{equation}
since for all $a,b,c'\in\R$,
\(
|a-b|^q \le (|a-c'|+|c'-b|)^q \le 2^{q-1}(|a-c'|^q+|c'-b|^q).
\)
Therefore all results for approximate pseudo-metrics apply with this constant $c$.

\subsection{Entropy potential upper bound for $q>d$}

We first upper bound the metric entropy of $\cH_L$ in $\|\cdot\|_\infty$.
The following bound is standard; we quote it in the form that follows from
\cite[Thm.~17]{vonLuxburgBousquet2004} by plugging in $\cX=[-1,1]^d$ with $\|\cdot\|_\infty$.

\begin{lemma}[Covering numbers of Lipschitz classes]\label{lem:lip-cover}
	There exists a universal constant $C_0\ge 1$ such that for every $L\ge 1$ and every $\delta\in(0,1]$,
	\[
	\log_2 N\big(\cH_L,\delta;\|\cdot\|_\infty\big)
	\ \le\
	\Big(\frac{8L}{\delta}\Big)^d \cdot \log_2\Big(\frac{C_0}{\delta}\Big),
	\]
	where $N(\cH_L,\delta;\|\cdot\|_\infty)$ denotes the $\delta$-covering number under $\|f-g\|_\infty$.
\end{lemma}

\begin{proof}
	By \cite[Thm.~17]{vonLuxburgBousquet2004}, for $\delta\in(0,1]$,
	\[
	N\big(\cH_L,\delta;\|\cdot\|_\infty\big)
	\ \le\
	\Big(2\Big\lceil\frac{2\,\diam(\cX)}{\delta}\Big\rceil+1\Big)^{\,N(\cX,\delta/(4L))}.
	\]
	Here $\diam(\cX)=\sup_{x,x'}\|x-x'\|_\infty=2$, and $\cX$ can be covered in $\|\cdot\|_\infty$
	by at most $(\lceil 4L/\delta \rceil)^d$ cubes of side $\delta/(2L)$, hence $N(\cX,\delta/(4L))\le (\lceil 4L/\delta \rceil)^d \leq (8L/\delta)^d$, where the last inequality uses $L, \frac{1}{\delta} \geq 1$.
	Taking $\log_2$ and using $2\lceil 4/\delta\rceil+1 \le C_0/\delta$ for a universal $C_0$ yields the claim.
\end{proof}

Since $d_\ell(f,g)=\|f-g\|_\infty^q$, an $\eps$-cover in $d_\ell$ is the same as an
$\eps^{1/q}$-cover in $\|\cdot\|_\infty$, hence
\[
N(\cH_L,\eps)\ =\ N\big(\cH_L,\eps^{1/q};\|\cdot\|_\infty\big).
\]

\begin{theorem}[Finite potential for $q>d$]\label{thm:qgt-d}
	Fix $d\ge 1$, $L\ge 1$, and $q>d$. Then $\Phi(\cH_L)<\infty$ and, moreover,
	\[
	\Phi(\cH_L)\ \le\ C_{d,q}\,L^d
	\]
	for a constant $C_{d,q}<\infty$ depending only on $(d,q)$.
	Consequently,
	\[
	\mathbb D_{\mathrm{onl}}(\cH_L)\ \le\ 4c\,\Phi(\cH_L)\ \le\ 4\cdot 2^{q-1}\,C_{d,q}\,L^d.
	\]
\end{theorem}
In particular, by the characterization of realizable cumulative loss in terms of $\mathbb D_{\mathrm{onl}}$
\cite{attias2023optimal}, the minimax realizable cumulative loss for $\cH_L$ under $\ell_q$
admits a horizon-free upper bound of order $O_{d,q}(L^d)$.

\begin{proof}
	By Lemma~\ref{lem:lip-cover} and the identity $N(\cH_L,\eps)=N(\cH_L,\eps^{1/q};\|\cdot\|_\infty)$,
	for $\eps\in(0,1]$,
	\[
	\log_2 N(\cH_L,\eps)
	\ \le\
	\Big(\frac{8L}{\eps^{1/q}}\Big)^d\cdot \log_2\Big(\frac{C_0}{\eps^{1/q}}\Big)
	\ =\
	(8^dL^d)\,\eps^{-d/q}\Big(\log_2 C_0+\frac{1}{q}\log_2\frac{1}{\eps}\Big).
	\]
	Therefore
	\[
	\Phi(\cH_L)
	\ =\ \int_{0}^{1}\log_2 N(\cH_L,\eps)\,d\eps
	\ \le\ 8^dL^d\Big(\log_2 C_0\int_0^1 \eps^{-d/q}\,d\eps\ +\ \frac{1}{q}\int_0^1 \eps^{-d/q}\log_2\frac{1}{\eps}\,d\eps\Big).
	\]
	Since $q>d$, we have $d/q<1$, and both integrals are finite; e.g.,
	$\int_0^1 \eps^{-d/q}\,d\eps=1/(1-d/q)$ and
	$\int_0^1 \eps^{-d/q}\log(1/\eps)\,d\eps = 1/(1-d/q)^2$.
	This yields $\Phi(\cH_L)\le C_{d,q}L^d$ for an explicit finite $C_{d,q}$.
	The bound on $\mathbb D_{\mathrm{onl}}(\cH_L)$ follows from \Cref{thm:intro_Donl-via-Phi} and \eqref{eq:cq}.
	Finally, \cite{attias2023optimal} gives the stated implication for realizable cumulative loss.
\end{proof}

\subsection{An efficient strategy (constructive bound for $q>d$).}
The entropy--potential proof above is non-constructive. For the above family of Lipschitz functions, we can obtain the same
$O_{d,q}(L^d)$ horizon-free bound via an explicit efficient strategy based on Lipschitz envelopes.

Consider the following strategy. Let $S_{t-1}:=\{(x_s,y_s)\}_{s<t}$.
Define the lower/upper envelopes
\begin{align*}
	\underline h_t(x)
	&:= \max\Big\{0,\ \max_{s<t}\big(y_s - L\|x-x_s\|_\infty\big)\Big\},\\
	\overline h_t(x)
	&:= \min\Big\{1,\ \min_{s<t}\big(y_s + L\|x-x_s\|_\infty\big)\Big\},
\end{align*}
(with the conventions $\max_{s<t}(\cdot)=-\infty$ and $\min_{s<t}(\cdot)=+\infty$ when $t=1$).
Upon receiving $x_t$, predict
\[
\hat y_t \ :=\ \frac{\underline h_t(x_t)+\overline h_t(x_t)}{2}\in[0,1],
\]
then observe $y_t$ and continue. The lemma below shows the correctness of this approach. 
\begin{lemma}\label{cor:efficient-lip}
	For every realizable sequence and every $q>d$,
	\[
	\sum_{t\ge 1} |y_t-\hat y_t|^q \ \le\ C'_{d,q}\,L^d
	\]
	for a finite constant $C'_{d,q}$ depending only on $(d,q)$.
	Moreover, each prediction can be computed in time $O(td)$ by evaluating the max/min over $S_{t-1}$.
\end{lemma}

\begin{proof}
	Fix $t$ and write the width function
	\[
	W_t(x)\ :=\ \overline h_t(x)-\underline h_t(x),\qquad w_t:=W_t(x_t).
	\]
	Since $\overline h_t$ is a pointwise minimum of $L$-Lipschitz functions and $\underline h_t$ is a pointwise
	maximum of $L$-Lipschitz functions (both clipped to $[0,1]$), we have that $W_t$ is $2L$-Lipschitz.
	Realizability implies $\underline h_t(x)\le f^\star(x)\le \overline h_t(x)$ for all $x$, hence
	$y_t=f^\star(x_t)\in[\underline h_t(x_t),\overline h_t(x_t)]$ and therefore
	\begin{equation}\label{eq:loss-vs-width}
		|y_t-\hat y_t|^q \ \le\ (w_t/2)^q .
	\end{equation}
	
	Let $\lambda$ denote Lebesgue measure on $\cX$, and define the potential (well-defined since $q>d$)
	\[
	\Psi_t \ :=\ \int_{\cX} W_t(x)^{\,q-d}\,d\lambda(x).
	\]
	Since the envelopes tighten, for all $x\in\cX$ we have $\underline h_{t+1}(x)\ge \underline h_t(x)$ and
	$\overline h_{t+1}(x)\le \overline h_t(x)$, hence $W_{t+1}(x)\le W_t(x)$.
	As $q-d>0$, this implies $W_{t+1}(x)^{q-d}\le W_t(x)^{q-d}$ and therefore $\Psi_{t+1}\le \Psi_t$.
	Set $r_t:=w_t/(8L)$ and $B_t:=\{x\in\cX:\ \|x-x_t\|_\infty\le r_t\}$.
	Since $B_t=\cX\cap (x_t+[-r_t,r_t]^d)$, each coordinate interval has length at least $r_t$, hence
	\[
	\lambda(B_t)\ \ge\ r_t^d\ =\ \Big(\frac{w_t}{8L}\Big)^d .
	\]
	By $2L$-Lipschitzness, for all $x\in B_t$ we have
	$W_t(x)\ge w_t-2Lr_t=(3/4)w_t$.
	After observing $(x_t,y_t)$, any consistent $h\in\cH_L$ must satisfy
	$|h(x)-y_t|\le L\|x-x_t\|_\infty$, hence for all $x$,
	\[
	\overline h_{t+1}(x)\le y_t+L\|x-x_t\|_\infty,\qquad
	\underline h_{t+1}(x)\ge y_t-L\|x-x_t\|_\infty,
	\]
	and thus on $B_t$, $W_{t+1}(x)\le 2Lr_t=w_t/4$. Consequently,
	\begin{align*}
		\Psi_t-\Psi_{t+1}
		&\ge \int_{B_t}\Big(W_t(x)^{q-d}-W_{t+1}(x)^{q-d}\Big)\,d\lambda(x)\\
		&\ge \lambda(B_t)\Big( (3w_t/4)^{q-d}-(w_t/4)^{q-d}\Big)\\
		%
		%
		&\ge \Big(\frac{w_t}{8L}\Big)^d \Big( (3/4)^{q-d}-(1/4)^{q-d}\Big)\,w_t^{q-d}
		\ =:\ \frac{c_{d,q}}{L^d}\,w_t^q,
	\end{align*}
	where $c_{d,q}:=8^{-d}\big((3/4)^{q-d}-(1/4)^{q-d}\big)>0$.
	Summing over $t$ and using $\Psi_t\ge 0$ gives
	\[
	\sum_{t\ge 1} w_t^q \ \le\ \frac{L^d}{c_{d,q}}\Psi_1.
	\]
	At $t=1$ we have $\underline h_1\equiv 0$, $\overline h_1\equiv 1$, hence $W_1\equiv 1$ and
	$\Psi_1=\lambda(\cX)=2^d$. Combining with \eqref{eq:loss-vs-width} yields
	\[
	\sum_{t\ge 1} |y_t-\hat y_t|^q
	\ \le\ 2^{-q}\sum_{t\ge 1} w_t^q
	\ \le\ 2^{-q}\cdot \frac{2^d}{c_{d,q}}\,L^d
	\ =:\ C'_{d,q}\,L^d.
	\]
	The above gives the proof of the lemma. 
\end{proof}

\subsection{Implications of the Efficient Algorithm}

We also have the following interesting results as a corollary of the above efficient algorithm. 

\begin{corollary}[Pointwise consistency of the envelope strategy]\label{cor:env-pointwise}
	For every $d\ge 1$ and $L\ge 1$, for every realizable sequence, the envelope strategy satisfies
	\[
	|\hat y_t-y_t|\ \longrightarrow\ 0\qquad\text{as }t\to\infty .
	\]
	Moreover, for every $\varepsilon\in(0,1]$,
	\[
	\#\{t\ge 1:\ |\hat y_t-y_t|>\varepsilon\}\ \le\ \Big(\frac{8L}{\varepsilon}\Big)^d.
	\]
\end{corollary}

\begin{proof}
	Let $W_t=\overline h_t-\underline h_t$ and $w_t:=W_t(x_t)$. Since
	$y_t\in[\underline h_t(x_t),\overline h_t(x_t)]$ and
	$\hat y_t=\tfrac12(\underline h_t(x_t)+\overline h_t(x_t))$, we have $|\hat y_t-y_t|\le w_t/2$; hence
	$|\hat y_t-y_t|>\varepsilon$ implies $w_t>2\varepsilon$.
	Fix $\varepsilon\in(0,1]$ and define $A_t:=\{x\in\cX:\ W_t(x)>\varepsilon\}$.
	Because the envelopes tighten pointwise, $W_{t+1}\le W_t$ and thus $A_{t+1}\subseteq A_t$.
	
	On the event $w_t>2\varepsilon$, set $r:=\varepsilon/(4L)$ and $B_t:=\{x\in\cX:\ \|x-x_t\|_\infty\le r\}$.
	As $W_t$ is $2L$-Lipschitz, for all $x\in B_t$,
	$W_t(x)\ge W_t(x_t)-2Lr>2\varepsilon-\varepsilon/2>\varepsilon$, hence $B_t\subseteq A_t$.
	After observing $(x_t,y_t)$, the constraint at time $t$ implies
	$W_{t+1}(x)\le 2L\|x-x_t\|_\infty\le 2Lr=\varepsilon/2$ for all $x\in B_t$, so $B_t\cap A_{t+1}=\emptyset$.
	Therefore $\lambda(A_t)-\lambda(A_{t+1})\ge \lambda(B_t)$.
	Since $\cX=[-1,1]^d$, each coordinate interval in $B_t=\cX\cap(x_t+[-r,r]^d)$ has length at least $r$, hence
	$\lambda(B_t)\ge r^d$.
	
	Summing over all $t$ with $|\hat y_t-y_t|>\varepsilon$ yields
	\[
	\#\{t:\ |\hat y_t-y_t|>\varepsilon\}\cdot r^d \ \le\ \lambda(A_1)\ \le\ \lambda(\cX)=2^d,
	\]
	so $\#\{t:\ |\hat y_t-y_t|>\varepsilon\}\le 2^d/r^d=(8L/\varepsilon)^d$.
	Since this holds for every $\varepsilon>0$, it follows that $|\hat y_t-y_t|\to 0$.
\end{proof}

From the above, we can prove Corollary~\ref{cor:env-pointwise-loss}

\begin{proof}[Proof of Corollary~\ref{cor:env-pointwise-loss}:]
	By Corollary~\ref{cor:env-pointwise}, $|\hat y_t-y_t|\to 0$. Since $\ell$ is continuous on the compact
	set $[0,1]^2$ and vanishes on the diagonal $\{(y,y):y\in[0,1]\}$, for every $\varepsilon>0$ there exists
	$\delta>0$ such that $|u-v|<\delta$ implies $\ell(u,v)<\varepsilon$. Hence $\ell(\hat y_t,y_t)\to 0$.
\end{proof}

In addition, we have the following corollary:

\begin{corollary}[Critical horizon-dependent bound for $q=d$]\label{cor:lip-qeqd-logT}
	Fix $d\ge 1$ and $L\ge 1$. For every realizable sequence and every horizon $T\ge 2$,
	the envelope strategy satisfies
	\[
	\sum_{t=1}^T |y_t-\hat y_t|^d \ \le\ C_d\,L^d\,(1+\log T),
	\]
	for a constant $C_d<\infty$ depending only on $d$ (e.g.\ $C_d=8^d$).
\end{corollary}

\begin{proof}
	Let $e_t:=|y_t-\hat y_t|\in[0,1]$ and for $\varepsilon\in(0,1]$ write
	$N_T(\varepsilon):=\#\{t\le T:\ e_t>\varepsilon\}$.
	By Corollary~\ref{cor:env-pointwise}, $N_T(\varepsilon)\le \min\{T,(8L/\varepsilon)^d\}$.
	Using the layer-cake identity,
	\[
	\sum_{t=1}^T e_t^d \ =\ d\int_0^1 \varepsilon^{d-1} N_T(\varepsilon)\,d\varepsilon
	\ \le\ d\int_0^1 \varepsilon^{d-1}\min\{T,(8L/\varepsilon)^d\}\,d\varepsilon .
	\]
	Let $\varepsilon_0:=\min\{1,\,8L\,T^{-1/d}\}$. Then
	\begin{align*}
		\sum_{t=1}^T e_t^d
		&\le d\int_0^{\varepsilon_0} \varepsilon^{d-1}T\,d\varepsilon
		\;+\; d\int_{\varepsilon_0}^1 \varepsilon^{d-1}(8L/\varepsilon)^d\,d\varepsilon \\
		&= T\varepsilon_0^d \;+\; d(8L)^d\int_{\varepsilon_0}^1 \varepsilon^{-1}\,d\varepsilon
		\ =\ T\varepsilon_0^d \;+\; d(8L)^d\log(1/\varepsilon_0).
	\end{align*}
	If $\varepsilon_0=1$, then $T\le (8L)^d$ and the bound is at most $(8L)^d$.
	Otherwise $\varepsilon_0=8L\,T^{-1/d}$, so $T\varepsilon_0^d=(8L)^d$ and
	$\log(1/\varepsilon_0)=\log(T^{1/d}/(8L))\le \frac{1}{d}\log T$ (since $\varepsilon_0\le 1$),
	giving $\sum_{t=1}^T e_t^d \le (8L)^d(1+\log T)$.
\end{proof}

\subsection{A realizable lower bound for $q\leq d$}

We next show that when $q = d$ and then $q<d$, the online dimension grows logarithmically and polynomially with the horizon, respectively. 
Hence, realizable cumulative loss is unbounded when $q \leq d$.

\begin{lemma}[Critical lower bound for $q=d$]
	Fix $d\ge 1$ and $L\ge 1$ on $\cX=[-1,1]^d$ with loss $\ell_d(\hat y,y)=|\hat y-y|^d$.
	For every horizon $T\ge 2$, for any (possibly randomized) online learner producing $\hat y_t\in[0,1]$,
	there exists a realizable sequence $(x_t,y_t)_{t=1}^T$ by $\cH_L$ such that
	\[
	\sum_{t=1}^T |\hat y_t-y_t|^d \ \ge\ c_d\,L^d\,\log\!\Big(1+\frac{T}{L^d}\Big),
	\]
	for a constant $c_d>0$ depending only on $d$.
\end{lemma}

\begin{proof}
	For $j\ge 0$ let $a_j:=2^{-j}/L$ and let $\mathcal Q_j$ be the axis-aligned grid partition of $[-1,1]^d$
	into cubes of side $a_j$ (intersected with $[-1,1]^d$), so that $\mathcal Q_{j+1}$ refines $\mathcal Q_j$.
	Let $M_j:=\lfloor 2/a_j\rfloor^d$ and let $x_{j,1},\dots,x_{j,M_j}$ be the corresponding cube centers.
	Set $\delta_j:=2^{-j-2}$, so $\sum_{j\ge 0}\delta_j\le 1/2$.
	
	\paragraph{Adversary.}
	Query points level-by-level: $x_{0,1},\dots,x_{0,M_0}$, then $x_{1,1},\dots,x_{1,M_1}$, etc., stopping after $T$ rounds
	(the last level may be partial).
	Introduce a root cube $Q_{-1}:=[-1,1]^d$ and define $v(Q_{-1})=1/2$.
	Whenever the adversary queries a center $x_t$ belonging to a cube $Q\in\mathcal Q_j$ not queried before,
	after seeing $\hat y_t$ define
	\[
	y_t:=v(Q):=v(\mathrm{parent}(Q))+s(Q)\delta_j,\qquad s(Q)\in\{\pm1\},
	\]
	where $\mathrm{parent}(Q)$ is the unique cube in $\mathcal Q_{j-1}$ containing $Q$ (and $\mathrm{parent}(Q):=Q_{-1}$ if $j=0$).
	Choose the sign $s(Q)$ so that $|\hat y_t-y_t|\ge \delta_j/2$ (possible since the two candidates are $2\delta_j$ apart).
	Also $y_t\in[0,1]$ because $v(Q_{-1})=1/2$ and the total variation along any chain is at most $\sum_j \delta_j\le 1/2$.
	
	\paragraph{Realizability.}
	Let $S_T:=\{x_t:t\le T\}$ and define $y(x_t):=y_t$.
	If $x,x'\in S_T$ have least common ancestor at level $m$, then their increments can differ by at most $2\delta_k$
	per level $k>m$, hence
	\[
	|y(x)-y(x')|\le 2\sum_{k>m}\delta_k = 2^{-m-1}.
	\]
	Since they split at level $m+1$, their centers satisfy $\|x-x'\|_\infty\ge a_{m+1}=2^{-(m+1)}/L$, so
	$$|y(x)-y(x')|\le 2^{-m-1}=L a_{m+1}\le L\|x-x'\|_\infty.$$
	Thus $y$ is $L$-Lipschitz on $S_T$, and by \cite{mcshane1934} there is an $L$-Lipschitz extension $\tilde f$ to $\cX$.
	Clipping $f^\star:=\clip_{[0,1]}(\tilde f)$ keeps the same values on $S_T$ and preserves $L$-Lipschitzness,
	so $f^\star\in\cH_L$ realizes the sequence.
	
	\paragraph{Loss lower bound.}
	Whenever a fresh level-$j$ cube is queried we ensured $|\hat y_t-y_t|\ge \delta_j/2$, hence
	$|\hat y_t-y_t|^d\ge (\delta_j/2)^d = 2^{-d(j+3)}$.
	Let $J$ be such that the first $J$ levels contribute at least one full level, i.e.
	\[
	N_J:=\sum_{j=0}^{J-1} M_j \ \le\ T \ <\ N_{J+1}:=\sum_{j=0}^{J} M_j .
	\]
	Then the $N_J$ rounds of fully completed levels yield
	\[
	\sum_{t=1}^T |\hat y_t-y_t|^d \ \ge\ \sum_{j=0}^{J-1} M_j\,2^{-d(j+3)}.
	\]
	Using $\lfloor 2/a_j\rfloor \ge 2/a_j-1=2^{j+1}L-1\ge 2^jL$ (since $L\ge 1$), we have $M_j\ge (2^jL)^d$, hence
	each full level contributes at least $2^{-3d}L^d$, and therefore
	\begin{equation}\label{eq:crit-lb-levels}
		\sum_{t=1}^T |\hat y_t-y_t|^d \ \ge\ 2^{-3d}L^d\,J.
	\end{equation}
	
	Finally, relate $J$ and $T$.
	If $J=0$, then all $T$ rounds are at level $0$ and $$\sum_{t=1}^T|\hat y_t-y_t|^d\ge T(\delta_0/2)^d=2^{-3d}T
	\ge 2^{-3d}L^d\log(1+T/L^d)$$ using $\log(1+u)\le u$.
	Assume $J\ge 1$. Since $T<N_{J+1}$ and $M_j\le (2/a_j)^d=(2^{j+1}L)^d$,
	\[
	T \ <\ \sum_{j=0}^{J} (2^{j+1}L)^d
	\ =\ 2^d L^d\,\frac{2^{d(J+1)}-1}{2^d-1}
	\ \le\ C_d\,L^d\,2^{d(J+1)}
	\]
	for $C_d:=\frac{2^d}{2^d-1}$. Hence
	\[
	\log\!\Big(1+\frac{T}{L^d}\Big)\ \le\ \log(1+C_d)+d(J+1)\log 2 \ \le\ K_d\,J,
	\]
	where $K_d:=\log(1+C_d)+2d\log 2$ (using $J+1\le 2J$ for $J\ge 1$).
	Combining with \eqref{eq:crit-lb-levels} gives the claim with $c_d:=2^{-3d}/K_d$.
\end{proof}

\begin{lemma}\label{lem:separated-set}
	For every integer $T\ge 1$, there exists a set $\{x_1,\dots,x_T\}\subseteq [-1,1]^d$ such that
	$\|x_t-x_s\|_\infty\ge 2T^{-1/d}$ for all $t\neq s$.
\end{lemma}
\begin{proof}
	Let $m:=\lfloor T^{1/d}\rfloor$ and consider the grid
	\(
	G:=\{-1,-1+2/m,\dots,1\}^d.
	\)
	Then $|G|=(m+1)^d\ge T$, and distinct points in $G$ differ by at least $2/m$ in $\|\cdot\|_\infty$.
	Since $m\le T^{1/d}$, we have $2/m\ge 2T^{-1/d}$.
	Then, any $T$ distinct points from $G$ satisfy the claim.
\end{proof}

\begin{theorem}[Unbounded online dimension for $q<d$]\label{thm:qlt-d}
	Fix $d\ge 1$, $L\ge 1$, and $q<d$.
	Then for every horizon $T\ge (2L)^d$,
	\[
	\mathbb D_{\mathrm{onl}}(\cH_L)\ \ge\ (2L)^q\,T^{1-q/d}.
	\]
\end{theorem}

Consequently, the minimax realizable cumulative loss for $\cH_L$ under $\ell_q$ is unbounded as $T\to\infty$
(and in fact grows at least on the order of $L^qT^{1-q/d}$), by \cite{attias2023optimal}.

\begin{proof}
	Fix $T\ge (2L)^d$ and let $x_1,\dots,x_T$ be as in Lemma~\ref{lem:separated-set}, so
	\(
	\|x_t-x_s\|_\infty\ge r:=2T^{-1/d}
	\)
	for $t\neq s$.
	Define $\Delta:=Lr=2LT^{-1/d}\le 1$ (the last inequality uses $T\ge (2L)^d$).
	
	We build a scaled Littlestone tree of depth $T$ as follows.
	At depth $t-1$ we query $x_t$, and the two outgoing edges are labeled by $s_{t,0}:=0$ and $s_{t,1}:=\Delta$.
	Thus every internal node has gap
	\(
	\gamma=\ell_q(0,\Delta)=\Delta^q.
	\)
	
	It remains to verify realizability of the tree: for every branch $b\in\{0,1\}^T$,
	define values on the finite set $\{x_1,\dots,x_T\}$ by $f_b(x_t):=b_t\Delta$.
	For any $t\neq s$,
	\[
	|f_b(x_t)-f_b(x_s)|\ \le\ \Delta\ =\ Lr\ \le\ L\|x_t-x_s\|_\infty,
	\]
	By the McShane extension theorem \cite{mcshane1934}, $f_b$ extends to an $L$-Lipschitz function
	$\tilde f_b:\cX\to\R$ satisfying $\tilde f_b(x_t)=f_b(x_t)=b_t\Delta$ for all $t$.
	Define $\hat f_b:\cX\to[0,1]$ by $\hat f_b(x):=\min\{1,\max\{0,\tilde f_b(x)\}\}$.
	Since clipping is $1$-Lipschitz, $\hat f_b$ is still $L$-Lipschitz, hence $\hat f_b\in\cH_L$,
	and $\hat f_b(x_t)=b_t\Delta$ for all $t$.
	Thus the tree is realizable by $\cH_L$.
	Along every branch, the sum of gaps equals $T\Delta^q$, hence
	\[
	\inf_{y\in\cP(\tree)}\sum_{i}\gamma_{y_i}\ =\ T\Delta^q\ =\ T(Lr)^q\ =\ T\,(2LT^{-1/d})^q\ =\ (2L)^q\,T^{1-q/d}.
	\]
	Taking the supremum over trees in the definition of $\mathbb D_{\mathrm{onl}}(\cH_L)$ yields the stated lower bound.
	Finally, \cite{attias2023optimal} converts this lower bound on $\mathbb D_{\mathrm{onl}}$ into a lower bound on
	minimax realizable cumulative loss (up to universal constants depending only on $c$, hence on $q$).
\end{proof}

\begin{corollary}[An $\Omega(L^d)$ lower bound (hence unavoidable for $q>d$)]
	\label{cor:ld-lb}
	Fix $d\ge 1$, $L\ge 1$, and $q\ge 1$. Then
	\[
	\mathbb D_{\mathrm{onl}}(\cH_L)\ \ge\ (2L)^d .
	\]
	Consequently, the minimax realizable cumulative loss under $\ell_q$ is $\Omega(L^d)$
	(up to universal constants depending only on $q$), by \cite{attias2023optimal}.
\end{corollary}

\begin{proof}
	Let $T:=(2L)^d$. By Lemma~\ref{lem:separated-set} there exist $x_1,\dots,x_T\in[-1,1]^d$ with
	$\|x_t-x_s\|_\infty\ge r:=2T^{-1/d}=1/L$ for all $t\neq s$. Set $\Delta:=Lr=1$.
	
	Consider the depth-$T$ scaled Littlestone tree that queries $x_t$ at depth $t-1$ and labels the two
	outgoing edges by $0$ and $\Delta$. Each node has gap $\gamma=\ell_q(0,\Delta)=1$.
	For any branch $b\in\{0,1\}^T$, define $f_b(x_t):=b_t\Delta\in\{0,1\}$. Then for $t\neq s$,
	$|f_b(x_t)-f_b(x_s)|\le 1=\Delta\le L\|x_t-x_s\|_\infty$, so $f_b$ extends (by McShane and clipping)
	to some $h_b\in\cH_L$ realizing the branch. Hence the tree is realizable and every root-to-leaf path
	has total gap $T$, so $\mathbb D_{\mathrm{onl}}(\cH_L)\ge T=(2L)^d$.
\end{proof}

\subsection{Compact and Totally Bounded Classes}
\label{app:total}

We remind the reader that for the class $\cH_L$ of $L$-Lipschitz functions, the entropy bound becomes borderline at $q=d$ and we show that the total loss grows logarithmically with the horizon $T$, where for $q<d$ the total loss grows linearly with $T$.
The above regimes between $q$ and $d$ give a good example to the fact that bounded loss is metric-dependent. Therefore, we ask an abstract question: under what conditions does $\lim_{t \rightarrow \infty} \ell(\hat y_t,y_t) = 0$ hold for every reasonable loss $\ell$ (not necessarily a pseudo-metric)?
In the agnostic online setting under absolute loss, online learnability admits a general characterization via finiteness of the \emph{sequential} $\eps$-fat shattering dimension for every $\eps>0$ \cite{rakhlin2015online}.
However, if $\cH$ is totally bounded with respect to $d_\ell(f,g)=\sup_{x\in\cX}\ell(f(x),g(x))$ (i.e., for every $\eps>0$ there exists a finite $\eps$-net under $d_\ell$),
then we obtain a simpler \emph{proper} and deterministic construction: for each fixed $\eps$ we take a finite $\eps$-net and run an elimination procedure, yielding at most $N(\cH,d_\ell,\eps)$ rounds with $\ell(\hat y_t,y_t)>\eps$; applying this for $\eps\downarrow 0$ implies $\ell(\hat y_t,y_t)\to 0$.

\begin{theorem}
	\label{thm:total}
	Let $\cY$ be any label space, let $\ell:\cY^2 \to \mathbb{R}_{\ge 0}$ be any loss function, and let $\cH \subseteq \cY^{\cX}$ be a function class.
	Assume that $\cH$ is \emph{totally bounded} with respect to
	$
	d_\ell(f,g)=\sup_{x \in \cX} \ell\bigl(f(x),g(x)\bigr).
	$
	Then, there is a deterministic proper algorithm that for any realizable sequence $(x_t,y_t)_{t \geq 1}$ returns $(\hat y_t)_{t \geq 1}$,
	which for every $\eps > 0$ satisfies
	$
	\left|\left\{t \mid \ell(\hat y_t,y_t) > \eps \right\}\right|\leq N(\cH,d_{\ell},\eps)-1;
	$
	hence, it follows that $\lim_{t \rightarrow \infty} \ell(\hat y_t,y_t) = 0$.
\end{theorem}
\begin{proof}
	Note that $\lim_{t \rightarrow \infty} \ell(\hat y_t,y_t) = 0$ is equivalent to having for every $\eps>0$ only finitely many times $t$
	for which $\ell(\hat y_t,y_t) > \eps$. Fix some $\eps > 0$. Let $N = N(\cH,d_{\ell},\eps)$. Since $\cH$ is totally bounded, there exists a finite $\eps$-net
	$\{f_1,\dots,f_N\}\subseteq \cH$ with respect to $d_{\ell}$, i.e., for every $f\in\cH$ there is an $i$ such that
	$d_\ell(f,f_i)\le \eps$.
	Select such representatives $\{f_i\}_{i=1}^N$. Upon receiving $x_t$, pick any index $i$ not yet eliminated and predict
	$\hat{y_t}:=f_i(x_t)$; if $\ell(\hat y_t,y_t) > \eps$, then eliminate $f_i$ (and never use it again).
	
	Let $f^\star\in\cH$ realize the sequence, i.e., $y_t=f^\star(x_t)$ for all $t$, and let $i^\star$ satisfy $d_\ell(f^\star,f_{i^\star})\le \eps$.
	Then for all $t$,
	\[
	\ell\bigl(f_{i^\star}(x_t),y_t\bigr)
	=\ell\bigl(f_{i^\star}(x_t),f^\star(x_t)\bigr)
	\le d_\ell(f_{i^\star},f^\star)\le \eps,
	\]
	so $f_{i^\star}$ is never eliminated. Each time $\ell(\hat y_t,y_t)>\eps$ occurs, we eliminate one representative, hence this can happen
	at most $N-1$ times. Since this holds for every $\eps>0$, applying the argument for a sequence $\eps=1,1/2,1/3,\dots$ implies
	$\ell(\hat y_t,y_t)\to 0$.
\end{proof}

Thus, as a corollary for general Lipschitz functions, pointwise learning \emph{eventually} occurs for any continuous loss with $\ell(y,y)=0$.
In this following, we give a second proof for  \Cref{thm:total} using totally boundedness properties of Lipschitz functions. 


\begin{proof}[A Second Proof of Corollary~\ref{cor:env-pointwise-loss}]
	Fix $\eps>0$. Since $\ell$ is continuous on the compact set $[0,1]^2$, it is uniformly continuous; moreover, as $\ell(y,y)=0$ for all $y$,
	there exists $\delta=\delta(\eps)>0$ such that for all $a,b\in[0,1]$,
	\[
	|a-b|\le \delta \ \Longrightarrow\ \ell(a,b)\le \eps.
	\]
	Hence, for any $f,g\in\cH_L$, if $\|f-g\|_\infty\le \delta$ then
	\[
	d_\ell(f,g)=\sup_{x}\ell(f(x),g(x))\le \eps.
	\]
	On the other hand, $(\cH_L,\|\cdot\|_\infty)$ is totally bounded (indeed compact) by Arzel\`a--Ascoli, since $[-1,1]^d$ is compact and $\cH_L$
	is uniformly bounded and equicontinuous (uniformly $L$-Lipschitz). Therefore, there exists a finite $\delta$-net of $\cH_L$ in $\|\cdot\|_\infty$,
	which is an $\eps$-net in $d_\ell$. Thus $\cH_L$ is totally bounded w.r.t.\ $d_\ell$, and the theorem applies to give a deterministic proper algorithm
	with only finitely many rounds where $\ell(\hat y_t,y_t)>\eps$. Since this holds for every $\eps>0$, we conclude $\ell(\hat y_t,y_t)\to 0$.
\end{proof}

It is tempting to try characterizing {\em eventual learnability} (as in \Cref{thm:total}) using totally boundedness or compactness of the function class with respect to $d_{\ell}$. While \Cref{thm:total} shows that totally boundedness is a sufficient condition, there are function classes such as \[\cH = \{x \mapsto \mathbb{I}_{x = c} \mid c \in [0,1]\}\] that are easily learnable but not totally bounded (or compact). Hence, total boundedness is not a necessary condition for eventual learnability.

\section{Constant Loss for One ReLU}
In this section, we give an efficient online realizable algorithm for the class of functions 
\[\cH = \{x \mapsto \rel\left(w \cdot x\right) \mid w \in [-1,1]^d, ||w||_2 \leq 1\},\]
Consider the following algorithm. 
Set $w_1 = 0$. Define the learning strategy such that for each round $1 \leq t \leq T$ define
$w_{t+1} = w_t - \alpha_t \cdot x_t$. 
where $\alpha_t = \rel(w_t \cdot x_t) - \rel(w \cdot x_t)$ which can be computed at the end of round $t$ since $y_t = \rel(w \cdot x_t)$ is given to the algorithm. 
Let $L = \sum_{1 \leq t \leq T} \ell_t$ be the accumulated loss, where we use 
\[
\ell_t = \| \rel(w_t \cdot x_t) - \rel(w \cdot x_t)\|^2. 
\] 
We give the following result. 

\begin{theorem}
	$L \leq 1$. 
\end{theorem}
\begin{proof}
	Define the potential function $\phi_t = \| w_t - w\|^2$ for every $1 \leq t \leq T + 1$. Fix some $1 \leq t \leq T$; then, 
	\begin{align*}
		\label{eq:1}
		\phi_{t+1}
		&= \|w_{t+1} - w\|^2 \\
		&= \|w_t - \alpha_t \cdot x_t - w\|^2 \\
		&= \|(w_t - w) - \alpha_t \cdot x_t\|^2\\
		&= \|(w_t - w)\|^2 - 2 (w_t - w) \cdot (\alpha_t x_t) + \| \alpha_t x_t \|^2\\
		&= \phi_t - 2 (w_t - w) \cdot (\alpha_t x_t) + \| \alpha_t x_t \|^2\\
		&= \phi_t - 2 \alpha_t  (w_t - w) \cdot x_t + \alpha^2_t \| x_t \|^2\\
	\end{align*}
	
	\begin{claim}
		\label{claim:alpha}
		$-2 \alpha_t (w_t - w) \cdot x_t \leq -2 \alpha^2_t$
	\end{claim}
	\begin{proof}
		Consider two cases. 
		\begin{enumerate}
			\item $w \cdot x_t > 0$. Consider two subcases.
			\begin{itemize}
				\item $w_t \cdot x_t \geq 0$. Then, $\alpha_t = (w_t \cdot x_t - w \cdot x_t)$ implying $-2 \alpha_t (w_t - w) \cdot x_t = -2 \alpha^2_t$.
				\item $w_t \cdot x_t < 0$. Then, $\alpha_t = -w \cdot x_t$ which satisfies $\alpha_t < 0$. Therefore, 
				$$-2 \alpha_t (w_t - w) \cdot x_t = -2 \alpha_t w_t \cdot x_t + 2 \alpha_t w \cdot x_t  = -2 \alpha_t w_t \cdot x_t - 2 \alpha^2_t \leq -2 \alpha^2_t.$$
			\end{itemize}
			\item $w \cdot x_t \leq 0$. Then, consider two subcases.
			\begin{itemize}
				\item $w_t \cdot x_t \leq 0$. Thus $\alpha_t = 0$ and the inequality holds. 
				\item $w_t \cdot x_t > 0$. Then, since $w \cdot x_t \leq 0$ it holds that $\alpha_t = w_t \cdot x_t$. Thus, $\alpha_t > 0$ and
				$$-2 \alpha_t (w_t - w) \cdot x_t = -2 \alpha_t w_t \cdot x_t + 2 \alpha_t w \cdot x_t  = -2 \alpha^2_t + 2 \alpha_t w \cdot x_t \leq -2 \alpha^2_t.$$
				The inequality holds since $\alpha_t > 0$ and $w \cdot x_t \leq 0$. 
			\end{itemize}
		\end{enumerate}
		This gives the proof. 
	\end{proof}
	By Claim~\ref{claim:alpha} and the equation above it follows that: 
	$$\phi_{t+1} \leq \phi_t - 2 \alpha^2_t + \alpha^2_t \| x_t \|^2$$
	Observe that 
	\begin{align*}
		\ell_t = \| \rel(w_t \cdot x_t) - \rel(w \cdot x_t)\|^2 = \left(\rel(w_t \cdot x_t) - \rel(w \cdot x_t)\right)^2=\alpha^2_t.
	\end{align*}


	Thus, $\ell_t = \alpha^2_t$ and by combining the above 
	\begin{equation*}
		\phi_{t+1} \leq \phi_t - 2 \alpha^2_t + \alpha^2_t \| x_t \|^2 \leq \phi_t - 2 \alpha^2_t + \alpha^2_t = \phi_t - \alpha^2_t = \phi_t -{\ell_t}. 
	\end{equation*} The second inequality follows from the assumption $\|x_t\| \leq 1$. The above equation gives ${\ell_t} \leq \phi_t - \phi_{t+1}$. Therefore,
	$$\begin{aligned}
		L
		&= \sum_{1 \leq t \leq T} \ell_t \\
		&\leq \sum_{1 \leq t \leq T} \left(\phi_t - \phi_{t+1}\right) \\
		&= \left(\phi_1 - \phi_{2}\right) + \left(\phi_2 - \phi_{3}\right) + \ldots + \left(\phi_{T-1} - \phi_{T}\right) + \left(\phi_T - \phi_{T+1}\right) \\
		&= \phi_{1} - \phi_{T+1} \\
		&\leq \phi_{1}\\
		&= \|w\|^2\\
		& \leq 1.\\
	\end{aligned}$$
	The last inequality follows from the assumption $\|w\| \leq 1$. 
\end{proof}

\section{Polylogarithmic loss for bounded $k$-ReLUs}
\label{sec:polylog-l2-k-relu-general-domain}

In this section, we show that for any fixed constant $k \geq 1$ we can still avoid an $\Omega(d)$ loss, in the bounded-norm
realizable square loss setting, for the class $\cH_{k,d}$ of bounded $k$-ReLUs. Then, we generalize it to deep networks with bounded norms in each layer in Section~\ref{sec:deepL}.  

We remind the reader the definition of the class hypothesis $\cH_{k,d}$:
Let $\clip(z):=\max\{-1,\min\{1,z\}\}$.
Fix $d\ge 1$ and let
\[
\cX = \{x\in[-1,1]^d:\ \|x\|_2\le 1\}.
\]
Fix an integer $k\ge 2$ (treated as a constant).
Consider the hypothesis class $\cH=\cH_{k,d}\subseteq[-1,1]^{\cX}$ of clipped $k$-ReLU sums:
\[
\cH
\ :=\
\Big\{
h_{a,w^1,\ldots,w^k}(x)
:=\clip\left(\sum_{j=1}^k a_j\,\rel(w^j \cdot x)\right)
:\ a_j\in[-1,1],\ \|w^j\|_2\le 1\ \ \forall j\in[k]
\Big\}.
\]
Since $\|x\|_2\le 1$ and $\|w^j\|_2\le 1$, we have $\rel(w^j\cdot x)\in[0,1]$, and therefore $h(x)\in[-1,1]$.

\subsection{A Lower Bound}

Our algorithm incurs a quadratic dependency on $k$ and polylogarithmic dependency on $d$. We first show that this is unavoidable. 


\begin{proof}
It is enough to prove the claim for even $k$. If $k$ is odd, ignore one ReLU and apply the even case to $k-1$, which changes the lower bound only by a universal constant. Assume therefore that $k$ is even.

First assume that $d$ is a power of two. For each sign vector $s=(s_1,\ldots,s_d)\in\{\pm1\}^d$, let $v_s=s/\sqrt d$, and define
\[
h_s(x)=\clip\left(\frac{k}{2}|\langle v_s,x\rangle|\right).
\]
Then $h_s\in\cH_{k,d}$, since
\[
\frac{k}{2}|\langle v_s,x\rangle|
=
\sum_{\ell=1}^{k/2}
\left(
\rel(\langle v_s,x\rangle)+\rel(\langle -v_s,x\rangle)
\right).
\]
This uses exactly $k$ ReLUs, all input weights have norm $1$, and all output weights are $1$.

The adversary maintains a partition of $[d]=\{1,\ldots,d\}$ into dyadic intervals. Initially the partition consists of the singleton intervals $\{1\},\ldots,\{d\}$. At later stages, the adversary only merges two adjacent dyadic intervals of the same size. For a block $B\subseteq[d]$ and a sign vector $s\in\{\pm1\}^d$, write $s_B=(s_i)_{i\in B}\in\{\pm1\}^B$ for the restriction of $s$ to the coordinates in $B$. Each current block $B$ has a public reference pattern $r_B\in\{\pm1\}^B$. The invariant is that the surviving signs are exactly those $s\in\{\pm1\}^d$ such that, for every current block $B$, one has $s_B=r_B$ or $s_B=-r_B$. Initially, set $r_{\{i\}}(i)=1$ for every singleton block $\{i\}$, so all signs survive and the invariant holds. The reference patterns are determined by the past transcript and may be revealed to the learner; the only hidden information is whether $s_B=r_B$ or $s_B=-r_B$ on each current block $B$.

Consider a merge of two adjacent current blocks $B_0,B_1$, each of size $m/2$, and write $B=B_0\cup B_1$, so $|B|=m$. The adversary queries
\[
x_B
=
\frac1{\sqrt m}
\left(
\sum_{i\in B_0}r_{B_0}(i)e_i
+
\sum_{i\in B_1}r_{B_1}(i)e_i
\right),
\]
where $e_1,\ldots,e_d$ are the standard basis vectors. This is a valid input because
\[
\|x_B\|_2^2
=
\frac1m\left(|B_0|+|B_1|\right)
=
1.
\]

By the invariant, before the label is chosen, both orientations of each child block are feasible independently. If the two child blocks have the same orientation, meaning either $s_{B_0}=r_{B_0}$ and $s_{B_1}=r_{B_1}$, or $s_{B_0}=-r_{B_0}$ and $s_{B_1}=-r_{B_1}$, then
\[
|\langle v_s,x_B\rangle|
=
\frac{|B_0|+|B_1|}{\sqrt{dm}}
=
\sqrt{\frac md}.
\]
If the two child blocks have opposite orientations, meaning either $s_{B_0}=r_{B_0}$ and $s_{B_1}=-r_{B_1}$, or $s_{B_0}=-r_{B_0}$ and $s_{B_1}=r_{B_1}$, then
\[
|\langle v_s,x_B\rangle|
=
\frac{\bigl||B_0|-|B_1|\bigr|}{\sqrt{dm}}
=
0.
\]
Thus the two labels
\[
0
\qquad\text{and}\qquad
\alpha_m=\min\left\{\frac{k}{2}\sqrt{\frac md},1\right\}
\]
are both realizable by surviving hypotheses.

For every prediction $\widehat y$, one of the two labels $0,\alpha_m$ causes square loss at least $\alpha_m^2/4$. Otherwise both $|\widehat y|$ and $|\widehat y-\alpha_m|$ would be smaller than $\alpha_m/2$, contradicting $|\alpha_m|\le |\widehat y|+|\widehat y-\alpha_m|$. Hence, on every level with $m\le 4d/k^2$, no clipping occurs, and each merge forces loss at least $k^2m/(16d)$.

After the label is chosen, the adversary updates the merged block $B$. If it chooses the high label $\alpha_m$, it keeps exactly the surviving signs with the same orientation and defines the new reference pattern by
\[
r_B(i)=
\begin{cases}
r_{B_0}(i), & i\in B_0,\\
r_{B_1}(i), & i\in B_1.
\end{cases}
\]
If it chooses the zero label, it keeps exactly the surviving signs with opposite orientations and defines
\[
r_B(i)=
\begin{cases}
r_{B_0}(i), & i\in B_0,\\
-r_{B_1}(i), & i\in B_1.
\end{cases}
\]
In both cases, the surviving restrictions on the merged block $B$ are exactly $r_B$ and $-r_B$. All other blocks and reference patterns are unchanged. Therefore the invariant is preserved, and the surviving set remains nonempty.

The adversary performs all merges bottom-up: first it merges singleton blocks into blocks of size $2$, then blocks of size $2$ into blocks of size $4$, and so on. At a level whose merged block size is $m$, there are $d/m$ merges. Therefore every level with $m\le 4d/k^2$ contributes loss at least
\[
\frac dm\cdot\frac{k^2m}{16d}
=
\frac{k^2}{16}.
\]
The dyadic levels are $m=2,4,\ldots,d$. The condition $m\le 4d/k^2$ holds for $\Omega(\log(d/k^2))$ of these levels whenever $d/k^2\to\infty$. Thus the total forced loss is $\Omega(k^2\log(d/k^2))$. For fixed $k$, this is $\Omega(k^2\log d)$.

The construction uses exactly $d-1$ queries and is realizable. Indeed, along every path, the invariant keeps at least one sign vector alive. Choosing any final surviving $s$, the fixed hypothesis $h_s\in\cH_{k,d}$ agrees with all revealed labels on that path.

Now consider general $d$. Let $D=2^{\lfloor \log_2 d\rfloor}$. Run the same construction on the first $D$ coordinates, and set all other coordinates to zero in both the queries and the vectors $v_s$. Equivalently, use sign vectors $s\in\{\pm1\}^D$ and vectors
\[
v_s=\frac1{\sqrt D}(s_1,\ldots,s_D,0,\ldots,0)\in\mathbb R^d.
\]
The proof above gives a lower bound of $\Omega(k^2\log(D/k^2))$ using $D-1$ queries. Since $D\ge d/2$, this is $\Omega(k^2\log(d/k^2))$ whenever $d/k^2\to\infty$. Also $D-1\le d-1\le T$, so the construction fits the stated horizon.
\end{proof}

\begin{corollary}
The upper bound $O(k^2\log^4 d)$ for clipped bounded $k$-ReLU sums is tight in its dependence on $k$, up to polylogarithmic factors in $d$. Moreover, the dependence on $d$ cannot be removed: for every fixed $k\ge2$, the realizable online square-loss minimax value is at least $\Omega(\log d)$. Thus a logarithmic dependence on the input dimension is necessary.
\end{corollary}


\subsection{The Algorithm for $\mathcal{H}_{k,d}$}

In the following, we give the $\tilde{O}(k^2)$ upper bound given in \Cref{thm:intro_k_bounded}. 

\paragraph{Proof Roadmap.}
We prove the upper bound by splitting any shattered tree into \emph{large-gap} and \emph{small-gap} nodes: large gaps are controlled by a truncated multiscale argument via sequential fat-shattering bounds \cite{rakhlin2015online}, while small gaps are handled by scaling into a comparison class and bounding its online dimension using our entropy-potential method and Euclidean covers. 
We conclude by combining the two contributions


For the construction only, to cleanly handle rescalings without leaving a bounded-loss regime, we work with the
\emph{clipped squared loss}
\[
\ell_{\cap}(\hat y,y)\ :=\ \min\Big\{1,\ \frac{(\hat y-y)^2}{4}\Big\}\in[0,1],
\qquad \hat y,y\in\R.
\]
For $\hat y,y\in[-1,1]$ (as in the output of functions from $\cH$) we have $\ell_{\cap}(\hat y,y)=(\hat y-y)^2/4$.
In the following, we will use the online dimension $\mathbb D_{\mathrm{onl}}(\cG;\ell_{\cap})$ with respect to \ $\ell_{\cap}$ (see Definition~\ref{def:online-dim}). In order to use the online dimension and the result of~\cite{attias2023optimal}, we show that $\ell_{\cap}$ is an approximate pseudo metric. 


\begin{lemma}[$\ell_{\cap}$ is a $2$-approximate pseudo-metric]
	\label{lem:lcap-approx-metric}
	For all $a,b,c\in\R$,
	\[
	\ell_{\cap}(a,b)\ \le\ 2\big(\ell_{\cap}(a,c)+\ell_{\cap}(c,b)\big).
	\]
\end{lemma}
\begin{proof}
	Let $u:=(a-b)^2/4$, $v:=(a-c)^2/4$, $w:=(c-b)^2/4$. Then $u\le 2(v+w)$ (the standard $2$-approximate triangle for squared loss).
	We show $\min\{1,u\}\le 2(\min\{1,v\}+\min\{1,w\})$.
	If $\max\{v,w\}\ge 1$, then the RHS is at least $2$ and the claim holds since the LHS is at most $1$.
	Otherwise $v<1$ and $w<1$, hence $\min\{1,v\}=v$ and $\min\{1,w\}=w$, so
	$\min\{1,u\}\le u\le 2(v+w)=2(\min\{1,v\}+\min\{1,w\})$.
\end{proof}

To get the bound for $\cH$, we partition the analysis into two. First, we show that relatively large gaps are bounded using
sequential fat-shattering dimension bounds \cite{rakhlin2015online}. Then, we upper bound infinitely small gaps using a
$1/d$ coupling to an auxiliary comparison class via our entropy-potential bound given in \Cref{thm:intro_Donl-via-Phi}. 

\subsection{Large gaps: a finite-scale multiscale branch bound via sequential fat}

For $\alpha>0$ and $\cF\subseteq[-1,1]^{\cX}$, we use $\fat_\alpha(\cF)$ to denote the sequential fat-shattering dimension
\cite{rakhlin2015online}, defined as follows for completeness.

\begin{definition}[Sequential fat-shattering dimension]\label{def:seq-fat}
	Let $\cF\subseteq \R^{\cX}$ and let $\alpha>0$.
	A pair of complete binary trees
	\[
	(x_u)_{u\in\{0,1\}^{<m}} \subseteq \cX
	\qquad\text{and}\qquad
	(s_u)_{u\in\{0,1\}^{<m}} \subseteq \R
	\]
	is \emph{$\alpha$-shattered} by $\cF$ if for every path $b=(b_1,\dots,b_m)\in\{0,1\}^m$
	there exists $f_b\in\cF$ such that for all $t=0,1,\dots,m-1$,
	\[
	b_{t+1}=1 \ \Rightarrow\ f_b\!\left(x_{b_{\le t}}\right)\ \ge\ s_{b_{\le t}}+\alpha,
	\qquad
	b_{t+1}=0 \ \Rightarrow\ f_b\!\left(x_{b_{\le t}}\right)\ \le\ s_{b_{\le t}}-\alpha,
	\]
	where $b_{\le t}:=(b_1,\dots,b_t)$ and $b_{\le 0}:=\emptyset$.
	The \emph{sequential fat-shattering dimension} of $\cF$ at scale $\alpha$ is
	\[
	\fat_{\alpha}(\cF)\ :=\ \sup\Big\{m\in\N:\ \exists\text{ a depth-$m$ tree $\alpha$-shattered by }\cF\Big\},
	\]
	with the convention $\sup\emptyset=0$ (and $\fat_\alpha(\cF)=\infty$ if arbitrarily large depths are possible).
\end{definition}
The following lemma shows the decrease in the above dimension along a path of a shattered tree. 

\begin{lemma}
	\label{lem:fatdrop-general-domain}
	Let $\cV\subseteq[-1,1]^{\cX}$ be nonempty, fix $x\in\cX$, and two values $a\neq b$.
	Let $\Delta:=|a-b|$ and~define
	\[
	\cV_0:=\{h\in\cV:h(x)=a\},\qquad \cV_1:=\{h\in\cV:h(x)=b\},
	\]
	assumed both nonempty. If $\Delta\ge 2\alpha$, then
	\[
	\min\bigl\{\fat_\alpha(\cV_0),\ \fat_\alpha(\cV_1)\bigr\}
	\ \le\
	\fat_\alpha(\cV)-1.
	\]
\end{lemma} 		We call a child $u^+$ of a node $u$ \emph{$\alpha$-decreasing} if
$\fat_{\alpha}(\cV(u^+))\le \fat_{\alpha}(\cV(u))-1$.
\begin{proof}
	Let $m:=\fat_\alpha(\cV)$ and assume for contradiction that both children have $\fat_\alpha$ dimension at least $m$.
	Without loss of generality assume $a>b$ (otherwise swap the roles of $\cV_0$ and $\cV_1$).
	Glue depth-$m$ $\alpha$-shattering trees for $\cV_0$ and $\cV_1$ under a new root at instance $x$
	with threshold $\theta:=(a+b)/2$, attaching the $\cV_0$ subtree to the branch requiring
	$f(x)\ge \theta+\alpha$ and the $\cV_1$ subtree to the branch requiring $f(x)\le \theta-\alpha$.
	Since $\Delta=a-b\ge 2\alpha$, we have $a\ge \theta+\alpha$ and $b\le \theta-\alpha$, yielding an $\alpha$-shattering tree
	of depth $m+1$ for $\cV$, a contradiction.
\end{proof}

Let $\alpha_r:=2^{-(r+2)}$ for $r\ge -1$, and set
\[
K:=\big\lceil \log_2(\sqrt d)\big\rceil.
\]

The next lemma shows that large gaps can be upper bound by a sum of fat-shattering components on scales $-1,\ldots, K$. 
\begin{lemma}[Large-gap branch bound via a truncated multiscale potential]
	\label{lem:large-gap-truncated-general-domain}
	Let $T$ be any scaled Littlestone tree shattered by $\cH\subseteq[-1,1]^{\cX}$, and let
	$\Delta_u:=|s_{u,0}-s_{u,1}|$ at node $u$.
	Then there exists a root-to-leaf branch $P$ in $T$ such that
	\[
	\sum_{u\in P:\ \Delta_u>1/\sqrt d}\Delta_u^2
	\ \le\
	16\sum_{r=-1}^{K}\alpha_r^2\,\fat_{\alpha_r}(\cH).
	\]
\end{lemma}

\begin{proof}
	For a node $u$, let $\cV(u)\subseteq\cH$ be its version space (hypotheses consistent with the root-to-$u$ labels),
	and define the truncated potential
	\[
	\Psi(u)\ :=\ 16\sum_{r=-1}^{K}\alpha_r^2\,\fat_{\alpha_r}(\cV(u)).
	\]
	Construct a branch top-down. At an internal node $u$:
	\begin{itemize}
		\item If $\Delta_u\le 1/\sqrt d$, move to an arbitrary child.
		\item If $\Delta_u>1/\sqrt d$, choose the unique $r\in\{-1,0,\dots,K\}$ such that
		\[
		2\alpha_r < \Delta_u \le 4\alpha_r.
		\]
		We justify existence as follows. Since $4\alpha_{-1}=2$ and always $\Delta_u\le 2$ (labels lie in $[-1,1]$), and also
		$\Delta_u>1/\sqrt d\ge 2^{-K}$ with $4\alpha_K=2^{-K}$, such $r\le K$ exists.
		Apply Lemma~\ref{lem:fatdrop-general-domain} at scale $\alpha_r$ to $\cV(u)$; since $\Delta_u\ge 2\alpha_r$,
		there exists a child $u^+$ with
		$\fat_{\alpha_r}(\cV(u^+))\le \fat_{\alpha_r}(\cV(u))-1$.
		Move to such a child $u^+$.
	\end{itemize}
	Along the resulting branch $P$, whenever we visit a node $u$ with $\Delta_u>1/\sqrt d$ and bucket index $r$,
	the potential drops by at least $16\alpha_r^2$, and also $\Delta_u^2\le (4\alpha_r)^2=16\alpha_r^2$.
	Thus,
	\[
	\sum_{u\in P:\Delta_u>1/\sqrt d}\Delta_u^2
	\ \le\
	\sum_{u\in P:\Delta_u>1/\sqrt d}\big(\Psi(u)-\Psi(u^+)\big)
	\ \le\
	\Psi(\text{root})
	\ =\
	16\sum_{r=-1}^{K}\alpha_r^2\,\fat_{\alpha_r}(\cH).
	\]
	Note that the argument only used the forced choice at nodes with $\Delta_u>1/\sqrt d$; at nodes with
	$\Delta_u\le 1/\sqrt d$ the child was arbitrary. Therefore, the same bound holds for \emph{any} branch that,
	whenever it reaches a node with $\Delta_u>1/\sqrt d$, follows the $\alpha_r$-decreasing child chosen in the construction.
\end{proof}

We now show using the work of \cite{rakhlin2015online} an upper bound on the fat-shattering dimension of the class. 
\begin{proposition}[Sequential fat bound for $\cH$]
	\label{prop:fatbound-l2-general-domain}
	There exists a universal constant $C>0$ such that for all $\alpha\in(0,1]$,
	\[
	\fat_{\alpha}(\cH)\ \le\ \frac{Ck^2}{\alpha^2}\,\log^3\!\Big(\frac{ek}{\alpha}\Big).
	\]
\end{proposition}
\begin{proof}
	We invoke the standard sequential Rademacher--fat chain from~\cite{rakhlin2015online}.
	
	\paragraph{Step 1 (linear class).}
	Let $\cG:=\{x\mapsto \langle w,x\rangle:\ \|w\|_2\le 1\}\subseteq[-1,1]^{\cX}$.
	By~\cite[Prop.~14]{rakhlin2015online}, for all $T\ge 1$,
	\[
	\Rad_T(\cG)\ \le\ \frac{C_0}{\sqrt{T}}
	\]
	for a universal constant $C_0$.
	
	\paragraph{Step 2 (single ReLU with a scalar).}
	Let
	\[
	\cR:=\{x\mapsto a\,\rel(\langle w,x\rangle):\ |a|\le 1,\ \|w\|_2\le 1\}.
	\]
	Since $\rel$ is $1$-Lipschitz, by~\cite[Cor.~5]{rakhlin2015online} there is a factor
	$\rho(T)=O(\log^{3/2}(eT))$ such that
	\[
	\Rad_T(\cR)\ \le\ 8\,\rho(T)\,\Rad_T(\cG)\ \le\ \frac{C_1\,\rho(T)}{\sqrt{T}}
	\]
	for a universal constant $C_1$.
	
	\paragraph{Step 3 (unclipped $k$-term sums and clipping).}
	Let
	\[
	\cH_0:=\Bigl\{x\mapsto \sum_{j=1}^k r_j(x):\ r_j\in\cR\Bigr\}.
	\]
	For any $\cX$-valued tree $z$, separability of the parameters across $j$ gives
	\[
	\Rad_T(\cH_0,z)
	=\frac1T\,\E_\eps\Bigl[\sup_{r_1,\dots,r_k\in\cR}\sum_{j=1}^k\sum_{t=1}^T \eps_t r_j(z_t(\eps))\Bigr]
	=\sum_{j=1}^k \Rad_T(\cR,z)
	= k\,\Rad_T(\cR,z).
	\]
	Hence $\Rad_T(\cH_0)\le k\,\Rad_T(\cR)\le C_1 k\rho(T)/\sqrt{T}$.
	Now $\cH$ is obtained by clipping $\cH_0$ pointwise to $[-1,1]$ (i.e., applying $\clip$),
	and $\clip$ is $1$-Lipschitz on $\R$, so by contraction \cite[Cor.~5]{rakhlin2015online},
	\[
	\Rad_T(\cH)\ \le\ \Rad_T(\cH_0)\ \le\ \frac{C_1 k\rho(T)}{\sqrt{T}}.
	\]
	
	\paragraph{Step 4 (fat from sequential Rademacher).}
	By~\cite[Prop.~9]{rakhlin2015online}, for all $\alpha\in(0,1]$,
	\[
	\alpha\sqrt{\frac{\min\{\fat_\alpha(\cH),T\}}{T}}
	\ \le\ C_2\,\Rad_T(\cH)
	\ \le\ \frac{C_3\,k\rho(T)}{\sqrt{T}}
	\]
	for universal constants $C_2,C_3$.
	Choosing $T=\Theta\!\big(\alpha^{-2}k^2\log^3(ek/\alpha)\big)$ large enough ensures
	\[C_3 k\rho(T)/\sqrt{T}<\alpha/2,\] which forces $\fat_\alpha(\cH)<T$ and yields the claim.
\end{proof}

We conclude on the sum of fat-shattering across different scales as follows. 

\begin{lemma}[Large-gap contribution is $O(k^2\log^4(edk))$]
	\label{lem:large-gap-Olog4-general-domain}
	\[
	16\sum_{r=-1}^{K}\alpha_r^2\,\fat_{\alpha_r}(\cH)\ =\ O\!\big(k^2\log^4(edk)\big),
	\]
	where $\alpha_r:=2^{-(r+2)}$ and $K=\lceil \log_2(\sqrt d)\rceil$.
\end{lemma}
\begin{proof}
	By Proposition~\ref{prop:fatbound-l2-general-domain},
	$\alpha_r^2\fat_{\alpha_r}(\cH)\le Ck^2\log^3(ek/\alpha_r)$.
	Since $\alpha_r=2^{-(r+2)}$, $\log(ek/\alpha_r)=\Theta(\log(ek)+r+1)$ and
	$\sum_{r\le K}(\log(ek)+r+1)^3=O((\log(ek)+K)^4)=O(\log^4(edk))$.
\end{proof}

Next, we complement the analysis by considering gaps smaller than $1/d$. 

\subsection{Small gaps: scaling into a comparison class and paying $1/d$}

Define the scaled comparison class
\[
\cF
\ :=\
\{\ g:\cX\to\R\ :\ \exists h\in\cH\ \text{s.t.}\ g(x)=\sqrt d\,h(x)\ \ \forall x\in\cX\ \}.
\]
Let $T$ be any scaled Littlestone tree shattered by $\cH$ (under $\ell_{\cap}$).
Consider the (deterministic) rule from the proof of Lemma~\ref{lem:large-gap-truncated-general-domain}:
whenever a reached node $u$ satisfies $\Delta_u>1/\sqrt d$, move to the $\alpha_r$-decreasing child selected there;
whenever $\Delta_u\le 1/\sqrt d$, the next child is chosen freely.
Then, we have the following statement for such $T$. 
\begin{lemma}[Small-gap amplification under $\ell_{\cap}$]
	\label{lem:small-gap-amplification-general-domain}
	There exists a branch $P$ in $T$ that follows the above rule at every node with $\Delta_u>1/\sqrt d$ and satisfies
	\[
	E_{\mathrm{small}}(P)\ :=\ \sum_{u\in P:\ \Delta_u\le 1/\sqrt d}\Delta_u^2
	\ \le\ \frac{4}{d}\,\mathbb D_{\mathrm{onl}}(\cF;\ell_{\cap}).
	\]
\end{lemma}

\begin{proof}
	Run the above rule, but leave the choices at nodes with $\Delta_u\le 1/\sqrt d$ unspecified.
	This induces a full binary tree $S$ whose internal nodes are exactly the reached nodes with
	$\Delta_u\le 1/\sqrt d$ (choices only occur there), with the same queried instances and the same two edge-labels
	$s_{u,0},s_{u,1}$ as in $T$ at those nodes. If needed, we pad shorter branches with dummy nodes of zero gap (same label on both edges), which preserves shattering and does not change branch energies, yielding a full binary tree of uniform depth.
	Because $T$ is shattered by $\cH$, every branch of $S$ corresponds to a branch of $T$ (obtained by inserting the
	forced moves at $\Delta_u>1/\sqrt d$), hence is realizable by some hypothesis in $\cH$.
	Therefore $S$ is shattered by $\cH$.
	Now scale all edge-labels in $S$ by $\sqrt d$ to obtain a tree $S'$ with labels
	$s'_{u,b}:=\sqrt d\,s_{u,b}$.
	For every branch of $S$, if $h_b\in\cH$ realizes it, then the scaled function $g_b:=\sqrt d\,h_b$ lies in $\cF$
	and realizes the corresponding branch of $S'$. Thus $S'$ is shattered by $\cF$.
	
	For every node $u$ of $S$ we have $\Delta_u\le 1/\sqrt d$, hence in $S'$ the gap is
	$\Delta'_u=\sqrt d\,\Delta_u\le 1$ and clipping in $\ell_\cap$ does not occur:
	\[
	\ell_{\cap}(s'_{u,0},s'_{u,1})=(\Delta'_u)^2/4=d\,\Delta_u^2/4.
	\]
	Therefore, for every branch $b$ of $S'$,
	\[
	E_{S'}(b)=\frac{d}{4}\sum_{u\in b}\Delta_u^2=\frac{d}{4}\cdot E_{\mathrm{small}}(P_b),
	\]
	where $P_b$ is the corresponding branch of $T$ produced by the rule.
	Hence
	\[
	\inf_b E_{S'}(b)=\frac{d}{4}\cdot \inf_b E_{\mathrm{small}}(P_b).
	\]
	Since $S'$ is a valid shattered tree for $\cF$, by definition of $\mathbb D_{\mathrm{onl}}(\cF;\ell_{\cap})$,
	\[
	\mathbb D_{\mathrm{onl}}(\cF;\ell_{\cap})\ \ge\ \inf_b E_{S'}(b)\ =\ \frac{d}{4}\cdot \inf_b E_{\mathrm{small}}(P_b).
	\]
	Choose a branch $b^*$ attaining the infimum. The corresponding branch $P:=P_{b^*}$ of $T$ follows the large-gap
	rule by construction and satisfies
	$E_{\mathrm{small}}(P)\le \frac{4}{d}\mathbb D_{\mathrm{onl}}(\cF;\ell_{\cap})$.
\end{proof}

Next, we bound the total loss of the comparison class using our generic potential argument. 
\subsection{Bounding $\mathbb D_{\mathrm{onl}}(\cF;\ell_{\cap})$ by $\widetilde O(kd)$ via Euclidean covers}

Let $\Theta:=([-1,1]\times B_2^d(0,1))^k\subseteq\R^{k(d+1)}$ be the parameter set for $\cH$:
for \[\theta=(a_1,\dots,a_k,w^1,\dots,w^k)\in\Theta\]
define
\[
h_\theta(x):=\clip\Big(\sum_{j=1}^k a_j\,\rel(w^j \cdot x)\Big)\in[-1,1],
\qquad
g_\theta(x):=\sqrt d\,h_\theta(x)\in\R.
\]
Then $\cH=\{h_\theta:\theta\in\Theta\}$ and $\cF=\{g_\theta:\theta\in\Theta\}$.

\begin{lemma}[Euclidean parameter Lipschitzness for $\cF$]
	\label{lem:F-lipschitz-general-domain}
	For all $\theta,\theta'\in\Theta$,
	\[
	\sup_{x\in\cX}|g_\theta(x)-g_{\theta'}(x)|
	\ \le\
	\sqrt{2kd}\,\|\theta-\theta'\|_2,
	\qquad
	\text{hence}\qquad
	d_{\ell_{\cap}}(g_\theta,g_{\theta'})
	\ \le\
	\min\left\{1,\ \frac{kd}{2}\,\|\theta-\theta'\|_2^2\right\}.
	\]
\end{lemma}
\begin{proof}
	Since $\clip$ is $1$-Lipschitz and $|\rel(u)-\rel(v)|\le |u-v|$, for any $x\in\cX$ we have
	\begin{align*}
		|h_\theta(x)-h_{\theta'}(x)|
		&\le
		\sum_{j=1}^k \big|a_j\,\rel( w^j \cdot x)-a_j'\,\rel(w^{j\prime} \cdot x)\big|\\
		&\le
		\sum_{j=1}^k \Big(|a_j-a_j'|\cdot|\rel(w^j \cdot x)|\ +\ |a_j'|\cdot|\rel( w^j \cdot x)-\rel(w^{j\prime} \cdot x)|\Big)\\
		&\le
		\sum_{j=1}^k \Big(|a_j-a_j'|\ +\ |\langle w^j-w^{j\prime},x\rangle|\Big)
		\ \le\
		\sum_{j=1}^k \Big(|a_j-a_j'|\ +\ \|w^j-w^{j\prime}\|_2\Big),
	\end{align*}
	using $\rel(w^j \cdot x)\in[0,1]$ and $|a_j'|\le 1$.
	By Cauchy--Schwarz,
	\[
	\sum_{j=1}^k |a_j-a_j'|\ \le\ \sqrt{k}\,\|a-a'\|_2,
	\qquad
	\sum_{j=1}^k \|w^j-w^{j\prime}\|_2\ \le\ \sqrt{k}\,\|w-w'\|_2,
	\]
	hence $\sup_x |h_\theta(x)-h_{\theta'}(x)|\le \sqrt{2k}\,\|\theta-\theta'\|_2$.
	Multiplying by $\sqrt d$ gives the first inequality.
	For the second, use $\ell_\cap(u,v)\le (u-v)^2/4$ and take $\sup_{x\in\cX}$.
\end{proof}

\begin{lemma}[Covering numbers of $\Theta$]
	\label{lem:theta-cover-general-domain}
	For all $r\in(0,1]$,
	\[
	N(\Theta,r;\|\cdot\|_2)\ \le\ \left(\frac{3\sqrt{2k}}{r}\right)^{k(d+1)}.
	\]
\end{lemma}
\begin{proof}
	Note that $\Theta\subseteq B_2^{k(d+1)}(0,\sqrt{2k})$. A standard volumetric bound for Euclidean balls \cite{Vershynin2018HDP} gives
	$N(B_2^{k(d+1)}(0,\sqrt{2k}),r)\le (3\sqrt{2k}/r)^{k(d+1)}$ for $r\in(0,1]$. 
\end{proof}

The following summarizes the total loss bound for the comparison class. 

\begin{proposition}[Online dimension of the comparison class]
	\label{prop:Donl-F-tildeO-d-general-domain}
	\[
	\mathbb D_{\mathrm{onl}}(\cF;\ell_{\cap})\ =\ O\big(kd\log(d)\big).
	\]
\end{proposition}
\begin{proof}
	By Lemma~\ref{lem:F-lipschitz-general-domain},
	$d_{\ell_{\cap}}(g_\theta,g_{\theta'})\le \min\{1,(L\|\theta-\theta'\|_2)^2\}$ with $L:=\sqrt{kd/2}$.
	Thus, for $\eps\in(0,1]$, an $\eps$-cover in $d_{\ell_\cap}$ is induced by an
	$r:=\min\{1,\sqrt{\eps}/L\}$ cover in $\|\cdot\|_2$, so
	\[
	N(\cF,\eps)\ \le\ N(\Theta,r;\|\cdot\|_2).
	\]
	By Lemma~\ref{lem:theta-cover-general-domain},
	\[
	N(\cF,\eps)\ \le\ \Big(\frac{3\sqrt{2k}}{r}\Big)^{k(d+1)}
	=\Big(3\sqrt{2k}\,\max\{1,L/\sqrt\eps\}\Big)^{k(d+1)}
	\le \Big(\frac{3\sqrt{2k}\,(1+L)}{\sqrt\eps}\Big)^{k(d+1)},
	\qquad(\eps\in(0,1]).
	\]
	Recall the entropy potential $\Phi(\cF):=\int_0^1 \log_2 N(\cF,\eps)\,d\eps$. Then
	\[
	\log_2 N(\cF,\eps)\ \le\ k(d+1)\log_2\!\big(3\sqrt{2k}\,(1+L)\big)\ +\ \frac{k(d+1)}{2}\log_2(1/\eps),
	\]
	so using $\int_0^1 \log_2(1/\eps)\,d\eps=1/\ln 2$,
	\[
	\Phi(\cF)\ \le\ k(d+1)\log_2\!\big(3\sqrt{2k}\,(1+L)\big)\ +\ \frac{k(d+1)}{2\ln 2}
	\ =\ O\big(kd\log d\big).
	\]
	Finally, since $\ell_{\cap}$ is a $2$-approximate pseudo-metric (Lemma~\ref{lem:lcap-approx-metric}),
	the standard entropy-potential branch argument yields
	$\mathbb D_{\mathrm{onl}}(\cF;\ell_{\cap})\le 8\,\Phi(\cF)=O(kd\log d)$.
\end{proof}

We conclude the two parts (large and small gaps) in the following section. 

\subsection{Main theorem}

\subsubsection*{Proof of \Cref{thm:intro_k_bounded}:}

\begin{proof}
	Fix any scaled Littlestone tree $T$ shattered by $\cH$.
	Apply Lemma~\ref{lem:small-gap-amplification-general-domain} to obtain a branch $P$ that follows the large-gap
	fat-shattering-drop rule at every node with $\Delta_u>1/\sqrt d$ and satisfies
	$ \sum_{u\in P:\Delta_u\le 1/\sqrt d}\Delta_u^2 \le \frac{4}{d} \mathbb D_{\mathrm{onl}}(\cF;\ell_{\cap})$.
	For this same branch $P$, the large-gap bound from Lemma~\ref{lem:large-gap-truncated-general-domain} applies
	(because $P$ follows the required rule at every large-gap node), yielding
	\[
	\sum_{u\in P:\Delta_u>1/\sqrt d}\Delta_u^2 \ \le\ 16\sum_{r=-1}^{K}\alpha_r^2\,\fat_{\alpha_r}(\cH)
	\ =\ O(k^2\log^4 d),
	\]
	where the last step uses Lemma~\ref{lem:large-gap-Olog4-general-domain}.
	For the remaining nodes on $P$ with $\Delta_u\le 1/\sqrt d$, Proposition~\ref{prop:Donl-F-tildeO-d-general-domain} gives
	\[
	\sum_{u\in P:\ \Delta_u\le 1/\sqrt d}\Delta_u^2
	\ \le\
	\frac{4}{d}\,\mathbb D_{\mathrm{onl}}(\cF;\ell_{\cap})
	\ =\ O(k\log d).
	\]
	Summing the two contributions yields $\sum_{u\in P}\Delta_u^2=O(k^2\log^4 d)$.
	Since on shattered trees for $\cH\subseteq[-1,1]^{\cX}$ we have $\ell_{\cap}(s_{u,0},s_{u,1})=\Delta_u^2/4$,
	it follows that $E_T(P)=\sum_{u\in P}\ell_{\cap}(s_{u,0},s_{u,1})=O(k^2\log^4 d)$, hence $\inf_b E_T(b)\le O(k^2\log^4 d)$.
	Taking $\sup_T$ gives $\mathbb D_{\mathrm{onl}}(\cH;\ell_{\cap})\le O(k^2\log^4 d)$.
	
	Finally, for realizable squared loss with labels in $[-1,1]$, one may clip the learner's predictions to $[-1,1]$
	without increasing squared loss, and on $[-1,1]$ we have squared loss equals to four times the loss $\ell_{\cap}$.
	Thus, the minimax realizable deterministic cumulative squared loss is at most
	$4\,\mathbb D_{\mathrm{onl}}(\cH;\ell_{\cap})=O(k^2\log^4 d)$. This gives the proof of the theorem. 
\end{proof}

\section{Extension to depth-$L$, width-$k$ networks.}
\label{sec:deepL}
We generalize \Cref{thm:intro_k_bounded} to a tight bound for depth-$L$ networks.
The natural deep extension of the bounded-norm shallow class $\cH_{k,d}$ imposes
$\|w\|_2\le 1$ on each neuron's weight vector at every layer (matching the shallow constraint $\|w^j\|_2\le 1$). 

Concretely, fix integers $L\ge 2$ and $k\ge 1$.
Define $\cH_{L,k,d}^{\mathrm{deep}}$ to be the class of all functions $h:\cX\to[-1,1]$ of the form
\[
h(x)=\clip\!\big(\langle a,z^{(L-1)}(x)\rangle+c\big),
\]
where $z^{(0)}=x$,\;
$z^{(\ell)}=\relu(W_\ell z^{(\ell-1)}+b_\ell)$ for $\ell=1,\dots,L-1$,
with:
\begin{itemize}
	\item $W_1\in\R^{k\times d}$ with each row $\|w_{1,j}\|_2\le 1$;\;
	$W_\ell\in\R^{k\times k}$ with each row $\|w_{\ell,j}\|_2\le 1$ for $\ell\ge 2$,
	\item $b_\ell\in[-1,1]^k$ for $\ell=1,\dots,L-1$,\quad $a\in[-1,1]^k$,\quad $c\in[-1,1]$.
\end{itemize}
Set $p:=kd+(L-2)k^2+Lk+1$ (total scalar parameters) and
$\cF_{L,k,d}^{\mathrm{deep}}:=\{\sqrt d\,h:\ h\in\cH_{L,k,d}^{\mathrm{deep}}\}$.

\begin{lemma}[Depth-$L$ analogue of \Cref{thm:intro_k_bounded}]
	\label{lem:deep-polylog-general-domain}
	For every fixed $L \geq 2$ it holds that 
	\[
	\mathbb D_{\mathrm{onl}}\!\big(\cH_{L,k,d}^{\mathrm{deep}};\ell_{\cap}\big)
	\ =\
	O_L\!\Big(k^{L}\,\log^{\,3L+1}\!\big(edkL\big)\Big).
	\]
	Consequently, the minimax realizable cumulative squared loss is
	$\widetilde O(k^{L})$ for fixed $L$.
	For $L=2$ (one hidden layer) this recovers the $\widetilde O(k^2)$ bound of \Cref{thm:intro_k_bounded}.
\end{lemma}

In the following, we use Big-O notations omitting constants that depend on $L$. In addition, for simplicity we use $C_L, C'_L$, and $C''_L$ to denote specific constants that depend only on $L$. 

We give the proof of the above lemma below. Whenever the arguments are analogous to the setting of $L=2$, we may omit some of the details for conciseness.  
\begin{proof}
	We repeat the large-gap/small-gap decomposition with threshold $1/\sqrt d$.
	
	\paragraph{Large gaps.}
	The proof of Lemma~\ref{lem:large-gap-truncated-general-domain} is purely combinatorial, so it immediately applies:
	for any shattered tree, along a branch that follows the decreasing-child rule at large-gap nodes,
	\[
	\sum_{\Delta_u>1/\sqrt d}\Delta_u^2
	\ \le\
	16\sum_{r=-1}^{K}\alpha_r^2\,\fat_{\alpha_r}\!\big(\cH_{L,k,d}^{\mathrm{deep}}\big),
	\qquad
	K=\lceil \log_2(\sqrt d)\rceil.
	\]
	It suffices to bound $\fat_\alpha\!\big(\cH_{L,k,d}^{\mathrm{deep}}\big)$.
	We do so via the sequential Rademacher complexity, using a
	\emph{vector Rademacher peeling argument} that exploits the row independence of the weight matrices.
	
	\medskip\noindent\textbf{Vector Rademacher recursion.}\;
	For a class of vector-valued functions $\cV\subseteq(\R^k)^{\cX}$ and a sequential tree $z=(z_t)_{t=1}^T$ of depth $T$ (i.e., each $z_t:\{\pm1\}^{t-1}\to\cX$ and $z_t(\eps)$ denotes $z_t(\eps_1,\ldots,\eps_{t-1})$), define
	\[
	\VRad_T(\cV)
	\ :=\
	\frac{1}{T}\,
	\E_\eps\!\Big[\sup_{v\in\cV}\,\Big\|\sum_{t=1}^T \eps_t\, v(z_t(\eps))\Big\|_2\Big]
	\]
	where $\eps_1,\ldots,\eps_T \sim \mathrm{Unif}\{\pm 1\}$ independently.
	
	Let $\cV_\ell$ ($\ell=1,\dots,L-1$) denote the class of layer-$\ell$ activation vectors
	$z^{(\ell)}:\cX\to\R^k$ (ranging over all admissible weight and bias choices at layers $1,\dots,\ell$).
	We prove:
	\begin{claim}
		\begin{equation}\label{eq:VRad-recursion}
			\VRad_T(\cV_\ell)
			\ \le\
			\sqrt{k}\,\rho(T)\cdot\VRad_T(\cV_{\ell-1})
			\ +\ O\left(\frac{\sqrt k\,\rho(T)}{\sqrt T}\right),
			\qquad \ell\ge 2,
		\end{equation}
	\end{claim}
	where $\rho(T)=O(\log^{3/2}(eT))$ is the sequential contraction overhead from
	\cite[Cor.~5]{rakhlin2015online}.
	
	\medskip\noindent\textbf{Normalization for Cor.~5.}\;
	Corollary~5 in \cite{rakhlin2015online} is stated for classes $G\subseteq[-1,1]^{\cZ}$.
	If $G\subseteq\R^{\cZ}$ satisfies $\sup_{g\in G,z\in\cZ}|g(z)|\le B<\infty$, define
	$\bar G:=\{g/B:\ g\in G\}\subseteq[-1,1]^{\cZ}$ and $\phi_B(u,z):=\phi(Bu,z)$.
	Then $\phi_B(\cdot,z)$ is $(B \lambda)$-Lipschitz whenever $\phi(\cdot,z)$ is $\lambda$-Lipschitz, and
	by Lemma~3 and Cor.~5 of \cite{rakhlin2015online} we have, whenever $\Rad_T(\bar G)\ge 1/T$,
	\begin{equation}\label{eq:cor5-normalized}
		\Rad_T(\phi\circ G)=\Rad_T(\phi_B\circ \bar G)\ \le\ \rho(T)\,\lambda\cdot \Rad_T(G).
	\end{equation}
	In our applications $G$ is symmetric and contains the constants $\pm 1$ (via the bias), hence
	$\Rad_T(\bar G)\ge c/(B\sqrt{T})$ for a universal constant $c>0$, so the condition
	$\Rad_T(\bar G)\ge 1/T$ holds for all $T\ge c^{-2}B^2$, which is satisfied by the choices of $T$ used below.
	
	\begin{proof}
		Fix a sequential tree $z$ and write $y_t:=z^{(\ell-1)}(z_t(\eps))\in\R^k$
		for the layer-$(\ell-1)$ activations (which depend on $\eps_{<t}$ and on the choice of all
		lower-layer parameters).
		The layer-$\ell$ activations are $z^{(\ell)}_t=\relu(W_\ell y_t+b_\ell)$, with $W_\ell$
		having $k$ rows $w_1,\dots,w_k$ satisfying $\|w_j\|_2\le 1$ and $b_\ell\in[-1,1]^k$.
		Since the $k$ rows act independently for fixed $y$ (where each row $w_j$ only affects coordinate $j$),
		\begin{align*}
			\sup_{W_\ell,b_\ell}\;
			\Big\|\sum_{t=1}^T \eps_t\,\relu(W_\ell y_t+b_\ell)\Big\|_2^2
			&= \sum_{j=1}^k \sup_{\|w\|_2\le 1,\,|b|\le 1}\;
			\Big(\sum_{t=1}^T \eps_t\,\relu(\langle w,y_t\rangle+b)\Big)^{\!2}\\
			&\le k\cdot\sup_{\|w\|_2\le 1,\,|b|\le 1}\;
			\Big(\sum_{t=1}^T \eps_t\,\relu(\langle w,y_t\rangle+b)\Big)^{\!2},
		\end{align*}
		where the inequality holds because all $k$ optimization problems have the same feasible set.
		Taking square roots:
		\[
		\sup_{W_\ell,b_\ell}\Big\|\sum_{t=1}^T\eps_t\,\relu(W_\ell y_t+b_\ell)\Big\|_2
		\le \sqrt k\cdot\sup_{\|w\|_2\le 1,\,|b|\le 1}\Big|\sum_{t=1}^T\eps_t\,\relu(\langle w,y_t\rangle+b)\Big|.
		\]
		Let $G:=\{t\mapsto \langle w,y_t\rangle+b:\ \|w\|_2\le 1,\ |b|\le 1\}$ and $\phi(u):=\relu(u)$.
		By \eqref{eq:cor5-normalized} (with $\lambda=1$),
		\[
		\E_\eps\Big[\sup_{\|w\|_2\le 1,\,|b|\le 1}\Big|\sum_{t=1}^T\eps_t\,\relu(\langle w,y_t\rangle+b)\Big|\Big]
		\ \le\
		\rho(T)\cdot
		\E_\eps\Big[\sup_{\|w\|_2\le 1,\,|b|\le 1}\Big|\sum_{t=1}^T\eps_t\,(\langle w,y_t\rangle+b)\Big|\Big].
		\]
		The right-hand side splits as
		\[
		\sup_{\|w\|_2\le 1}\Big|\Big\langle w,\sum_{t=1}^T\eps_t y_t\Big\rangle\Big|
		\ +\ 
		\sup_{|b|\le 1}\Big|b\sum_{t=1}^T\eps_t\Big|
		\ =\
		\Big\|\sum_{t=1}^T\eps_t y_t\Big\|_2
		\ +\ 
		\Big|\sum_{t=1}^T\eps_t\Big|.
		\]
		Moreover $\E_\eps\big|\sum_{t=1}^T\eps_t\big|=O(\sqrt{T})$.
		Therefore, taking $\E_\eps[\cdot]$, then $\sup$ over lower-layer parameters, and dividing by $T$ yields~\eqref{eq:VRad-recursion}.
	\end{proof}
	
	\medskip\noindent\textbf{Base case.}\;
	For layer $1$, the same argument gives
	\[
	\VRad_T(\cV_1)
	\ =\
	O\!\left(\frac{\sqrt{k}\,\rho(T)}{\sqrt{T}}\right).
	\]
	
	\medskip\noindent\textbf{Unrolling the recursion.}\;
	By induction, using~\eqref{eq:VRad-recursion} for $\ell=2,\dots,L-1$ and the base case:
	\[
	\VRad_T(\cV_{L-1})
	\ \le\
	\frac{C_L\,(\sqrt{k})^{L-1}\rho(T)^{L-1}}{\sqrt T}.
	\]
	Let $\cG$ be the \emph{pre-clipped} scalar class
	$\cG:=\{x\mapsto \langle a,z^{(L-1)}(x)\rangle+c:\ a\in[-1,1]^k,\ c\in[-1,1]\}$.
	Since $\clip$ is $1$-Lipschitz, applying \eqref{eq:cor5-normalized} to $\clip\circ \cG$ yields
	\[
	\Rad_T\!\big(\cH_{L,k,d}^{\mathrm{deep}}\big)
	\ \le\
	\rho(T)\cdot \Rad_T(\cG).
	\]
	Moreover,
	\[
	\Rad_T(\cG)
	\ \le\
	\sqrt{k}\cdot\VRad_T(\cV_{L-1})\ +\ O\!\Big(\frac{1}{\sqrt{T}}\Big),
	\]
	where the $\sqrt{k}$ is by Cauchy--Schwarz and the $O(1/\sqrt{T})$ term comes from the
	free offset $c\in[-1,1]$.
	Consequently,
	\[
	\Rad_T\!\big(\cH_{L,k,d}^{\mathrm{deep}}\big)
	\ \le\
	\frac{C'_L\,k^{L/2}\rho(T)^{L}}{\sqrt{T}}.
	\]
	
	\medskip\noindent\textbf{Fat-shattering bound.}\;
	By \cite[Prop.~9]{rakhlin2015online}:
	$$\alpha\sqrt{\min\{\fat_\alpha,T\}/T}
	\le C\,\Rad_T(\cH_{L,k,d}^{\mathrm{deep}})
	\le C'_L\,k^{L/2}\rho(T)^{L}/\sqrt T.$$
	Choosing
	$T=\Theta\!\big(\alpha^{-2}k^{L}\log^{3L}(ekL/\alpha)\big)$
	forces $\fat_\alpha<T$, hence
	\[
	\fat_{\alpha}\!\big(\cH_{L,k,d}^{\mathrm{deep}}\big)
	\ \le\
	C_L''\,\frac{k^{L}}{\alpha^2}\,
	\log^{3L}\!\Big(\frac{ekL}{\alpha}\Big).
	\]
	Plugging into the truncated multiscale sum:
	\[
	\sum_{\Delta_u>1/\sqrt d}\Delta_u^2
	\ \le\
	C_L\,k^{L}
	\sum_{r=-1}^{K}\log^{3L}\!\big(ekL/\alpha_r\big)
	\ =\
	O\!\Big(k^{L}\log^{3L+1}(edkL)\Big).
	\]
	
	\paragraph{Small gaps.}
	Apply Lemma~\ref{lem:small-gap-amplification-general-domain} with
	$\cH=\cH_{L,k,d}^{\mathrm{deep}}$ and $\cF=\cF_{L,k,d}^{\mathrm{deep}}$:
	for the same branch $P$,
	\[
	\sum_{\Delta_u\le 1/\sqrt d}\Delta_u^2
	\ \le\
	\frac{4}{d}\,\mathbb D_{\mathrm{onl}}(\cF_{L,k,d}^{\mathrm{deep}};\ell_{\cap}).
	\]
	We bound $\mathbb D_{\mathrm{onl}}(\cF_{L,k,d}^{\mathrm{deep}};\ell_{\cap})$ via the entropy potential.
	Let $K_2$ denote the $\ell_2$ parameter Lipschitz constant:
	$\|h_\theta-h_{\theta'}\|_\infty\le K_2\|\theta-\theta'\|_2$.
	Since $\relu$ is $1$-Lipschitz and every weight matrix has spectral norm at most $\sqrt{k}$
	(each of $k$ rows has $\ell_2$-norm $\le 1$), a standard inductive argument
	(analogous to Lemma~\ref{lem:deep-param-lip} given in the next sections) gives $K_2=O\big(k^{(L-1)/2}\sqrt{L}\big)$.
	Multiplying by $\sqrt d$ for $\cF_{L,k,d}^{\mathrm{deep}}$, a volumetric cover of the
	parameter set $\Theta\subseteq B_2(0,\sqrt{2Lk})$ in $\R^p$ yields
	\[
	\log N(\cF_{L,k,d}^{\mathrm{deep}},\eps;\ell_\cap)
	\ \le\
	p\,\log\!\Big(\frac{C\sqrt{Lkd}\,K_2}{\sqrt\eps}\Big),
	\qquad \eps\in(0,1].
	\]
	Integrating:
	$\Phi(\cF_{L,k,d}^{\mathrm{deep}})=O\!\big(p\,\log(C\sqrt{Lkd}\,K_2)\big)=O(pL\log(kd))$, and hence
	$\mathbb D_{\mathrm{onl}}(\cF_{L,k,d}^{\mathrm{deep}};\ell_{\cap})=O(pL\log(kd))$.
	Therefore
	\[
	\sum_{\Delta_u\le 1/\sqrt d}\Delta_u^2
	\ =\
	O\!\Big(\frac{pL\log(kd)}{d}\Big)
	\ =\
	O\!\big(kL\log(kd)\big),
	\]
	since the $kd/d=k$ term dominates for large $d$.
	
	\medskip
	Adding large- and small-gap contributions:
	\[
	\sum_{u\in P}\Delta_u^2
	=
	O\!\Big(k^L\log^{3L+1}(edkL)+kL\log(kd)\Big)
	=
	O\!\Big(k^L\log^{3L+1}(edkL)\Big).
	\]
	Using $\ell_{\cap}(s_{u,0},s_{u,1})=\Delta_u^2/4$ and the same reduction to squared loss as in
	\Cref{thm:intro_k_bounded} proves the claim.
\end{proof}

\label{rem:deep-lb-kLminus1}
\subsubsection*{A matching $\Omega(k^L \log (d/k^L))$ lower bound}
In the following, we prove a matching lower bound. The proof is based on generalizing the ideas from Theorem~\ref{thm:LBd} to depth $L$ networks.  

\begin{lemma}
Let $L\ge2$ be fixed. If $d/k^L\to\infty$, then
\[
\mathbb D_{\mathrm{onl}}\!\big(\cH_{L,k,d}^{\mathrm{deep}};\ell_{\cap}\big)
\ \ge\
\Omega_L\!\left(k^L\log(d/k^L)\right).
\]
\end{lemma}

\begin{proof}
It is enough to prove the claim for even $k$. If $k$ is odd, use only $k-1$ neurons in each hidden layer and set all other rows and output weights to zero; since $L$ is fixed, this changes the lower bound only by a constant factor. Hence assume $k$ is even.

First assume that $d$ is a power of two. For each $s=(s_1,\ldots,s_d)\in\{\pm1\}^d$, let $v_s=s/\sqrt d$. In the first layer, set the first $k/2$ rows of $W_1$ equal to $v_s$, set the last $k/2$ rows equal to $-v_s$, and set $b_1=0$. For every $\ell=2,\ldots,L-1$, set $W_\ell=k^{-1/2}\mathbf 1\mathbf 1^\top$ and $b_\ell=0$. Finally set $a=\mathbf 1$ and $c=0$. All constraints are satisfied: $\|v_s\|_2=\|-v_s\|_2=1$, every row of $k^{-1/2}\mathbf 1\mathbf 1^\top$ has norm $1$, all biases are zero, and all output weights are in $[-1,1]$.

For this network, the unclipped output is $(k^{L/2}/2)|\langle v_s,x\rangle|$. Indeed, after the first layer the sum of the $k$ activations is $(k/2)|\langle v_s,x\rangle|$. Each fixed hidden layer maps a nonnegative vector with coordinate-sum $R$ to the constant vector whose coordinates are all $R/\sqrt k$, so the new coordinate-sum is $\sqrt k R$. Applying this through the $L-2$ fixed hidden layers gives final pre-clipped output $k^{(L-2)/2}(k/2)|\langle v_s,x\rangle|=(k^{L/2}/2)|\langle v_s,x\rangle|$. Thus the class contains the functions
\[
h_s(x)=\clip\left(\frac{k^{L/2}}2|\langle v_s,x\rangle|\right),
\qquad s\in\{\pm1\}^d.
\]

We construct a realizable scaled Littlestone tree for this subclass. The construction maintains a partition of $[d]=\{1,\ldots,d\}$ into dyadic intervals. For a block $B\subseteq[d]$ and a sign vector $s\in\{\pm1\}^d$, write $s_B=(s_i)_{i\in B}$. Each current block $B$ has a public reference pattern $r_B\in\{\pm1\}^B$. The invariant is that the surviving signs are exactly those $s$ such that, for every current block $B$, either $s_B=r_B$ or $s_B=-r_B$. Initially the blocks are singletons and $r_{\{i\}}(i)=1$, so all signs survive.

Consider a merge of two adjacent current blocks $B_0,B_1$, each of size $m/2$, and let $B=B_0\cup B_1$. The tree node is labeled by
\[
x_B
=
\frac1{\sqrt m}
\left(
\sum_{i\in B_0}r_{B_0}(i)e_i
+
\sum_{i\in B_1}r_{B_1}(i)e_i
\right).
\]
Then $\|x_B\|_2^2=(|B_0|+|B_1|)/m=1$. If the two child blocks have the same orientation, meaning either $s_{B_0}=r_{B_0}$ and $s_{B_1}=r_{B_1}$, or $s_{B_0}=-r_{B_0}$ and $s_{B_1}=-r_{B_1}$, then $|\langle v_s,x_B\rangle|=\sqrt{m/d}$. If they have opposite orientations, meaning either $s_{B_0}=r_{B_0}$ and $s_{B_1}=-r_{B_1}$, or $s_{B_0}=-r_{B_0}$ and $s_{B_1}=r_{B_1}$, then $|\langle v_s,x_B\rangle|=0$. Thus the two outgoing labels
\[
0
\qquad\text{and}\qquad
\alpha_m=\min\left\{\frac{k^{L/2}}2\sqrt{\frac md},1\right\}
\]
are both realizable by surviving hypotheses.

After the high-label branch, keep exactly the surviving signs with the same orientation and define $r_B$ by $r_B(i)=r_{B_0}(i)$ for $i\in B_0$ and $r_B(i)=r_{B_1}(i)$ for $i\in B_1$. After the zero-label branch, keep exactly the surviving signs with opposite orientations and define $r_B$ by $r_B(i)=r_{B_0}(i)$ for $i\in B_0$ and $r_B(i)=-r_{B_1}(i)$ for $i\in B_1$. In both cases the surviving restrictions on $B$ are exactly $r_B$ and $-r_B$, while all other blocks are unchanged. Hence the invariant is preserved, and the surviving set is nonempty on every branch.

The tree performs all merges bottom-up. At a level whose merged block size is $m$, there are $d/m$ nodes. For every level with $m\le4d/k^L$, no clipping occurs, so $\alpha_m^2=k^L m/(4d)$. Since $\alpha_m\in[0,1]$ on these levels, $\ell_{\cap}(0,\alpha_m)=\alpha_m^2/4=k^L m/(16d)$. Therefore the total gap contributed by such a level is $(d/m)\cdot k^Lm/(16d)=k^L/16$. The dyadic merged block sizes are $m=2,4,\ldots,d$, and the condition $m\le4d/k^L$ holds for $\Omega(\log(d/k^L))$ levels whenever $d/k^L\to\infty$. Hence every branch has total $\ell_{\cap}$-gap at least $\Omega(k^L\log(d/k^L))$.

The tree is realizable by the invariant: along every branch, at least one sign vector survives, and the corresponding fixed hypothesis $h_s$ realizes all labels on that branch. Therefore $$\mathbb D_{\mathrm{onl}}(\cH_{L,k,d}^{\mathrm{deep}};\ell_{\cap})\ge\Omega(k^L\log(d/k^L))$$ when $d$ is a power of two.

For general $d$, run the same construction on the first $D=2^{\lfloor\log_2 d\rfloor}$ coordinates and set all remaining coordinates to zero in the queries and in the vectors $v_s$. Since $D\ge d/2$, the same lower bound follows whenever $d/k^L\to\infty$.
\end{proof}

	We conclude this subsection with the following discussion. 
	
	\paragraph{The covering route and small $d$.} A complementary bound avoids the exponential $k^L$ term. Since the class has $p=kd+O(Lk^2)$ parameters and the parameter-to-function Lipschitz constant $K_2=O(\sqrt L\,k^{(L-1)/2})$ enters only logarithmically in the covering number, the entropy-potential method gives \[ \mathbb D_{\mathrm{onl}}\!\big(\cH_{L,k,d}^{\mathrm{deep}};\ell_{\cap}\big) \le O\big((kd+Lk^2)L\log(kd)\big). \] Thus, for $d\le k^{L-1}$, this is at most $O(k^L\operatorname{polylog}(kdL))$, and is the better route in the small-$d$ regime. In the polylogarithmic-in-$d$ regime, the $k^L$ dependence is necessary up to polylogarithmic factors in $d$: whenever $d/k^L\to\infty$, the lower bound above gives \[ \mathbb D_{\mathrm{onl}}\!\big(\cH_{L,k,d}^{\mathrm{deep}};\ell_{\cap}\big) \ge \Omega_L\!\left(k^L\log(d/k^L)\right). \]

\section{Proper and $\eps$-efficient Lower Bound for Bounded  $k$-ReLUs}
\label{sec:proper-online-hard-relu}

For a sample $S=\{(x_i,y_i)\}_{i=1}^m\subseteq \{x\in\R^d:\|x\|_2\le 1\}$ define
\[
L_S(h)\ :=\ \frac1m\sum_{i=1}^m (h(x_i)-y_i)^2,
\qquad
\OPT(S)\ :=\ \min_{h\in\cH_{k,d}} L_S(h).
\]
We call $S$ \emph{realizable} if $\OPT(S)=0$.
An algorithm is \emph{proper} if it outputs $\hat h\in\cH_{k,d}$.
We call it \emph{$\eps$-efficient} if its running time is $2^{o(1/\eps)}\poly(d,m)$.
An algorithm achieves \emph{realizable average loss $\le \eps$}
if on every realizable sequence of length $T$ it guarantees
\[
\frac1T\sum_{t=1}^T (h_t(x_t)-y_t)^2\ \le\ \eps.
\]

For a network returned as parameters $(a,w^1,\ldots,w^k)$, write its \emph{unclipped} form as
\[
\tilde h(x)\ :=\ \sum_{j=1}^k a_j\,\rel(w^j\!\cdot x),
\qquad\text{and}\qquad
h(x)\ :=\ \mathrm{clip}(\tilde h(x)),
\]
where $\mathrm{clip}(z):=\min\{1,\max\{-1,z\}\}$. Note that for $\|x\|_2\le 1$ and $\|w^j\|_2\le 1$,
we have $0\le \rel(w^j\cdot x)\le 1$, hence $|\tilde h(x)|\le \sum_{j=1}^k |a_j|\le k$.

\begin{theorem}[Proper batch hardness \cite{DBLP:conf/innovations/GoelKM021}]
	\label{thm:proper-hard-d}
	Assume \textnormal{Gap-ETH}~\cite{DBLP:journals/eccc/Dinur16,DBLP:journals/corr/ManurangsiR16} and fix any constant $k\ge 2$.
	There is no (possibly randomized) $\eps$-efficient proper algorithm that, for every $\eps>0$,
	every $d\in \N$, and 
	every realizable sample $S\subseteq \{x\in\R^d:\|x\|_2\le 1\}\times[-\tfrac12,\tfrac12]$ of size $m$,
	outputs $\hat h\in\cH_{k,d}$ with $L_S(\hat h)\le \eps$.
	In particular, taking $\eps:=1/d$ rules out running time $2^{o(d)}\poly(d,m)$.
\end{theorem}

\begin{proof}
	Let $C_k := 1+4\bigl(k+\tfrac12\bigr)^2$.
	Suppose, towards a contradiction, that there is an $\eps$-efficient proper algorithm $\cA$ for our
	(clipped) class $\cH_{k,d}$.
	
	Given any realizable sample $S=\{(x_i,y_i)\}_{i=1}^m$ with $\|x_i\|_2\le 1$ and $y_i\in[-\tfrac12,\tfrac12]$,
	run $\cA$ with accuracy parameter $\eps' := \eps/C_k$ and obtain a hypothesis
	$h=\mathrm{clip}(\tilde h)\in\cH_{k,d}$, where $\tilde h(x)=\sum_{j=1}^k a_j\rel(w^j\cdot x)$
	is the unclipped network.
	
	Let $B:=\{i:\ |\tilde h(x_i)|>1\}$ be the indices where clipping is active. For every $i\in B$ we have
	$h(x_i)\in\{-1,1\}$ and $|y_i|\le\tfrac12$, hence $(h(x_i)-y_i)^2\ge (1-\tfrac12)^2=\tfrac14$.
	Therefore,
	\[
	\frac{|B|}{m}\cdot \frac14 \ \le\ \frac1m\sum_{i\in B}(h(x_i)-y_i)^2 \ \le\ L_S(h)\ \le\ \eps',
	\quad\text{so}\quad \frac{|B|}{m}\le 4\eps'.
	\]
	On $[m]\setminus B$ we have $h(x_i)=\tilde h(x_i)$, while on $B$ we use $|\tilde h(x_i)|\le k$ and $|y_i|\le\tfrac12$
	to bound $(\tilde h(x_i)-y_i)^2\le (k+\tfrac12)^2$. Thus,
	\[
	L_S(\tilde h)
	\ \le\ L_S(h) \;+\; \frac{|B|}{m}\bigl(k+\tfrac12\bigr)^2
	\ \le\ \eps' + 4\eps'\bigl(k+\tfrac12\bigr)^2
	\ =\ C_k\,\eps'
	\ =\ \eps.
	\]
	
	Finally, convert $\tilde h$ to the exact parameterization of \cite{DBLP:conf/innovations/GoelKM021} by
	setting $\bar a_j:=\mathrm{sign}(a_j)\in\{\pm1\}$ (and if $a_j=0$ take $\bar a_j=1$) and
	$\bar w^j := |a_j|\,w^j$; then $\|\bar w^j\|_2\le 1$ and
	$\sum_{j=1}^k \bar a_j\rel(\bar w^j\!\cdot x)=\sum_{j=1}^k a_j\rel(w^j\!\cdot x)=\tilde h(x)$.
	Moreover, given any realizable instance of the bounded $k$-ReLU training problem with labels in $[-k,k]$,
	we can scale the labels by a factor $\alpha\le 1/(2k)$ (and keep the same inputs) to obtain an equivalent
	realizable instance with labels in $[-\tfrac12,\tfrac12]$; an $\eps$-approximate solution for the scaled instance
	yields a $(\eps/\alpha^2)$-approximate solution for the original instance.
	Choosing $\alpha$ as a fixed constant (depending only on $k$) preserves $\eps$-efficiency up to constant factors
	in $1/\eps$, and thus contradicts the results of \cite{DBLP:conf/innovations/GoelKM021}.
\end{proof}

\begin{lemma}
	\label{lem:online-to-batch-proper}
	Suppose there exists an $\eps$-efficient proper online algorithm $\cA$ that achieves realizable
	average loss $\le \eps$ for $\cH_{k,d}$ for every horizon $T$.
	Then there exists an $\eps$-efficient proper batch algorithm that, on every realizable sample
	$S$ of size $m$, outputs $\hat h\in\cH_{k,d}$ with $L_S(\hat h)\le \eps$.
\end{lemma}

\begin{proof}
	Given a realizable sample $S$, run $\cA$ for $T:=m$ rounds on a randomized online sequence obtained by sampling
	indices $i_t\sim\mathrm{Unif}([m])$ independently and feeding $(x_t,y_t)=(x_{i_t},y_{i_t})$ to $\cA$.
	Record the hypotheses $h_1,\dots,h_T\in\cH_{k,d}$ output by $\cA$.
	
	Conditioned on the history up to round $t$, we have
	$\E[(h_t(x_{i_t})-y_{i_t})^2 \mid h_t]=L_S(h_t)$.
	Hence
	\[
	\E\!\left[\frac1T\sum_{t=1}^T L_S(h_t)\right]
	=\E\!\left[\frac1T\sum_{t=1}^T (h_t(x_{i_t})-y_{i_t})^2\right]\le \eps.
	\]
	Therefore $\E[\min_{t\le T} L_S(h_t)]\le \eps$ (since $\min\le$ average). By Markov,
	$\Pr\big[\min_{t\le T} L_S(h_t) > 2\eps\big]\le \tfrac12$.
	Compute $L_S(h_t)$ for all $t$ and output $\hat h:=\arg\min_{t\le T} L_S(h_t)$; then
	$\Pr[L_S(\hat h)\le 2\eps]\ge \tfrac12$.
	Run the procedure $O(\log(1/\delta))$ times and output the best hypothesis found to boost success probability,
	and replace $\eps$ by $\eps/2$ to obtain $L_S(\hat h)\le \eps$ with high constant probability.
	The runtime remains $2^{o(1/\eps)}\poly(d,m)$.
\end{proof}

\begin{corollary}[Proper online hardness]
	\label{cor:proper-online-hard}
	Assume Gap-ETH and fix any constant $k\ge 2$.
	There is no $\eps$-efficient proper online algorithm that achieves realizable average squared loss $\le \eps$
	for $\cH_{k,d}$ for all $\eps>0$ and all horizons $T$.
	In particular, for $T=d$ and any $\eps=\eps(d)$ with $\eps(d)=\omega(1/\log d)$ and $\eps(d)=o(1)$, this rules out any
	$\poly(d)$-time proper online algorithm that, given $\eps(d)$, guarantees on all realizable length-$d$ sequences
	\[
	\sum_{t=1}^d (h_t(x_t)-y_t)^2 \ \le\ \eps(d)\,d \ =\ o(d).
	\]
\end{corollary}

\begin{proof}
	Otherwise Lemma~\ref{lem:online-to-batch-proper} would yield an $\eps$-efficient proper batch algorithm,
	contradicting Theorem~\ref{thm:proper-hard-d}.
\end{proof}

\section{Classification}
In the following sections, we focus on regressions for achieving loss bounds independent of the time horizon. For classification on the other hand, even highly restricted function classes have unbounded loss as we show next.

For now, we consider the family $\cN$ of $\relu$ networks with input and output dimension $1$, a single hidden layer with two hidden neurons, and such that the absolute values of all weights and biases are bounded by $1$. We assume the input domain is $\cX=[-1,1]$.
In the following theorem, we strengthen \cite[Theorem 4.13]{geneson2025mistake} in two orthogonal criteria: most importantly, even though the nets in $\cN$ have bounded weights and biases, we prove that the Littlestone dimension of $\cN$ is still infinite. Less importantly, we also use networks with only two hidden neurons, instead of $4$.

\begin{theorem}\label{thm:LD-infty-two-relu}
	We have $\LD(\cN)=\infty$.
\end{theorem}

\begin{proof}
	Fix $\eps\in(0,1]$. For every $\theta\in[-1+\eps,1]$, define
	\[
	f_{\eps,\theta}(x)\ :=\ \relu(\theta-x)\;-\;\relu(\theta-x-\eps),\qquad x\in[-1,1].
	\]
	We first verify that $f_{\eps,\theta}\in\cN$. Consider a two-hidden-unit ReLU network
	\[
	x\ \mapsto\ a_1\,\relu(w_1 x+b_1)\;+\;a_2\,\relu(w_2 x+b_2)\;+\;b .
	\]
	Set
	\[
	w_1=w_2=-1,\qquad b_1=\theta,\qquad b_2=\theta-\eps,\qquad a_1=1,\qquad a_2=-1,\qquad b=0.
	\]
	All parameters have absolute value at most $1$ since $\theta\in[-1+\eps,1]$ and $\eps\le 1$, and this choice implements $f_{\eps,\theta}$.
	
	Next, note that for any $x\in[-1,1]$:
	\begin{equation}\label{eq:f-threshold-behavior}
		f_{\eps,\theta}(x)=
		\begin{cases}
			0, & x\ge \theta,\\
			\eps, & x\le \theta-\eps,
		\end{cases}
	\end{equation}
	(and for $x\in(\theta-\eps,\theta)$ the value lies in $(0,\eps)$).
	Indeed, if $x=\theta+\delta$ with $\delta\ge0$, then
	\[
	f_{\eps,\theta}(x)=\relu(-\delta)-\relu(-\delta-\eps)=0-0=0.
	\]
	If $x=\theta-\eps'$ with $\eps'\ge \eps$, then
	\[
	f_{\eps,\theta}(x)=\relu(\eps')-\relu(\eps'-\eps)=\eps'-(\eps'-\eps)=\eps.
	\]
	
	We now show that for every $d\in\mathbb{N}$ there exists a complete binary tree of depth $d$ that is shattered by
	\[
	\cF_\eps\ :=\ \{f_{\eps,\theta}:\ \theta\in[-1+\eps,1]\}\subseteq \cN,
	\]
	and hence $\LD(\cN)=\infty$.
	
	Fix $d\in\mathbb{N}$ and set $\eps:=2^{-d-2}$. We construct a depth-$d$ shattered tree as follows.
	Label every left edge by $0$ and every right edge by $\eps$.
	For each node $v$ we maintain an interval $I_v=[L_v,U_v]\subseteq[-1+\eps,1]$ of thresholds consistent with the labels along the path from the root to $v$.
	
	At the root, set $I_{\mathrm{root}}=[-1+\eps,1]$.
	Given an internal node $v$ with $I_v=[L_v,U_v]$ satisfying $|I_v|:=U_v-L_v>\eps$, label $v$ by
	\[
	x_v\ :=\ \frac{L_v+U_v-\eps}{2}.
	\]
	Then $x_v\in[L_v,U_v]$ and $x_v+\eps\in[L_v,U_v]$.
	Let $v_\ell$ and $v_r$ denote the left and right children of $v$. Define
	\begin{equation}\label{eq:update}
		I_{v_\ell}:=[L_v,x_v],\qquad I_{v_r}:=[x_v+\eps,U_v].
	\end{equation}
	These intervals are nonempty since $|I_v|>\eps$.
	Moreover, by~\eqref{eq:f-threshold-behavior}, for every $\theta\in I_{v_\ell}$ we have $f_{\eps,\theta}(x_v)=0$,
	and for every $\theta\in I_{v_r}$ we have $f_{\eps,\theta}(x_v)=\eps$. Thus the labeling choice at $v$ is realizable by $\cF_\eps$ and the recursive update maintains the consistency invariant.
	
	It remains to verify that along every path, the recursion can proceed for at least $d$ levels, i.e., that $|I_v|>\eps$ for all nodes $v$ of depth at most $d$.
	Let $\ell_t$ denote the length $|I_v|$ at depth $t$ (it is the same for all nodes at the same depth by symmetry of the construction).
	From~\eqref{eq:update},
	\[
	\ell_0 = 2-\eps,\qquad
	\ell_t = \frac{\ell_{t-1}-\eps}{2}\quad\text{for }t\ge 1.
	\]
	Solving yields $\ell_t = 2^{-t+1}-\eps$, which can be checked by induction:
	\[
	\ell_t=\frac{(2^{-t+2}-\eps)-\eps}{2}=2^{-t+1}-\eps.
	\]
	In particular,
	\[
	\ell_d = 2^{-d+1}-\eps = 2^{-d+1}-2^{-d-2} = 2^{-d-2}(8-1) > 2^{-d-2}=\eps.
	\]
	Therefore all nodes up to depth $d$ satisfy $|I_v|>\eps$, so the constructed tree is a complete depth-$d$ shattered tree for $\cF_\eps\subseteq\cN$.
	Since $d$ was arbitrary, $\LD(\cN)=\infty$.
\end{proof}

The above theorem establishes that with $0/1$ loss, even the simple nets in $\cN$ are unlearnable.
When we change the loss to squared loss, we shift to a regression problem, which is potentially easier.
In this paper we show that $\cN$ admits a finite realizable cumulative-loss guarantee under squared loss.

\section{Implications of the entropy-potential bound}

In this section we record several immediate consequences of
Theorem~\ref{thm:potential-bound-general}. Throughout, $\ell$ is a $c$-approximate pseudo-metric
(Definition~\ref{def:approx-pseudometric}), $d_\ell$ is the induced sup pseudo-metric
\eqref{eq:induced-metric}, and $\Phi(\cdot)$ is the entropy potential \eqref{eq:potential}.
Recall that the \emph{online dimension} $\mathbb{D}_{\mathrm{onl}}(\cH)$ is defined as the supremum, over all
scaled Littlestone trees $\tree$ for $\cH$, of the minimum branch-sum of gaps
$\inf_{y\in\cP(\tree)}\sum_i \gamma_{y_i}$ (see \Cref{sec:prel}).

\subsection{Bounding the online dimension by metric entropy}

Theorem~\ref{thm:potential-bound-general} immediately yields a bound on the online dimension in terms
of the entropy potential. This gives the proof of \Cref{thm:intro_Donl-via-Phi}. 


\begin{proof}[Proof of \Cref{thm:intro_Donl-via-Phi}]
	Fix any scaled Littlestone tree $\tree$ for $\cH$. By Theorem~\ref{thm:potential-bound-general},
	there exists a branch $y\in\cP(\tree)$ such that $\sum_i \gamma_{y_i}\le 4c\,\Phi(\cH)$.
	Therefore,
	\[
	\inf_{y\in\cP(\tree)}\sum_i \gamma_{y_i}\ \le\ 4c\,\Phi(\cH).
	\]
	Taking the supremum over $\tree$ in the definition of $\mathbb{D}_{\mathrm{onl}}(\cH)$ gives the claim.
\end{proof}

Combining \Cref{thm:intro_Donl-via-Phi} with the optimal realizable cumulative-loss
characterization in terms of $\mathbb{D}_{\mathrm{onl}}(\cH)$ of \cite{attias2023optimal} given in \Cref{thm:SLD}, yields an
immediate realizable cumulative-loss finite upper bound whenever $\Phi(\cH)<\infty$. We give more concrete consequences in what follows. 

\subsection{Parametric (polynomial) covering numbers}

A common regime is \emph{polynomial metric entropy}, which typically arises for classes
parameterized by a bounded number of parameters and that are Lipschitz in the parameters
with respect to $d_\ell$. In this regime, the covering numbers satisfy
$N(\cH,\varepsilon)\lesssim (C/\varepsilon)^p$ for some $p$, so the entropy potential
$\int_0^{\diam(\cH)} \log N(\cH,\varepsilon)\,d\varepsilon$ is finite and grows linearly in $p$
(up to constants depending on $C$ and $\diam(\cH)$).

\begin{corollary}[Polynomial covers imply finite potential]\label{cor:poly-covers}
	Assume $\diam(\cH)\le 1$ and there exist constants $A\ge 1$ and $p\ge 1$ such that for all
	$\eps\in(0,1]$,
	\[
	N(\cH,\eps)\ \le\ \Big(\frac{A}{\eps}\Big)^p.
	\]
	Then
	\[
	\Phi(\cH)\ \le\ p\big(\log_2 A + 1/\ln 2\big)
	\qquad\text{and}\qquad
	\mathbb{D}_{\mathrm{onl}}(\cH)\ \le\ 4c\,p\big(\log_2 A + 1/\ln 2\big).
	\]
\end{corollary}

\begin{proof}
	By the assumed covering bound,
	\[
	\log_2 N(\cH,\eps)\ \le\ p\log_2(A/\eps)\ =\ p\log_2 A + p\log_2(1/\eps).
	\]
	Integrating over $\eps\in(0,1]$ gives
	\[
	\Phi(\cH)
	\le p\log_2 A + p\int_0^1 \log_2(1/\eps)\,d\eps
	= p\log_2 A + \frac{p}{\ln 2},
	\]
	using $\int_0^1 \ln(1/\eps)\,d\eps = 1$.
	The bound on $\mathbb{D}_{\mathrm{onl}}(\cH)$ follows from \Cref{thm:intro_Donl-via-Phi}. 
\end{proof}

\section{From parameter covers to entropy potentials}
\label{sec:param-to-potential}

This section provides a generic route to upper bound the entropy potential $\Phi(\cH)$ (and hence the
online dimension via \Cref{thm:intro_Donl-via-Phi}) for classes that admit a Lipschitz
parameterization. The key observation is that a covering of the parameter set induces a covering of
the hypothesis class under the loss-induced pseudo-metric.
Fix an instance domain $\cX$, a label space $\cY$, and an approximate pseudo-metric loss $\ell:\cY\times\cY\to\R_{\ge 0}$.

\begin{definition}[Modulus-Lipschitz parameterization]
	\label{def:modulus-lipschitz}
	Let $(\Theta,\|\cdot\|)$ be a normed space and let $\omega:[0,\infty)\to[0,\infty)$ be non-decreasing
	with $\omega(0)=0$ and continuous. A map $\theta\mapsto h_\theta\in \cY^\cX$ is \emph{$\omega$-Lipschitz (w.r.t.\ $\|\cdot\|$ and $d_\ell$)}
	on $\Theta$ if for all $\theta,\theta'\in\Theta$,
	\[
	d_\ell(h_\theta,h_{\theta'})\ \le\ \omega\big(\|\theta-\theta'\|\big).
	\]
\end{definition}

We have the following general claim. 
\begin{lemma}\label{lem:master-cover-transfer}
	Let $\Theta$ be a set in a normed space $(\Theta,\|\cdot\|)$ and let $\cH:=\{h_\theta:\theta\in\Theta\}$.
	Assume the parameterization is $\omega$-Lipschitz in the sense of
	Definition~\ref{def:modulus-lipschitz}, for some nondecreasing  continuous $\omega$ with $\omega(0)=0$.
	Define the inverse
	\[
	\omega^{\leftarrow}(\eps)\ :=\ \sup\{t\ge 0:\ \omega(t)\le \eps\},\qquad \eps\ge 0,
	\]
	with the convention $\sup\emptyset:=0$.
	Then for every $\eps>0$,
	\[
	N\big(\cH,\eps\big)\ \le\ N\big(\Theta,\omega^{\leftarrow}(\eps)\big),
	\]
	where the covering number on the right is with respect to $\|\cdot\|$.
\end{lemma}

\begin{proof}
	Fix $\eps>0$ and set $r:=\omega^{\leftarrow}(\eps)$.
	Let $S\subseteq\Theta$ be an $r$-cover of $\Theta$ under $\|\cdot\|$.
	Define $\cN:=\{h_s:\ s\in S\}\subseteq \cH$.
	For any $\theta\in\Theta$ choose $s\in S$ such that $\|\theta-s\|\le r$.
	By $\omega$-Lipschitzness and monotonicity of $\omega$,
	\[
	d_\ell(h_\theta,h_s)\ \le\ \omega(\|\theta-s\|)\ \le\ \omega(r)\ \le\ \eps,
	\]
	where $\omega(r)\le \eps$ holds by the definition of $r=\omega^{\leftarrow}(\eps)$ and the fact that $\omega$ is continuous.
	Thus $\cN$ is an $\eps$-cover of $\cH$. Minimizing over such $S$ yields
	$N(\cH,\eps)\le N(\Theta,r)=N(\Theta,\omega^{\leftarrow}(\eps))$.
\end{proof}

\subsection{Entropy potential bounds from parameter entropy}

Assume $\diam(\cH):=\sup_{f,g\in\cH} d_\ell(f,g) < \infty$ and recall
\[
\Phi(\cH)\ :=\ \int_{0}^{\diam(\cH)} \log_2 N(\cH,\eps)\,d\eps.
\]
\Cref{lem:master-cover-transfer} immediately reduces bounding $\Phi(\cH)$ to bounding the metric entropy of $\Theta$.

\begin{corollary}[Potential bound via parameter entropy]\label{cor:Phi-via-Theta}
	Under the assumptions of Lemma~\ref{lem:master-cover-transfer}, if $\diam(\cH)<\infty$ then
	\[
	\Phi(\cH)\ \le\ \int_{0}^{\diam(\cH)} \log_2 N\big(\Theta,\omega^{\leftarrow}(\eps)\big)\,d\eps.
	\]
\end{corollary}

\begin{proof}
	By Lemma~\ref{lem:master-cover-transfer}, for every $\eps>0$,
	$\log_2 N(\cH,\eps)\le \log_2 N(\Theta,\omega^{\leftarrow}(\eps))$.
	Integrating over $(0,\diam(\cH)]$ gives the claim (the integrands are measurable as $\eps\mapsto N(\cdot,\eps)$
	is nonincreasing).
\end{proof}

\subsection{Parametric polynomial covers}

The most common regime is when $\Theta$ has polynomial covering numbers in the chosen norm.
The next corollaries show that in this case $\Phi(\cH)$ (and hence $\mathbb{D}_{\mathrm{onl}}(\cH)$) is finite and scales
linearly with the parameter dimension.

\begin{corollary}[Linear modulus: absolute-loss-type control]\label{cor:linear-modulus}
	Assume $\diam(\cH)\le 1$ and that the parameterization is linear-modulus Lipschitz:
	$
	d_\ell(h_\theta,h_{\theta'})\ \le\ L\|\theta-\theta'\|$ for all $\theta,\theta'\in\Theta$,
	i.e.,\ $\omega(t)=Lt$ for a Lipschitz constant $L>0$. Assume moreover that for some constants $C\ge 1$ and $p\ge 1$,
	$
	N(\Theta,r)\ \le\ \Big(\frac{C}{r}\Big)^p\qquad\text{for all }r\in(0,1].
	$
	Then for all $\eps\in(0,1]$,
	\[
	N(\cH,\eps)\ \le\ \Big(\frac{C\max\{L,1\}}{\eps}\Big)^p,
	\]
	and consequently
	\[
	\Phi(\cH)\ \le\ p\Big(\log_2\!\big(C\max\{L,1\}\big)+\frac{1}{\ln 2}\Big).
	\]
\end{corollary}

\begin{proof}
	For $\eps\in(0,1]$, Lemma~\ref{lem:master-cover-transfer} with $\omega(t)=Lt$ gives
	$N(\cH,\eps)\le N(\Theta,\eps/L)$ (since $\omega^{\leftarrow}(\eps)=\eps/L$).
	If $L\ge 1$ then $N(\Theta,\eps/L)\le (CL/\eps)^p$.
	If $L<1$ then $\eps/L>1$, so by monotonicity
	$N(\Theta,\eps/L)\le N(\Theta,1)\le C^p\le (C/\eps)^p$.
	
	In both cases,
	\[
	N(\cH,\eps)\ \le\ \Big(\frac{C\max\{L,1\}}{\eps}\Big)^p.
	\]
	Therefore,
	\[
	\log_2 N(\cH,\eps)\ \le\ p\log_2\!\big(C\max\{L,1\}\big)+p\log_2(1/\eps).
	\]
	Integrating over $\eps\in(0,1]$ yields
	\[
	\Phi(\cH)\ \le\ p\log_2\!\big(C\max\{L,1\}\big)+p\int_0^1 \log_2(1/\eps)\,d\eps
	\ =\ p\log_2\!\big(C\max\{L,1\}\big)+\frac{p}{\ln 2},
	\]
	using $\int_0^1 \ln(1/\eps)\,d\eps=1$.
\end{proof}

\begin{corollary}[Power modulus: squared and $q$-power losses]\label{cor:power-modulus}
	Assume $\diam(\cH)\le 1$ and that for some $q\ge 1$,
	$
	d_\ell(h_\theta,h_{\theta'})\ \le\ (L\|\theta-\theta'\|)^q$ for all  $\theta,\theta'\in\Theta$,
	i.e., \ $\omega(t)=(Lt)^q$ for $L>0$. Assume again that $N(\Theta,r)\le (C/r)^p$ for all $r\in(0,1]$.
	Then for all $\eps\in(0,1]$,
	\[
	N(\cH,\eps)\ \le\ \Big(\frac{C\max\{L,1\}}{\eps^{1/q}}\Big)^p,
	\]
	and consequently
	\[
	\Phi(\cH)\ \le\ p\log_2\!\big(C\max\{L,1\}\big)\ +\ \frac{p}{q\ln 2}.
	\]
\end{corollary}

\begin{proof}
	For $\eps\in(0,1]$, Lemma~\ref{lem:master-cover-transfer} with $\omega(t)=(Lt)^q$ gives
	\[
	N(\cH,\eps)\ \le\ N\!\Big(\Theta,\,\omega^{\leftarrow}(\eps)\Big)
	\ =\ N\!\left(\Theta,\,\frac{\eps^{1/q}}{L}\right),
	\]
	since $\omega^{\leftarrow}(\eps)=\eps^{1/q}/L$.
	If $L\ge 1$ then $N(\Theta,\eps^{1/q}/L)\le (CL/\eps^{1/q})^p$.
	If $L<1$ then $\eps^{1/q}/L>1$, so by monotonicity
	$N(\Theta,\eps^{1/q}/L)\le N(\Theta,1)\le C^p\le (C/\eps^{1/q})^p$.
	In both cases,
	\[
	N(\cH,\eps)\ \le\ \Big(\frac{C\max\{L,1\}}{\eps^{1/q}}\Big)^p.
	\]
	Thus $\log_2 N(\cH,\eps)\le p\log_2\!\big(C\max\{L,1\}\big)+\frac{p}{q}\log_2(1/\eps)$.
	Integrating over $\eps\in(0,1]$ yields the displayed bound, using again
	$\int_0^1 \log_2(1/\eps)\,d\eps = 1/\ln 2$.
\end{proof}

\subsection{Implications for realizable online regression}

Combining Corollaries~\ref{cor:Phi-via-Theta}--\ref{cor:power-modulus} with
\Cref{thm:intro_Donl-via-Phi} and Theorem~\ref{thm:potential-bound-general} yields realizable
cumulative-loss bounds under any $c$-approximate pseudo-metric loss $\ell$ for which one can
establish a modulus-Lipschitz parameterization and a covering bound for the parameter set $\Theta$
in the chosen norm. We give an example for a class of two-ReLU functions.

\subsection{Warm Up Application: a two-ReLU class under generic losses}

We illustrate how the entropy-potential method yields a covering-based bound for a simple
ReLU network class under a broad family of losses. Then, in the next section we generalize similar arguments to deep networks with arbitrary Lipschitz activations (rather than ReLU).
The argument applies to any loss $\ell$
that (i) is a $c$-approximate pseudo-metric (Definition~\ref{def:approx-pseudometric}) and
(ii) is dominated by a continuous non decreasing function of the absolute label gap.

\paragraph{Network class.}
Fix $\cX\subseteq[-1,1]^d$ and recall that $\rel(z):=\max\{0,z\}$.
Let $p:=2d+5$ and $\Theta:=[-1,1]^p$ with coordinates
$\theta=(w_1,b_1,a_1,w_2,b_2,a_2,c)$ where
$w_1,w_2\in[-1,1]^d$ and $b_1,b_2,a_1,a_2,c\in[-1,1]$.
Define the \emph{clipped} two-ReLU class $\cH\subseteq[0,1]^\cX$ by
\begin{equation}\label{eq:two-relu-class-clipped}
	h_\theta(x)
	:=
	\mathrm{clip}_{[0,1]}\!\Big(a_1\rel(\langle w_1,x\rangle+b_1)+a_2\rel(\langle w_2,x\rangle+b_2)+c\Big),
	\qquad \theta\in\Theta,
\end{equation}
where $\mathrm{clip}_{[0,1]}(z):=\min\{1,\max\{0,z\}\}$.	
We use the following auxiliary claim. 

\begin{lemma}\label{lem:Theta-general-cover}
	Let $\|\cdot\|$ be any norm on $\R^p$, and let $\Theta=[-1,1]^p$. Define
	\[
	R_\Theta \ :=\ \sup_{\theta\in\Theta}\|\theta\|,
	\qquad\text{and}\qquad
	\alpha\ :=\ \sup_{\|u\|_\infty\le 1}\|u\|.
	\]
	Then $R_\Theta=\alpha$, and for every $r>0$,
	\[
	N(\Theta,r)\ \le\ \left(1+\frac{2\alpha}{r}\right)^p.
	\]
	In particular, for every $r\in(0,2\alpha]$,
	\[
	N(\Theta,r)\ \le\ \left(\frac{4\alpha}{r}\right)^p.
	\]
\end{lemma}

\begin{proof}
	Since $\Theta=\{u:\|u\|_\infty\le 1\}$, we have $R_\Theta=\sup_{\|u\|_\infty\le 1}\|u\|=\alpha$.
	Let $\delta:=\min\{1,r/\alpha\}$.
	If $\|u-v\|_\infty\le \delta$, then $u-v=\delta z$ for some $z$ with $\|z\|_\infty\le 1$, hence
	$\|u-v\|=\delta\|z\|\le \delta\alpha\le r$.
	Therefore, any $\delta$-cover of $\Theta$ in $\|\cdot\|_\infty$ is an $r$-cover in $\|\cdot\|$.
	A uniform grid in $[-1,1]^p$ with mesh $\delta$ has size at most $(2/\delta+1)^p$, so
	\[
	N(\Theta,r)\ \le\ \left(\frac{2}{\delta}+1\right)^p
	\le \left(1+\frac{2\alpha}{r}\right)^p.
	\]
	If moreover $r\le 2\alpha$, then $2\alpha/r\ge 1$ and thus $1+2\alpha/r\le 4\alpha/r$, giving
	$N(\Theta,r)\le (4\alpha/r)^p$.
\end{proof}

\begin{theorem}[Two-ReLU entropy potential under dominated losses]\label{thm:two-relu-phi}
	Suppose there exists a nondecreasing function $\varphi:[0,\infty)\to[0,\infty)$ with $\varphi(0)=0$ such that
	\[
	\ell(u,v)\ \le\ \varphi\big(|u-v|\big)\qquad\text{for all }u,v\in\R.
	\]
	Let $\|\cdot\|$ be any norm on $\R^p$ and assume that there exists a constant $K\ge 0$ such that for all $\theta,\theta'\in\Theta$,
	\begin{equation}\label{eq:two-relu-Lip-general-norm}
		\|h_\theta-h_{\theta'}\|_\infty\ \le\ K\,\|\theta-\theta'\|.
	\end{equation}
	Define the generalized inverse
	\[
	\varphi^{\leftarrow}(\eps)\ :=\ \sup\{t\ge 0:\ \varphi(t)\le \eps\}.
	\]
	Then for every $\eps\in(0,\diam(\cH)]$,
	\[
	N(\cH,\eps)\ \le\ \left(\frac{4\alpha K}{\varphi^{\leftarrow}(\eps)}\right)^p,
	\]
	with the convention that the right-hand side is $+\infty$ if $\varphi^{\leftarrow}(\eps)=0$.
	Consequently,
	\[
	\Phi(\cH)
	\ \le\
	p\,\diam(\cH)\,\log_2\!\big(4\alpha K\big)
	\ +\
	p\int_0^{\diam(\cH)} \log_2\!\Big(\frac{1}{\varphi^{\leftarrow}(\eps)}\Big)\,d\eps.
	\]
	In particular, if $\varphi(t)=t^q$ for some $q\ge1$ and $\diam(\cH)\le 1$, then
	\[
	\Phi(\cH)\ \le\ p\log_2(4\alpha K)\ +\ \frac{p}{q\ln 2}.
	\]
\end{theorem}

\begin{proof}
	If $K=0$ then \eqref{eq:two-relu-Lip-general-norm} implies $h_\theta$ is constant over $\theta\in\Theta$, hence
	$N(\cH,\eps)=1$ for all $\eps>0$ and the claims are immediate. Assume henceforth $K>0$.
	For any $\theta,\theta'\in\Theta$,
	\[
	d_\ell(h_\theta,h_{\theta'})
	=\sup_{x\in\cX}\ell\big(h_\theta(x),h_{\theta'}(x)\big)
	\le \sup_{x\in\cX}\varphi\big(|h_\theta(x)-h_{\theta'}(x)|\big)
	\le \varphi\big(\|h_\theta-h_{\theta'}\|_\infty\big)
	\le \varphi\big(K\|\theta-\theta'\|\big).
	\]
	Thus, any $r$-cover of $\Theta$ in $\|\cdot\|$ induces a $\varphi(Kr)$-cover of $\cH$ in $d_\ell$, and therefore
	\[
	N(\cH,\eps)\ \le\ N\!\left(\Theta,\ \frac{\varphi^{\leftarrow}(\eps)}{K}\right).
	\]
	Set $r:=\varphi^{\leftarrow}(\eps)/K$. Since $\Theta$ has diameter at most $2R_\Theta=2\alpha$ under $\|\cdot\|$,
	\eqref{eq:two-relu-Lip-general-norm} implies $\|h_\theta-h_{\theta'}\|_\infty\le 2\alpha K$ for all $\theta,\theta'\in\Theta$,
	hence $\diam(\cH)\le \varphi(2\alpha K)$ and so $\varphi^{\leftarrow}(\eps)\le 2\alpha K$ for all $\eps\le\diam(\cH)$.
	Therefore $r\le 2\alpha$, and Lemma~\ref{lem:Theta-general-cover} yields
	\[
	N(\cH,\eps)\ \le\ \left(\frac{4\alpha}{r}\right)^p
	=\left(\frac{4\alpha K}{\varphi^{\leftarrow}(\eps)}\right)^p.
	\]
	Taking $\log_2$ and integrating over $\eps\in(0,\diam(\cH)]$ gives the bound on $\Phi(\cH)$.
	If $\varphi(t)=t^q$ and $\diam(\cH)\le 1$, then $\varphi^{\leftarrow}(\eps)=\eps^{1/q}$ and
	\[
	\int_0^1 \log_2\!\Big(\frac{1}{\eps^{1/q}}\Big)\,d\eps
	=\frac{1}{q}\int_0^1 \log_2(1/\eps)\,d\eps
	=\frac{1}{q\ln 2},
	\]
	which yields the final bound.
\end{proof}

\paragraph{Instantiations.}
For $\|\cdot\|_1$, one has $\alpha=\sup_{\|u\|_\infty\le 1}\|u\|_1=p$, and one may take $K=d+1$.
For $\|\cdot\|_2$, one has $\alpha=\sqrt{p}$, and one may take $K=(d+1)\sqrt{p}$.

\subsection{Application: depth-$L$ networks with Lipschitz activations}
\label{sec:depth}
We now extend Theorem~\ref{thm:two-relu-phi} from the clipped two-ReLU class to depth-$L$ fully-connected networks with a generic Lipschitz activation. The arguments are analogous to the above proofs.  

\paragraph{Network class.}
Fix $\cX\subseteq[-1,1]^d$. Let $\sigma:\R\to\R$ be $L_\sigma$-Lipschitz, i.e.,
$|\sigma(u)-\sigma(v)|\le L_\sigma |u-v|$ for all $u,v\in\R$, and apply $\sigma$ coordinatewise to vectors.
Fix integers $L\ge 2$ and $k\ge 1$ and set the number of scalar parameters as 
\[
p \ :=\  kd + (L-2)k^2 + Lk + 1.
\]
Let $\Theta_{L,k}:=[-1,1]^p$, interpreted as the set of flattened parameters
\[
\theta=(W_1,b_1,\dots,W_{L-1},b_{L-1},a,c),
\]
where $W_1\in[-1,1]^{k\times d}$, $b_1\in[-1,1]^k$, and for $2\le \ell\le L-1$ we have
$W_\ell\in[-1,1]^{k\times k}$ and $b_\ell\in[-1,1]^k$, while the output weights satisfy
$a\in[-1,1]^k$ and $c\in[-1,1]$.
For each $\theta\in\Theta_{L,k}$ define recursively
\[
z^{(0)}(x)=x,\qquad
z^{(\ell)}(x)=\sigma\!\big(W_\ell z^{(\ell-1)}(x)+b_\ell\big),\quad \ell=1,\dots,L-1,
\]
and define the clipped output
\begin{equation}\label{eq:deep-lip-class}
	h_\theta(x)
	\ :=\
	\mathrm{clip}_{[0,1]}\!\big(\langle a,z^{(L-1)}(x)\rangle+c\big),
	\qquad x\in\cX.
\end{equation}
Let $\cH_{L,k,\sigma}:=\{h_\theta:\theta\in\Theta_{L,k}\}\subseteq[0,1]^\cX$.

\paragraph{A parameter Lipschitz bound.}
Let $M_0:=1$ and define
\[
M_1 \ :=\ |\sigma(0)| + L_\sigma\,(dM_0+1),\qquad
M_\ell \ :=\ |\sigma(0)| + L_\sigma\,(kM_{\ell-1}+1)\quad\text{for } \ell=2,\dots,L-1,
\]
\[
\bar M \ :=\ \max_{0\le \ell\le L-1} M_\ell,
\qquad
S \ :=\ \sum_{s=0}^{L-2}(L_\sigma k)^s.
\]
Define
\begin{equation}\label{eq:K-deep}
	K_{L,k,\sigma}
	\ :=\
	(1+\bar M)\bigl(1+L_\sigma S\bigr).
\end{equation}

\begin{lemma}\label{lem:deep-param-lip}
	For the $\ell_1$ norm on $\R^p$, for all $\theta,\theta'\in\Theta_{L,k}$ one has
	\[
	\|h_\theta-h_{\theta'}\|_\infty\ \le\ K_{L,k,\sigma}\,\|\theta-\theta'\|_1.
	\]
\end{lemma}

\begin{proof}
	Since $\mathrm{clip}_{[0,1]}$ is $1$-Lipschitz, it suffices to bound the corresponding
	unclipped outputs. Fix $x\in\cX$ and write $z^{(\ell)}(x)$ and ${z'}^{(\ell)}(x)$ for the hidden
	vectors under $\theta$ and $\theta'$. By induction, using $|\sigma(t)|\le |\sigma(0)|+L_\sigma|t|$,
	$\|W_1 v\|_\infty\le d\|v\|_\infty$ for $W_1\in[-1,1]^{k\times d}$,
	$\|Wv\|_\infty\le k\|v\|_\infty$ for $W\in[-1,1]^{k\times k}$, and $\|b_\ell\|_\infty\le 1$, we get
	$\|z^{(\ell)}(x)\|_\infty\le M_\ell$ and $\|{z'}^{(\ell)}(x)\|_\infty\le M_\ell$ for all $\ell$.
	
	Let $\Delta_\ell:=\|z^{(\ell)}(x)-{z'}^{(\ell)}(x)\|_1$. Using Lipschitzness of $\sigma$ and
	$\|Wv\|_1\le k\|v\|_1$ (since each column sum of $|W|$ is at most $k$), one obtains
	\[
	\Delta_\ell
	\ \le\
	L_\sigma\Big(k\Delta_{\ell-1}
	+\bar M\|W_\ell-W'_\ell\|_1
	+\|b_\ell-b'_\ell\|_1\Big),
	\qquad \ell=1,\dots,L-1.
	\]
	Unrolling this recursion gives
	\[
	\Delta_{L-1}
	\ \le\
	L_\sigma(\bar M+1)\Big(\sum_{s=0}^{L-2}(L_\sigma k)^s\Big)\,\|\theta-\theta'\|_1
	\ =\
	L_\sigma(\bar M+1)S\,\|\theta-\theta'\|_1.
	\]
	Finally,
	\[
	|\langle a,z^{(L-1)}(x)\rangle+c-\langle a',{z'}^{(L-1)}(x)\rangle-c'|
	\ \le\
	\Delta_{L-1}+\bar M\|a-a'\|_1+|c-c'|
	\ \le\
	K_{L,k,\sigma}\|\theta-\theta'\|_1,
	\]
	and taking $\sup_{x\in\cX}$ yields the claim.
\end{proof}

\paragraph{Entropy potential under dominated losses.}
Let $\ell$ be any loss dominated by a nondecreasing $\varphi$ as in Theorem~\ref{thm:two-relu-phi}.
Apply Theorem~\ref{thm:two-relu-phi} to $\cH_{L,k,\sigma}$ with $\|\cdot\|=\|\cdot\|_1$.
Then $\alpha=\sup_{\|u\|_\infty\le 1}\|u\|_1=p$ and Lemma~\ref{lem:deep-param-lip} provides
\eqref{eq:two-relu-Lip-general-norm} with $K=K_{L,k,\sigma}$. Hence, for every
$\eps\in(0,\diam(\cH_{L,k,\sigma})]$,
\[
N(\cH_{L,k,\sigma},\eps)
\ \le\
\left(\frac{4pK_{L,k,\sigma}}{\varphi^{\leftarrow}(\eps)}\right)^p,
\]
and the corresponding bound on $\Phi(\cH_{L,k,\sigma})$ follows exactly as in
Theorem~\ref{thm:two-relu-phi}.
In particular, if $\varphi(t)=t^q$ and $\diam(\cH_{L,k,\sigma})\le 1$ (e.g.\ since $h_\theta\in[0,1]^\cX$),
then
\[
\Phi(\cH_{L,k,\sigma})
\ \le\
p\log_2\!\big(4pK_{L,k,\sigma}\big)
\ +\
\frac{p}{q\ln 2}.
\]

\section{When the entropy potential diverges}
\label{app:div}
The entropy-potential approach yields a uniform (all-horizon) bound on realizable cumulative loss
precisely when $\Phi(\cH)<\infty$. In particular, if $\log N(\cH,\eps)$ diverges too quickly as
$\eps \rightarrow 0$, then $\Phi(\cH)=\infty$ and \Cref{thm:intro_Donl-via-Phi} becomes vacuous.
A typical example is a nonparametric entropy growth of the form
$\log N(\cH,\eps)\asymp \eps^{-p}$: in this case $\Phi(\cH)<\infty$ holds if and only if $p<1$,
since $\int_0^1 \eps^{-p}\,d\eps<\infty$ exactly for $p<1$.
When $\Phi(\cH)=\infty$, one should not expect a horizon-free bound derived from this potential. 
instead, it is natural to seek horizon-dependent guarantees. 	In the following sections, we show that \Cref{thm:intro_Donl-via-Phi} is not tight, in the sense that there are classes with $N(\cH,\varepsilon) \sim (C/\varepsilon)^p$ for large $p$, but their loss can be still bounded by $O(1)$. In addition, we show below that the converse direction does not hold: our potential upper bound may be infinite despite finite scaled littlestone dimension.

\begin{proposition}[$\Phi(\cH)=\infty$ while $\mathbb{D}_{\mathrm{onl}}(\cH)<\infty$]
	There exist $(\cX,\cY,\ell)$ with $\ell$ a metric and a class $\cH\subseteq \cY^{\cX}$
	such that $\diam(\cH)<\infty$, $\Phi(\cH)=\infty$, but $\mathbb{D}_{\mathrm{onl}}(\cH)<\infty$.
\end{proposition}

\begin{proof}
	Let $\cX=\{x_1,x_2,\dots\}$ and 
	let $a_k=2^{-k}$ and $m_k:=2^{2^k}$, so $a_k\log_2 m_k=1$ for all $k$. It follows that $(a_k)_{k\ge1}\subset(0,1]$ is a strictly decreasing sequence 
	with $a_k\downarrow 0$ and $\sum_{k\ge1} a_k<\infty$. Let 
	\[
	\cY \ :=\ \{(k,i): k\ge1,\ i\in[m_k]\}.
	\]
	Define $\ell:\cY\times\cY\to\R_{\ge0}$ by
	\[
	\ell\big((k,i),(k',i')\big):=
	\begin{cases}
		0, & (k,i)=(k',i'),\\
		a_k, & k=k',\ i\neq i',\\
		\max\{a_k,a_{k'}\}, & k\neq k'.
	\end{cases}
	\]
	Then $\ell$ is a metric: for any $y_1,y_2,y_3\in\cY$,
	$\ell(y_1,y_2)\le \max\{\ell(y_1,y_3),\ell(y_2,y_3)\}$, which is immediate from the case analysis
	($\max\{\cdot,\cdot\}$ always dominates the larger scale among the involved blocks).
	Define the hypothesis class
	\[
	\cH\ :=\ \big\{h:\cX\to\cY:\ \exists (i_k)_{k\ge1}\text{ with }h(x_k)=(k,i_k)\ \forall k\big\}.
	\]
	For the induced metric $d_\ell(f,g)=\sup_{x\in\cX}\ell(f(x),g(x))$, we have
	\[
	d_\ell(f,g)=\max\{a_k:\ f(x_k)\neq g(x_k)\},
	\]
	so $\diam(\cH)=a_1<\infty$.
	
	\paragraph{Finite $\mathbb{D}_{\mathrm{onl}}(\cH)$.}
	Fix any scaled Littlestone tree for $\cH$. If a node queries $x_k$ for the first time along a branch,
	then both outgoing edge labels must lie in $\{(k,i):i\in[m_k]\}$, so its gap is either $0$ (same label)
	or $a_k$ (distinct labels). After the branch takes one of these edges, the value of $h(x_k)$ is fixed
	in the version space along that branch; hence any later query of $x_k$ on the same branch can only have
	both children nonempty if both edge labels coincide with that fixed value, yielding gap $0$.
	Therefore, along every branch the total gap is at most $\sum_{k\ge1} a_k$, and hence for every tree
	\(
	\inf_{y}\sum_i \gamma_{y_i}\le \sum_{k\ge1} a_k.
	\)
	Taking the supremum over trees gives
	\[
	\mathbb{D}_{\mathrm{onl}}(\cH)\ \le\ \sum_{k\ge1} a_k\ <\ \infty.
	\]
	
	\paragraph{Infinite $\Phi(\cH)$.}
	For $\varepsilon>0$, let $K(\varepsilon):=\max\{k: a_k>\varepsilon\}$ (which is finite as $a_{k} \downarrow 0$) where $K(\eps) = 0$ for $\eps \geq a_1$.
	Then two hypotheses are within $d_\ell$-distance $\le\varepsilon$ iff they agree on all coordinates
	$x_1,\dots,x_{K(\varepsilon)}$. Thus, for every $f\in\cH$ we have
	\[
	B_{d_\ell}(f,\varepsilon)
	=\{g\in\cH:\ g(x_i)=f(x_i)\ \forall i\le K(\varepsilon)\},
	\]
	where $B_{d_\ell}(f,\varepsilon):=\{g\in\cH:\ d_\ell(f,g)\le \varepsilon\}$ denotes the (closed) $\varepsilon$-ball around $f$ in $(\cH,d_\ell)$.
	In particular, any $\varepsilon$-cover must contain at least one element from each choice of
	$(h(x_1),\dots,h(x_{K(\varepsilon)}))$, so
	$N(\cH,\varepsilon)\ge \prod_{i=1}^{K(\varepsilon)} m_i$.
	Conversely, fixing one representative for each such choice yields an $\varepsilon$-cover, hence
	\[
	N(\cH,\varepsilon)\ =\ \prod_{i=1}^{K(\varepsilon)} m_i
	\qquad(\varepsilon>0).
	\]

	Since $K(\varepsilon)=k$ for all $\varepsilon\in(a_{k+1},a_k)$, we get
	\[
	\begin{aligned}
		\Phi(\cH)
		&=\int_0^{a_1}\log_2 N(\cH,\varepsilon)\,d\varepsilon
		=\sum_{k\ge1}(a_k-a_{k+1})\sum_{i=1}^k \log_2 m_i\\
		&=\lim_{n\to\infty}\Big(\sum_{k=1}^n a_k\log_2 m_k \;-\; a_{n+1}\sum_{i=1}^n \log_2 m_i\Big).
	\end{aligned}
	\]
	For $a_k=2^{-k}$ and $m_k=2^{2^k}$ we have $\sum_{k=1}^n a_k\log_2 m_k=n$ and
	$a_{n+1}\sum_{i=1}^n \log_2 m_i = 1-2^{-n}\to 1$, hence $\Phi(\cH)=+\infty$.
	Moreover, $\mathbb{D}_{\mathrm{onl}}(\cH)\le \sum_{k\ge1} a_k = \sum_{k\ge1}2^{-k}=1$.
\end{proof}

\end{document}